\documentclass[a4paper,11pt]{article}

\usepackage{xr}
\makeatletter

\newcommand*{\addFileDependency}[1]{% argument=file name and extension
\typeout{(#1)}% latexmk will find this if $recorder=0
\@addtofilelist{#1}
\IfFileExists{#1}{}{\typeout{No file #1.}}
}\makeatother

\usepackage{xr-hyper}
\usepackage{mattsstyle}
\usepackage{tcolorbox}
\usepackage{soul}
\usepackage{bm}
\usepackage{comment}
\usepackage{xfrac}

\usepackage{booktabs,tabularx,ragged2e}
\newcolumntype{Y}{>{\RaggedRight\arraybackslash}X}
\usepackage{adjustbox} 

\def\Lp#1{\mathrm{L}^{#1}}

\def\Ck#1{\mathrm{C}^{#1}}

\def\spaceBar{\, | \,}

\newcommand{\ub}{\mathbf{u}}

\DeclareMathOperator{\nn}{NN}

\def\commentOut#1{}

\allowdisplaybreaks

\usepackage{pifont}
\title{A Deep Learning Framework for Multi-Operator Learning: Architectures and Approximation Theory}
\author[1]{Adrien Weihs}
\author[2]{Jingmin Sun}
\author[3]{Zecheng Zhang}
\author[1]{Hayden Schaeffer}

\affil[1]{Department of Mathematics,\protect\\ University of California Los Angeles,\protect\\ Los Angeles, CA 90095, USA. \vspace{\baselineskip}}
\affil[2]{Department of Applied Mathematics and Statistics,\protect\\ Johns Hopkins University,\protect\\ Baltimore, MD 21218 , USA. \vspace{\baselineskip}}
\affil[3]{Department of Applied Computational Mathematics and Statistics, \protect\\ University of Notre Dame,\protect\\ Notre Dame, IN 46556, USA }

\date{October 2025}
\newcommand{\UAPNetwork}{\mathrm{MONet}}
\newcommand{\UAPNetworkFinite}{\mathrm{MONet_{vect}}}
\newcommand{\ScalingNetwork}{\mathrm{MNO}}
\newcommand{\CCNet}{\mathrm{Net}}

\begin{document}

\maketitle

\begin{abstract}
\noindent While many problems in machine learning focus on learning mappings between finite-dimensional spaces, scientific applications require approximating mappings between function spaces, i.e., operators. 
We study the problem of learning collections of operators and provide both theoretical and empirical advances. 
We distinguish between two regimes: (i) multiple operator learning, where a single network represents a continuum of operators parameterized by a parametric function, and (ii) learning several distinct single operators, where each operator is learned independently. 
For the multiple operator case, we introduce two new architectures, $\ScalingNetwork$ and $\UAPNetwork$, and establish universal approximation results in three settings: continuous, integrable, or Lipschitz operators. For the latter, we further derive explicit scaling laws that quantify how the network size must grow to achieve a target approximation accuracy. For learning several single operators, we develop a framework for balancing architectural complexity across subnetworks and show how approximation order determines computational efficiency.
Empirical experiments on parametric PDE benchmarks confirm the strong expressive power and efficiency of the proposed architectures. Overall, this work establishes a unified theoretical and practical foundation for scalable neural operator learning across multiple operators.
\end{abstract}



\section{Introduction} %\label{sec:intro}

Classical machine learning is primarily concerned with learning functions of the form  
\[
f : \mathbb{R}^n \to \mathbb{R}^d
\]  
where finite-dimensional inputs are mapped to finite-dimensional outputs. In many scientific and engineering applications, however, the goal is to approximate mappings between function spaces,  
\[
G : U \to V
\]  
where \(U\) and \(V\) are typically subsets of infinite-dimensional Banach or Hilbert spaces. 
Such problems arise, for instance, in learning solution operators of ordinary and partial differential equations~\cite{ChenChen1995, deepOnet,Bhattacharya,li2021fourier}, and span a wide range of scientific and engineering domains~\cite{pathak2022fourcastnet,zhu2023fourierdeeponet,jiang2023fouriermionet,li2023fnoseismic,moya2023operatorgrid,chen2023neuraloperator}. 
This framework is known as operator learning, and it extends classical supervised learning from the setting of functions to that of operators acting on functions. 
We refer to \cite{KOVACHKI2024419,lu2022comprehensive} and references therein for a review and comparison of approaches in this topic.

Neural networks form a natural framework for operator learning, combining the flexibility to approximate complex nonlinear mappings with a strong record of empirical success in scientific and engineering applications \cite{deepLearningImages,Graves2013SpeechRW,KHOO_LU_YING_2021,Jentzen,zhangBelnet}. Modern operator-learning networks~\cite{ChenChen1993,ChenChen1995} typically decompose the learned operator into interacting subnetworks that process different aspects of the input, such as spatial variables, input functions, or parameters, before combining them through summation or tensor-like contractions.
For example, in the terminology of DeepONet~\cite{deepOnet}, the branch subnetwork encodes the input function, while the trunk subnetwork represents a basis for the output function space. This basis, constructed by neural networks, can also be designed to mimic classical finite element bases.
These architectures draw inspiration from low-rank approximations, where complex mappings are expressed as sums of separable, lower-dimensional functions~\cite{markovsky2012lowrank}. Neural operator networks can thus be viewed as nonlinear analogues of such expansions, with each subnetwork learning one component of a functional basis.

However, the neural network approach also introduces challenges that are inherent to its design: in particular, how to construct architectures that are both simple to implement and empirically effective, while also supported by rigorous mathematical guarantees. These difficulties become especially pronounced in the multi-operator setting, where the need to represent numerous complex operators often leads to rapidly increasing architectural complexity.

In this work, we distinguish between two related but conceptually distinct settings in which numerous operators are involved. The first is the multiple operator learning setting,
$G : W \to \{G[\alpha]: U^{(\alpha)} \mapsto V^{(\alpha)} \}_{\alpha \in W}$,
where the parametric function $\alpha \in W$ serves as an explicit input to the network, allowing a single model to represent a continuum of operators indexed by $\alpha$.
The second concerns learning several single operators,
$\{G^{(j)} : U^{(j)} \to V^{(j)}\}_{j \in J}$,
where each operator is learned independently and the dependence on the index $j$ remains external to the model.
We summarize and contrast the main differences between these two formulations in Table~\ref{tab:single-vs-multiple}. This distinction clarifies the different modeling challenges posed by operator learning in practice. Building on this, we investigate fundamental theoretical and practical aspects of designing expressive and efficient neural architectures for learning collections of operators. Specifically, we address three central questions:
\begin{enumerate}[label=\textbf{Q.\arabic*}]
    \item Can one construct architectures that are provably expressive, yielding (quantitative) universal approximation guarantees?
    %\label{q1}
    \item How can network architectures be designed to exploit shared structure across related operators, balance complexity among functional components, and attain optimal approximation and scaling performance? %\label{q2}
    \item Are the proposed architectures empirically efficient and capable of strong performance on representative learning tasks? %\label{q3}
\end{enumerate}
In addressing these questions, we provide both theoretical insights and empirical evidence that clarify the principles underlying expressive and scalable operator-learning networks.

\subsection{Key Contributions}

Our main contributions in the multiple operator learning setting (summarized in Table \ref{tab:contributions}) are as follows:

\begin{enumerate}

    \item We introduce two new architectures for multiple operator learning, 
    $\ScalingNetwork$ and $\UAPNetwork$, designed to generalize existing 
    operator learning models and provide flexible building blocks for 
    theoretical and practical analysis. 
    
    \item We establish \textbf{universal approximation} results for multi-operator learning, 
    showing that both architectures can approximate any continuous operator 
    to arbitrary accuracy on compact sets. 
    
    \item We establish a \textbf{weak universal approximation} property for multi-operator learning, 
    proving that both architectures can approximate measurable operators, 
    thereby extending expressivity guarantees beyond the continuous setting. 
    
    \item We establish a \textbf{strong universal approximation} property 
    for Lipschitz operators when approximated using our proposed $\ScalingNetwork$ model for multi-operator learning. In this case 
    we also derive \textbf{scaling laws}, i.e., quantitative estimates of 
    the required network size to achieve a prescribed accuracy. Specifically, we show that the approximation error $\eps$ scales as follows for $N_\#$ the total number of parameters in the network: \[\eps \asymp  \l \frac{\log \log N_\# }{ \log \log \log N_\#} \r^{-1/d_W},\]
    where $d_W$ denotes the dimension of the domain of functions in $W$. This is done in the general multi-operator setting without additional knowledge of the properties of the collection operators besides their regularity.

    \item We show that our scaling results apply not only to the proposed 
    $\ScalingNetwork$, but also \textbf{extend to a broad family of architectures}, 
    including MIONet \cite{mionet}, thereby unifying several 
    existing approaches under a common framework.

    \item We complement our theoretical contributions with \textbf{empirical validation} on a wide range of PDE problems, considering both discrete and continuum input parameters $\alpha$, and demonstrate that the proposed architectures achieve strong performance in practice.
\end{enumerate}

\begin{table}[H]
\centering
\small
\renewcommand{\arraystretch}{1.5}
\setlength{\tabcolsep}{8pt}
\begin{tabularx}{\linewidth}{>{\bfseries}p{6cm} Y Y}
\toprule
& $\UAPNetwork$ & $\ScalingNetwork$ \\
\midrule
New architectures 
& \checkmark 
& \checkmark \\
\midrule
Standard UAP (continuous operators) 
& \checkmark (Theorem \ref{thm:mainResult:universalApproximationI})
& \checkmark (Theorem \ref{thm:mainResult:universalApproximationI})\\
\midrule
Weak UAP (measurable operators) 
& \checkmark (Theorem \ref{thm:mainResult:universalApproximationII})
& \checkmark (Theorem \ref{thm:mainResult:universalApproximationII})\\
\midrule
Strong UAP (Lipschitz operators) 
& — 
& \checkmark (Theorem \ref{thm:main:multipleOperatorApproximation})\\
\midrule
Scaling laws (quantitative rates) 
& — 
& \checkmark (Theorem \ref{thm:main:multipleOperatorApproximation})\\
\midrule
Empirical validation 
& \checkmark 
& \checkmark \\
\bottomrule
\end{tabularx}
\caption{Summary of contributions in the multiple operator learning setting: expressivity guarantees/universal approximation property (UAP), scaling laws, and empirical validation for $\ScalingNetwork$ and $\UAPNetwork$ which are two architectures for multiple operator learning.}
\label{tab:contributions}
\end{table}

\subsection{Additional Contributions}

In addition to the main contributions above, our results contribute to the setting of learning several single operators (summarized in Table~\ref{tab:single-operator-contributions}) and are summarized as follows:

\begin{enumerate}
    \item We establish a \textbf{principled framework for selecting architectures} when approximating several single operators, showing that the index dependence $j$ can, under suitable structural conditions on $U^{(j)}$ and $V^{(j)}$, be absorbed into a single network component for improved efficiency. This generalizes to other related works D2NO \cite{zhang2024d2no}, MIONet \cite{mionet}, and MODNO \cite{zhang2024modno}. 
    \item We demonstrate that the \textbf{theoretical approximation order determines how scaling complexity is distributed between subnetworks}, and that the computational burden can be shifted between components without affecting the expressive capacity of the model.
    \item We show that the \textbf{theoretical approximation order directly impacts efficiency}, emphasizing the key role of architecture in determining computational complexity. Specifically, the attainable rate of approximation with respect to the total number of parameters $N_\#$ depends on the adopted approximation order, yielding either\[
\Bigg(\frac{\log N_\#}{\log\log N_\#}\Bigg)^{-\frac{1}{(1+d_V)d_U}} \quad \text{or} \quad \Bigg(\frac{\log N_\#}{\log\log N_\#}\Bigg)^{-\frac{1}{d_U}},\]
where $d_U$ and $d_V$ denote the dimensions of the domain of functions in $U$ and $V$, respectively. 
    
\end{enumerate}

\begin{table}[H]
\centering
\small
\renewcommand{\arraystretch}{1.5}
\setlength{\tabcolsep}{8pt}
\begin{tabularx}{\linewidth}{>{\bfseries}p{8cm} Y}
\toprule
Principled framework for architecture design 
& \checkmark (Remarks \ref{rem:uniformOperatorApproximation}, \ref{rem:alternative}, \ref{rem:mnoScalingLaws}) \\

\midrule
Balancing of scaling complexity across subnetworks 
& \checkmark (Theorem \ref{thm:back:operatorApproximation} and Remark \ref{rem:balancingComplexity}) \\

\midrule
Impact of approximation order on efficiency 
&  \checkmark (Remarks \ref{rem:totalParametersOperatorLearning} and \ref{rem:complexityOrder}) \\
\bottomrule
\end{tabularx}
\caption{Summary of contributions for the setting of learning several single operators. 
%The results clarify how architectural design and approximation order influence efficiency and scaling behavior.}
}
\label{tab:single-operator-contributions}
\end{table}
Note that the approximation order refers to the choice of hierarchical approximation steps, e.g., approximating functions, then functionals (then operators---in the multiple operator case). 

\subsection{Related works} %\label{subsec:relatedWorks}

A wide range of research has contributed to the development of operator-learning theory and practice, spanning multi-operator learning strategies, neural operator architectures, and the theoretical foundations of expressivity and scaling. We briefly summarize these directions below.

\paragraph{Multi-Operator Learning}
The motivation for learning collections of operators arises in several contexts: in some applications, it is inherent to the problem formulation itself, while in others, it serves as a means to improve the generalization capability of operator-learning models. Recently, several multi-operator learning approaches have been introduced~\cite{sun2025foundation,liu2024prose,mccabe2023multiple,yang2023incontext,yang2023prompting, cao2024vicon, zhang2024modno, zhang2024d2no, liu2025bcat, ye2025pdeformer, zhang2025probabilistic}. In particular, the works of~\cite{sun2025foundation,liu2024prose} demonstrated that multi-operator learning can accurately address new tasks beyond those seen during training.

As previously discussed, one can either (1) learn several single operators independently, or
(2) consider a more general setting in which the family of operators is encoded through a (discrete or continuous) parametric function~$\alpha$. In the former case, exemplified by~\cite{zhang2024modno}, one exploits only the information contained in the input functions of different operators, which typically limits the model’s ability to handle highly varying families of operators and prevents generalization to unseen ones. By contrast, in the latter case, methods that employ an operator-encoding strategy \cite{sun2025foundation,liu2024prose,yang2023prompting, negrini2025multimodal, liu2024prosefd} incorporate an explicit representation of the operator, such as its governing equation, symbolic form, text, or task label, alongside the corresponding input functions. This additional encoding generally yields stronger generalization and represents a potential construction for PDE foundation models. Notably, the inclusion of operator information enables zero-shot generalization to new PDE tasks, as shown in\cite{sun2025foundation}, and such approaches have demonstrated promising capabilities for addressing out-of-distribution tasks without costly retraining.

Despite these advances, a rigorous theoretical understanding of the expressivity and scaling behavior of neural networks in multi-operator regimes remains limited. For learning several single operators, our work provides guidance for architectural design to leverage common structure and achieve optimal scaling complexity. For multiple operator learning, our work introduces new architectures and provides the analysis of universal approximation as well as scaling behavior. 

\paragraph{Neural Operator Architectures}
A variety of neural operator architectures have been developed to approximate mappings between infinite-dimensional function spaces efficiently. Among the most widely used are Deep Operator Networks (DeepONets)\cite{deepOnet}, which rely on low-rank functional decompositions; Fourier Neural Operators \cite{li2021fourier}, motivated by Fourier spectral methods; and Deep Green Networks \cite{deepGreen,BoulleGreen}, which learn Green’s functions of PDEs directly. These models differ primarily in their structural assumptions which in turn govern their discretization strategy, scalability, and domain of applicability. Several variants, including Graph Neural Operators and Multipole Graph Neural Operators\cite{anandkumar2019neural,multipole}, further leverage sparsity or multiscale interactions to reduce computational cost. 
We refer to \cite{KOVACHKI2024419,Goswami2023} and references therein for further models in operator learning.

The proposed $\ScalingNetwork$ and $\UAPNetwork$ architectures retain the separable structure characteristic of DeepONet, while extending these models to the multiple operator learning regime.

\paragraph{Expressivity and Scaling Laws}
The foundation of operator-learning theory rests on universal approximation results, which establish that a given architecture can approximate a broad class of operators to arbitrary accuracy.
The development of an operator network and the study of universal approximation for mappings between spaces of scalar-valued functions is due to \cite{ChenChen1995,ChenChen1993}. 
Extensions to DeepONet were provided by \cite{Lanthaler2022, liu2024neuralscalinglawsdeep}, to the Fourier Neural Operator by \cite{Kovachki2021}, and to PCA-Net by \cite{Bhattacharya} etc. 
Further notable developments related to this work include \cite{mionet,CASTRO2023127413,castro2022,Huang2025,Kovachki2023,zhangBelnet, zhang2025discretization}.

Beyond universal approximation, neural scaling laws provide a quantitative framework for characterizing how network performance scales with data size, model capacity, and computational costs. Developing a theoretical foundation for these laws is essential, as it enables rigorous analysis of generalization error in deep learning and offers predictive insight into how performance improves with increasing data, model complexity, or training time~\cite{kaplan2020scalinglawsneurallanguage}. Empirical studies such as \cite{dehoop2022costaccuracytradeoffoperatorlearning} explored the cost–accuracy trade-off across neural operator architectures, quantifying how network size and data availability affect approximation error. In \cite{liu2024neuralscalinglawsdeep}, the scaling laws and complexity for deep ReLU networks and DeepONet were rigorously derived and analyzed. Additionally, complexity analyses were carried out theoretically for DeepONet by \cite{Lanthaler2022} and extended to PCA-Net in \cite{lanthalerPCAnet}. Related analyses can be found in \cite{marcati2023,herrman,lanthalerStuart,furuya2023globally}. Sample complexity bounds for DeepONet and related models are studied in \cite{liu2024,liu2024neuralscalinglawsdeep}, and out-of-distribution generalization estimates in \cite{benitez}.

In the context of multiple-operator learning, empirical analyses have recently been reported in \cite{sun2025lemonlearninglearnmultioperator,jollie2025time}.
In this work, we establish the universal approximation of $\ScalingNetwork$ and $\UAPNetwork$ for multiple operators.
We also partly extend the work in \cite{liu2024neuralscalinglawsdeep} to the multiple operator setting and derive scaling laws for $\ScalingNetwork$ and related models. 

The remainder of the paper is structured as follows: in Section~\ref{sec:background}, we review the mathematical background relevant to our proposed methods; in Section~\ref{sec:main}, we present our main results; in Section \ref{sec:proofs}, we provide detailed proofs of our results; in Section~\ref{sec:experiments}, we show the strong empirical performance of our proposed models for multiple operator learning; and in Section~\ref{sec:discussion}, we conclude with a summary of our contributions and a discussion of potential directions for future work.

\section{Background} \label{sec:background}

\subsection{Operator learning} \label{sec:back:operator}

In this section, we recall key results in operator learning, related to the framework introduced by \cite{ChenChen1995}, which forms the basis for our subsequent universal approximation analysis. We start by defining the class of Tauber–Wiener (TW) activation functions.

\begin{mydef}[Tauber–Wiener functions]
A function \( \sigma : \mathbb{R} \to \mathbb{R} \) is called a Tauber–Wiener function if for all $a<b$, $\eps > 0$, and $f \in \Ck{0}([a,b])$, there exists a linear combination $g(x) = \sum_{i=1}^N c_i \, \sigma(\lambda_i x + \theta_i)$ such that \[
\Vert f - g \Vert_{\Ck{0}([a,b])} < \eps,
\]
where $N=N(\eps)$ depends on the desired accuracy.
\end{mydef}
The above definition requires that the activation function~$\sigma$ enable the construction of dense subsets of $\Ck{0}([a,b])$. Typical examples of activation functions satisfying this condition, i.e., belonging to the class of Tauber–Wiener functions, include the hyperbolic tangent, bounded sigmoid functions, Gaussian functions, and oscillatory functions such as the sine.
In particular, the ReLU activation 
$\sigma(x) = \max(0,x)$ is also Tauber-Wiener: indeed, the mapping  
\[
\phi \mapsto \int_{\mathbb{R}} \max(0,x)\,\phi(x)\,dx
\]  
defines a continuous linear functional on the Schwartz space, and since ReLU is 
non-polynomial, \cite[Theorem~1]{ChenChen1995} ensures that it belongs to the 
Tauber-Wiener class. ReLU is a popular choice in practice, and we adopt it as our activation 
function in the experiments presented in Section~\ref{sec:experiments}.

Assuming a TW activation function, we introduce the following network, which we denote by $\CCNet$.

\begin{mydef}[$\CCNet$ Network]
    For fixed positive integers \( m,n,p\), constants \( c_{ki}, \zeta_k,  \theta_{ki},\xi_{kij} \in \mathbb{R} \), points \( \omega_k \in \mathbb{R}^n \), \( x_j \in \Omega_U \) (\( i = 1, \dots, n \), \( k = 1, \dots, p \), \( j = 1, \dots, m \)), we define a $\CCNet$ network as:
    \begin{equation}\label{eq:chenchen}
         \CCNet[u](x) = \sum_{k=1}^p \sum_{i=1}^n c_{ki} \, \sigma \left( \sum_{j=1}^m \xi_{kij} u(x_j) + \theta_{ki} \right) \cdot \sigma(\omega_k \cdot x + \zeta_k)
    \end{equation}
    for a continuous function $u:\Omega_U \mapsto \bbR$, $x\in\mathbb{R}^n$ and for some activation function $\sigma\in TW$.
\end{mydef}

Note that the $\CCNet$ network defined in Eq. \eqref{eq:chenchen} can also be re-written as \[
\CCNet[u](x) = \sum_{k=1}^p b_k(u) \tau_k(x)
\]
with $b_k(u) = \sum_{i=1}^n c_{ki} \sigma\l \sum_{j=1}^m \xi_{kij}u(x_j) + \theta_{ki}\r$ and $\tau_k(x) = \sigma(\omega_k \cdot x + \zeta_k)$ both being shallow networks. If we extend the latter to deep networks, one recovers the popular DeepONet architecture \cite{deepOnet}. In this way, the network is a linear combination of the product of nonlinear (separated) sub-networks. 

The $\CCNet$ network enjoys a universal approximation property for nonlinear continuous operators over compact sets. 

\begin{theorem}[Universal Approximation Theorem for Single Operator  \cite{ChenChen1995}]\label{theory_3}
Suppose that Assumptions \ref{assumption:Main:assumptions:A1}, \ref{assumption:Main:assumptions:S2} and \ref{assumption:Main:assumptions:S3} hold. Let \( G \) be a nonlinear continuous operator mapping \( U \mapsto V \), then, for any \( \varepsilon > 0 \), there exists a neural network defined in Eq. \eqref{eq:chenchen}, such that
\begin{equation*}
    \left\| G[u](x) - \CCNet[u](x) \right\|_{\Lp{\infty} \left(U \times \Omega_V \right)} < \varepsilon.
\end{equation*}
\end{theorem}

In Section~\ref{sec:main}, we will introduce new neural network architectures and extend Theorem~\ref{theory_3} to the multiple operator setting. To prepare for this, it is helpful to outline the proof strategy of Theorem~\ref{theory_3} and introduce the technical tools it relies on, which will also be used in the proof of Theorem~\ref{thm:mainResult:universalApproximationI}. The main idea is to sequentially separate the input variables of the operator $G$, thereby reducing the operator approximation problem to the task of approximating functions in finite-dimensional spaces. This reduction is supported by the following result, which guarantees that continuous functions can be uniformly approximated by neural networks with TW activations.

\begin{theorem}[Universal Approximation for Functions \cite{ChenChen1995}]\label{theory_1}
Suppose that Assumption \ref{assumption:Main:assumptions:S2} holds and let $\sigma$ be a TW function. Then, for any \( \varepsilon > 0 \), there exist \( N \in \mathbb{N} \), \( \theta_i \in \mathbb{R} \), \( \omega_i \in \mathbb{R}^n \), and continuous linear functionals \( c_i:U \mapsto \bbR \) such that
\[
\left| f(x) - \sum_{i=1}^N c_i(f) \sigma(\omega_i \cdot x + \theta_i) \right| < \varepsilon
\]
holds for all \( x \in \Omega_U \) and \( f \in U \). 
\end{theorem}

Specifically, for a fixed $u \in U$, the mapping $G[u] \in V$ is a function $G[u]:\Omega_V \mapsto \bbR$, and Theorem~\ref{theory_1} stipulates the existence of functionals $\{c_i:V \to \bbR\}_{i=1}^N$ such that
\[
\big| G[u](x) - \sum_{i=1}^N c_i(G[u])\,\sigma(\omega_i \cdot x + \theta_i) \big| < \varepsilon.
\]

The next step is to approximate the continuous functionals $c_i:V \to \bbR$. To this end, one constructs a sequence of finite-dimensional subspaces $V_{\eta_k} \subseteq V$ that approximate $V$ increasingly well: for every $v \in V$ and $\delta > 0$, there exist $k \in \mathbb{N}$ and $v_k \in V_{\eta_k}$ such that $\|v - v_k\| < \delta$. By continuity, it follows that $|c_i(v) - c_i(v_k)| < \varepsilon$. Moreover, on each finite-dimensional subspace $V_{\eta_k}$, the functional $c_i$ can be identified with a function $\hat{c}_i:\bbR^{\dim(V_{\eta_k})} \to \bbR$, which can itself be approximated by Theorem~\ref{theory_1}. Denoting this approximation by $N(v_k)$, the triangle inequality yields
\[
|c_i(v) - N(v_k)| \leq |c_i(v) - c_i(v_k)| + |\hat{c}_i(v_k) - N(v_k)|,
\]
which completes the argument. Below, we describe the construction and approximation properties of the subsets $V_k$ which we also use in the proof of Theorem \ref{thm:mainResult:universalApproximationI}.

 First, we recall that if $V$ is a compact subset of $\Ck{0}(\Omega_V)$ where $\Omega_V$ is itself compact, then it is uniformly bounded and equicontinuous by the Arzel{\`a}-Ascoli theorem. Therefore, there is a decreasing sequence $\eta_1>\eta_2>\cdots>\eta_n\to 0$ and $\delta_1>\delta_2>\cdots>\delta_n\to 0$ such that if $\|x-y\|<\eta_k$,  then \begin{equation}\label{delta_k}
        |v(x)-v(y) |<\delta_k
    \end{equation} for all $v\in V$. Then, by the compactness of $\Omega_V$ and induction, we can find a sequence $\{x_i\}_{i=1}^\infty \subseteq \Omega_V$ and a sequence of positive integers $n(\eta_1)<n(\eta_2)<\cdots <n(\eta_k)\to\infty$, such that the first $n(\eta_k)$ elements \begin{equation}\label{etaknet}
        N(\eta_k) =\{x_1,\cdots,x_{n_{\eta_k}}\}
    \end{equation} is an $\eta_k$-net in $\Omega_V$. % and $n(\eta_k):=\vert N(\eta_k) \vert$.

For each $\eta_k$-net and index $1 \leq j \leq n(\eta_k)$, we define functions 
    \[
T_{\eta_{k},j}^* (x) = \begin{cases} 
1 - \frac{\| x - x_j \|}{\eta_k} & \text{if } \| x - x_j \| \leq \eta_k \\ 
0 & \text{otherwise}
\end{cases} \qquad \text{and} \qquad T_{\eta_{k},j} (x) = \frac{T_{\eta_{k},j}^* (x)}{\sum_{j=1}\limits^{n(\eta_k)} T_{\eta_{k},j}^* (x)}.
\]
Note that $\{T_{\eta_{k},j} (x)\}_{j=1}^{n(\eta_k)}$ is a partition of unity, i.e., $0 \leq T_{\eta_{k},j} (x) \leq 1$, $\sum_{j=1}\limits^{n(\eta_k)} T_{\eta_{k},j} (x) \equiv 1$, and $T_{\eta_{k},j} (x) = 0 \text{ if } \| x - x_j \| > \eta_k.$ Furthermore, the functions $T_{\eta_{k},j} (x)$ act as basis elements of the finite-dimensional space $V_{\eta_k} = \{v_{\eta_k}: v \in V\}$ where, for each $v \in V$, $v_{\eta_k}$ is defined as
\begin{equation}\label{u_etak}
    v_{\eta_k} (x) := \sum_{j=1}^{n(\eta_k)} v(x_j) T_{\eta_{k},j} (x).
\end{equation}
Finally, we let $V^* = V \cup \left(\bigcup_{k=1}^{\infty} V_{\eta_k}\right)$. Equation~\eqref{u_etak} essentially maps a finite-dimensional encoding of $v$ back into a function $v_{\eta_k}$.  The approximation properties of $v_{\eta_k}$ are summarized in the next lemma.

\begin{lemma}[Finite-dimensional Approximations of Function Spaces \cite{ChenChen1995}]\label{lemma2}
Assume that $V$ is a compact subset of $\Ck{0}(\Omega_V)$ where $\Omega_V$ is itself compact.
\begin{enumerate}
    \item For each fixed $k$, $V_{\eta_k}$ is a compact set of dimension $n(\eta_k)$ in $C(\Omega_V)$.
    \item For every $v \in V$, there exists $v_{\eta_k}\in V_{\eta_k}$ with
    \[
    \| v - v_{\eta_k} \|_{C(\Omega_V)} < \delta_k.
    \]
    \item $V^*$ is a compact set in $C(\Omega_V)$.
\end{enumerate}
\end{lemma}

\subsection{Scaling laws for operator learning} \label{sec:back:scalingLaws}

In this section, we review the main ideas behind establishing scaling laws for (multiple) operator learning, i.e. obtaining rates of convergence for the approximation of operators using neural networks. In particular, we focus on the setting in \cite{liu2024neuralscalinglawsdeep} which underpins our analysis.

We start by defining the following class of neural networks. This class is both general and flexible, encompassing a wide family of architectures, and can be readily implemented using standard deep learning frameworks.

\begin{mydef}[Feedforward ReLU Network Class] \label{def:networkClass}
Let \( q : \mathbb{R}^{d_1} \to \mathbb{R} \) be a feedforward ReLU network defined as
\[
q(x) = W_L \cdot \mathrm{ReLU}\left(W_{L-1} \cdots \mathrm{ReLU}(W_1 x + b_1) + \cdots + b_{L-1} \right) + b_L,
\]
where \( W_\ell \) are weight matrices, \( b_\ell \) are bias vectors, and \( \mathrm{ReLU}(a) = \max\{a, 0\} \) is applied element-wise.

We define the class of such feedforward networks with ReLU activations:
\[
\cF_{\rm NN}(d_1, d_2, L, p, K, \kappa, R) = \left\{ [q_1, q_2, \dots, q_{d_2}]^\top \in \mathbb{R}^{d_2} \; \middle| \;
\begin{array}{l}
\text{each } q_k : \mathbb{R}^{d_1} \to \mathbb{R} \text{ has the above form with} \\
L \text{ layers, width bounded by } p, \\
\|q_k\|_{\Lp{\infty}} \leq R, \quad \|W_\ell\|_{\infty,\infty} \leq \kappa, \quad \|b_\ell\|_\infty \leq \kappa, \\
\sum_{\ell=1}^L \left( \|W_\ell\|_0 + \|b_\ell\|_0 \right) \leq K
\end{array}
\right\},
\]
where
\begin{itemize}
    \item \( \|q\|_{\Lp{\infty}} = \sup_{x \in \Omega} |q(x)| \),
    \item \( \|W_\ell\|_{\infty,\infty} = \max_{i,j} |[W_\ell]_{ij}| \),
    \item \( \|b_\ell\|_\infty = \max_i |[b_\ell]_i| \),
    \item \( \|\cdot\|_0 \) denotes the number of nonzero elements.
\end{itemize}
This network class consists of vector-valued functions with input dimension \( d_1 \), output dimension \( d_2 \), depth \( L \), width at most \( p \), at most \( K \) nonzero parameters, all bounded in magnitude by \( \kappa \), and uniformly bounded output norm by \( R \).
\end{mydef}

In analogy with our discussion in Section~\ref{sec:back:operator}, scaling laws are derived sequentially by fixing the inputs of the operator $G$ to be approximated. This requires quantitative approximation results for both functions and functionals using the network class of Definition~\ref{def:networkClass}, which can then be combined to obtain operator-level guarantees. Specifically, we will use the following result on function approximation in the proofs of Theorems \ref{thm:back:functionalApproximationLinfty}, \ref{thm:back:operatorApproximation} and \ref{thm:main:multipleOperatorApproximation}.

\begin{theorem}[Function Approximation 
\cite{liu2024neuralscalinglawsdeep}]\label{thm:back:functionApproximation}
Let $d_U>0$ be an integer, $\gamma_1,\beta_U,L_U>0$ be constants and assume that $U(d_U,\gamma_U,\beta_U,L_U)$ satisfies Assumption \ref{assumption:Main:assumptions:S4}. There exists some constant $C$ depending on $\gamma_U$ and $L_U$ such that the following holds. For any $\varepsilon>0$,
\begin{itemize}
    \item let $N = C\sqrt{d_U}\varepsilon^{-1}$ and let $\{c_k\}_{k=1}^{N^{d_U}}$ be a uniform grid on $\Omega_U$ with spacing $2\gamma_U/N$ along each dimension;
    \item consider the network architecture $\cF_{\rm NN}(d_U, 1, L, p, K, \kappa, R)$
with parameters scaling as
    \begin{align*}
&L = \mathcal{O}\left(d_U^2\log d_U+d_U^2\log(\varepsilon^{-1})\right),\quad p = \mathcal{O}(1),\quad K = \mathcal{O}\left(d_U^2\log d_U+d_U^2\log(\varepsilon^{-1})\right),\\
&\kappa=\mathcal{O}(d_U^{d_U/2+1}\varepsilon^{-d_U-1}),\qquad \qquad \quad R=1
    \end{align*}
where the constants hidden in $\mathcal{O}$ depend on $\gamma_U$ and $L_U$. 
\end{itemize}
Then, there exists networks $\{q_k\}_{k=1}^{N^{d_U}} \subset \cF_{\rm NN}(d_U, 1, L, p, K, \kappa, R)$ such that 
\begin{align*}
        \left\|u-\sum_{k=1}^{N^{d_U}} u(c_k) q_k\right\|_{\Lp{\infty}(\Omega_U)}\leq \varepsilon.
        %\label{equation_thm3_not_in_functional_form}
    \end{align*}
for any $u\in U$.  
\end{theorem}

In Section~\ref{sec:main}, we present slightly modified versions of the functional and operator scaling laws from \cite{liu2024neuralscalinglawsdeep}, which make certain constants in the approximating network explicit and play a central role in the proof of the multiple operator case, Theorem~\ref{thm:main:multipleOperatorApproximation}.

\section{Main results} \label{sec:main}

\subsection{Notation, Assumptions and Setting} %\label{sec:assumptions}

We denote the Lebesgue measure on $\bbR^n$ by $\lambda$ and write $|\Omega|$ for the Lebesgue measure of a set $\Omega$. We denote the set of continuous function over a set $K$ as $\Ck{0}(K)$ and the set of continuous maps from $U$ to $V$ as $\Ck{0}(U,V)$. For a vector $z$ and a matrix $Z$, we denote by $[z]_i$ and $[Z]_{ij}$ their $i$-th and $ij$-th element respectively.
We denote the ball of radius $\delta$ with center $x$ by $\mathcal{B}_\delta(x)$.

\subsubsection{Assumptions}

\begin{assumptions} We make the following assumptions on the activation functions used in the operator networks.

\begin{enumerate}[label=\textbf{A.\arabic*}]
    \item The activation function $\sigma$ is a Tauber-Wiener function. \label{assumption:Main:assumptions:A1}
    \item The activation function $\sigma$ is continuous and/or bounded. \label{assumption:Main:assumptions:A2}
\end{enumerate}
\end{assumptions}

\begin{assumptions} We make the following assumption on our spaces.

\begin{enumerate}[label=\textbf{S.\arabic*}]
\item The space of 
$W \subseteq \Ck{0}(\Omega_W)$ is a compact subspace where $\Omega_W$ is a compact subset of the Banach space $\mathcal{A}$.    \label{assumption:Main:assumptions:S1} 
\item  The space 
$U \subseteq \Ck{0}(\Omega_U)$ is a compact subspace where $\Omega_U$ is a compact subset of the Banach space $\mathcal{U}$. \label{assumption:Main:assumptions:S2}
\item The space 
$V$ is $\Ck{0}(\Omega_V)$ where $\Omega_V$ is a compact subset of $\mathbb{R}^n$. \label{assumption:Main:assumptions:S3}
\item The space $U(d_U,\gamma_U,L_U,\beta_U)$ is a function set such that \begin{enumerate}
    \item any function $u \in U$ is defined on $\Omega_U := [-\gamma_U,\gamma_U]^{d_U}$;
    \item for all functions $u \in U$ and $x,y \in \Omega_U$, we have \[
    \vert u(x) - u(y) \vert \leq L_U \vert x - y \vert;
    \]
    \item for all functions $u \in U$, we have $\Vert u \Vert_{\Lp{\infty}} \leq \beta_U$.
\end{enumerate}\label{assumption:Main:assumptions:S4}
\end{enumerate}
\end{assumptions}

\begin{assumptions} We make the following assumption on the measures.

\begin{enumerate}[label=\textbf{M.\arabic*}]
\item $\nu$ is a probability measure on 
$W$. \label{assumption:Main:assumptions:M1} 
\item  $\mu$ is a probability measure on 
$U$. \label{assumption:Main:assumptions:M2} 
\end{enumerate}
\end{assumptions}

\begin{assumptions} We make the following assumption on the operators.

\begin{enumerate}[label=\textbf{O.\arabic*}]
\item For every 
$\alpha \in W$, the operator $G[\alpha]:U \mapsto  V$ is nonlinear and continuous.    \label{assumption:Main:assumptions:O1} 
\item  The map $\alpha \in W \mapsto G[\alpha]$ is continuous. \label{assumption:Main:assumptions:O2}
\item The map $\alpha \in W \mapsto G[\alpha]$ is Borel measurable and $G[\alpha][u](x) \in \Lp{2}_{\nu \times \mu \times \lambda}(W \times U \times V)$ \label{assumption:Main:assumptions:O3}
\end{enumerate}
\end{assumptions}

\subsubsection{Multiple Operator Network Architectures} \label{sec:main:architectures}

We introduce two neural network architectures for approximating multi-operator mappings. First, we consider the Multiple Operator Network (MONet) which is a direct extension of the neural network in Eq. \eqref{eq:chenchen} and we will show that it enjoys universal approximation properties for continuous and measurable multiple operators mappings in Theorems \ref{thm:mainResult:universalApproximationI} and \ref{thm:mainResult:universalApproximationII}.

\begin{mydef}[$\UAPNetwork$ Network] \label{def:uapNetwork}
     For fixed  positive integers \( M, N, P, m ,p\), constants  \( c_{kij}\), \(\zeta_k\), \(\xi_{kil}\), \(\varphi_{kijh}\), \(\rho_{kij}\), \(\theta_{ki}\in \mathbb{R} \),  points \( \omega_k \in \mathbb{R}^n \), \( x_l \in \Omega_U, z_h\in \Omega_W \) (\( i = 1, \dots, M \); \( k = 1, \dots, N \); \( j = 1, \dots, P \); \( h = 1,\dots, p\); \(l=1,\dots,m \)), we define a $\UAPNetwork$ network as:
     \begin{align}\label{eq:mole1}
        \UAPNetwork[\alpha][u](x) &= \sum_{k=1}^N \sum_{i=1}^M \tau_k(x) b_{ki}(u) L_{ki}(\alpha) = \sum_{k=1}^N  \tau_k(x)  \sum_{i=1}^M b_{ki}(u) L_{ki}(\alpha)
     \end{align}
     for continuous functions $\alpha\in W:\Omega_W \mapsto \bbR$ and $u\in U: \Omega_U \mapsto \bbR $, $x\in\mathbb{R}^n$, activation function $ \sigma \in TW$ and networks $\tau_k(x) = \sigma(\omega_k \cdot x + \zeta_k)$, $b_{ki}(u) = \sigma\l \sum_{l = 1}^{m} \xi_{kil} u(x_l) + \theta_{ki} \r$ and \[L_{ki}(\alpha) = \sum_{j=1}^P c_{kij} \sigma\l \sum_{h=1}^p \varphi_{kijh} \alpha(z_h) + \rho_{kij} \r.\] 
\end{mydef}

Note that the notation used in the operator networks assumes that the inputs are continuous functions; however, the input functions are encoded as finite dimensional vectors by first being evaluated on some points. In some sense, the resulting encoded vectors are the true inputs to the networks. In all the proofs describing a specific architecture as in the ones of Theorem \ref{thm:mainResult:universalApproximationI}, Lemma \ref{lem:orthogonal}, Theorem \ref{thm:back:operatorApproximation} and Theorem \ref{thm:main:multipleOperatorApproximation}, precise statements will be made.

\begin{remark}[$\UAPNetworkFinite$ Network]
    The network in Eq. \eqref{eq:mole1} can be simplified in the case where the parameter inputs are finite-dimensional, i.e., $\alpha \in \bbR^p$. 
     For fixed  positive integers \( M, N, P, m ,p\), constants \( c_{kij}\), \(\zeta_k\), \(\xi_{kil}\), \(\varphi_{kijh}\), \(\rho_{kij}\), \(\theta_{ki}\in \mathbb{R} \), points \( \omega_k \in \mathbb{R}^n \), \( x_l \in \Omega_U \) (\( i = 1, \dots, M \); \( k = 1, \dots, N \); \( j = 1, \dots, P \); \( h = 1,\dots, p\);  \(l=1,\dots,m \)), we define a $\UAPNetworkFinite$ network with vector parameter $\alpha \in \bbR^p$ as
     \begin{equation}\label{eq:mole2}
        \UAPNetworkFinite[\alpha][u](x)=\sum_{k=1}^N \sum_{i=1}^M \sum_{j=1}^P c_{kij} \,  \sigma \left( \sum_{h=1}^p \varphi_{kijh} [\alpha]_h + \rho_{kij} \right) \cdot \sigma \left( \sum_{l=1}^m \xi_{kil} u(x_l) + \theta_{ki} \right) \cdot \sigma(\omega_k \cdot x + \zeta_k)
     \end{equation}
      for a continuous function $u: \Omega_U \mapsto \bbR $, point $x\in\mathbb{R}^n$ and some activation function $ \sigma \in TW$.
      The proof of the universal approximation for finite dimensional $\alpha$ is given in Corollary \ref{theory_4c}.
\end{remark}

We introduce the Multiple Nonlinear Operator (MNO) Network, which is shown to provide strong empirical results in Section~\ref{sec:experiments}. We establish scaling laws in Theorem \ref{thm:main:multipleOperatorApproximation} for this architecture.

\begin{mydef}[$\ScalingNetwork$ Network] \label{def:scalingNetwork}
For fixed  positive integers \(P, H^{(p)}\), $1 \leq p \leq P$, we define a $\ScalingNetwork$ network as \[
  \mathrm{MNO}[\alpha][u](x) =    \sum_{p=1}^{P} l_p(\alpha) \sum_{k=1}^{H^{(p)}} b_{pk}(u) \tau_{pk}(x)=    \sum_{p=1}^{P} \sum_{k=1}^{H^{(p)}} l_p(\alpha) b_{pk}(u) \tau_{pk}(x)
     \]
     for continuous functions $\alpha$, $u$, and networks $l_p$, $b_k$, $\tau_{pk}$ in some classes $\cF_{\rm NN}$.  
\end{mydef}

We summarize both network architectures and their associated expressivity guarantees in Table \ref{tab:uap-vs-scaling}.

\begin{table}[h]
\centering
\small
\renewcommand{\arraystretch}{2}
\setlength{\tabcolsep}{6pt}
\begin{tabularx}{\linewidth}{>{\bfseries}p{2.9cm} Y Y}
\toprule
& $\UAPNetwork$ & $\ScalingNetwork$ \\
\midrule
Definition
& Definition ~\ref{def:uapNetwork}
& Definition ~\ref{def:scalingNetwork} \\

Expression
& $\displaystyle \sum_{k=1}^{N}\sum_{i=1}^{M} \tau_k(x)\, b_{ki}(u)\, L_{ki}(\alpha)$
& $\displaystyle \sum_{p=1}^{P}\sum_{k=1}^{H^{(p)}} l_p(\alpha)\, b_{pk}(u)\, \tau_{pk}(x)$ \\

Components
& $\tau_k$, $b_{ki}$, $L_{ki}$ are shallow networks
& $l_p$, $b_{pk}$, $\tau_{pk}$ are deep networks in $\mathcal{F}_{\mathrm{NN}}$ \\

UAP
& Approximates continuous and measurable multiple operator mappings
  (Theorems ~\ref{thm:mainResult:universalApproximationI} and \ref{thm:mainResult:universalApproximationII})
& Approximates continuous and measurable multiple operator mappings
  (Theorems ~\ref{thm:mainResult:universalApproximationI} and \ref{thm:mainResult:universalApproximationII}) \\

Scaling laws
& — 
& Quantitative approximation for Lipschitz multiple operator mappings
  (Theorem ~\ref{thm:main:multipleOperatorApproximation}) \\
\bottomrule
\end{tabularx}
\caption{Comparison of $\ScalingNetwork$ and $\UAPNetwork$ architectures: definition, expression, component type, universal approximation properties (UAP), and scaling laws.}
\label{tab:uap-vs-scaling}
\end{table}

\subsection{Main results}

\subsubsection{Universal Approximation}
In this section, we show that both network architectures introduced in Section \ref{sec:main:architectures} can approximate families of nonlinear operators. The first result is analogous to classical universal approximation results for neural networks. In particular, it assumes that all our functions are continuous. In the theorems below, $\nn$ refers to both $\ScalingNetwork$ and $\UAPNetwork$.

\begin{theorem}[Universal Approximation Theorem for Multiple Operators in $\Lp{\infty}$] \label{thm:mainResult:universalApproximationI}
Assume that Assumptions \ref{assumption:Main:assumptions:A1}, \ref{assumption:Main:assumptions:S1}, \ref{assumption:Main:assumptions:S2}, \ref{assumption:Main:assumptions:S3}, \ref{assumption:Main:assumptions:O1} and \ref{assumption:Main:assumptions:O2} hold.
Then for any \( \varepsilon > 0 \),  there exists a network $\nn$, of the form given in Definition~\ref{def:uapNetwork} or Definition~\ref{def:scalingNetwork} (with $l_p$, $b_{pk}$, and $\tau_{pk}$ being defined as in Definition~\ref{def:uapNetwork}), such that
\begin{equation}\label{eq:thm4}
        \left\| G[\alpha][u](x) - \nn[\bm{\alpha}][\bm{u}](x)\right\|_{\Lp{\infty}(W \times U \times \Omega_V)} < \varepsilon
\end{equation}
for all functions $\alpha \in W$ and $u \in U$ and where $\bm{\alpha}$ and $\ub$ are discretizations of $\alpha$ and $u$ that only depend on $W$ and $U$ respectively. 
\end{theorem}
The proof is provided in Section~\ref{sec_mainResult_universalApproximationI}. Next, we relax the continuity requirement on the map $\alpha \mapsto G[\alpha]$, extending the result from continuous to measurable operator families. In this more general setting, the approximation is obtained in the $\Lp{2}$-norm rather than the $\Lp{\infty}$-norm.

\begin{theorem}[Universal Approximation Theorem for Multiple Operators in $\Lp{2}$] \label{thm:mainResult:universalApproximationII}
    Assume that Assumptions \ref{assumption:Main:assumptions:A1}, \ref{assumption:Main:assumptions:A2}, \ref{assumption:Main:assumptions:S1}, \ref{assumption:Main:assumptions:S2}, \ref{assumption:Main:assumptions:S3}, \ref{assumption:Main:assumptions:M1}, \ref{assumption:Main:assumptions:M2}, \ref{assumption:Main:assumptions:O1} and \ref{assumption:Main:assumptions:O3} hold. 
    Then, for every $\eps > 0$, there exists a network $\nn$, of the form given in Definition~\ref{def:uapNetwork} (with $L_{ki}(\alpha) = \gamma_{ki} \l \sum_{j=1}^P c_{kij} \sigma\l \sum_{h=1}^p \varphi_{kijh} \alpha(z_h) + \rho_{kij} \r \r$ where $\gamma_{ki}$ are ReLu neural networks) or Definition~\ref{def:scalingNetwork} (with $l_p$, $b_{pk}$, and $\tau_{pk}$ being defined as in Definition~\ref{def:uapNetwork}), such that
    \begin{equation} \label{eq::mainResult:universalApproximationII}
         \left\| G[\alpha][u](x) - \nn[\bm{\alpha}][\bm{u}](x)\right\|_{\Lp{2}_{\nu \times \mu \times \lambda}(W \times U \times \Omega_V)} < \varepsilon
    \end{equation}
for any functions $\alpha \in W$ and $u \in U$ and where $\bm{\alpha}$ and $\ub$ are discretizations of $\alpha$ and $u$ that only depend on $W$ and $U$ respectively. 
\end{theorem}
The proof is given in Section \ref{sec_mainResult_universalApproximationII}.

\subsubsection{Scaling Laws}

In this section, we establish the scaling laws for the $\ScalingNetwork$ architecture. Our strategy is indirect: we first carry out the analysis for an equivalent, but more explicit, architecture in Theorem \ref{thm:main:multipleOperatorApproximation}. The scaling laws for $\ScalingNetwork$ then follow as a corollary through suitable reformulations as detailed in Remark \ref{rem:mnoScalingLaws}. Moreover, the results derived for the setting of learning several single operators emerge naturally from this analysis. Our analysis proceeds hierarchically, beginning at the functional level, extending to the approximation of (several) single operators, and then yielding the final general multiple operator learning result.

As discussed in Section \ref{sec:back:scalingLaws}, we start by considering a revised version of \cite[Theorem 6]{liu2024neuralscalinglawsdeep} for quantitative functional approximation through neural networks. For the sake of completeness, in Section \ref{sec_backfunctionalApproximationLinfty}, we provide a modified proof which explicitly determines the values of some of the constants in the approximating network architecture.

\begin{theorem}[Functional Approximation]\label{thm:back:functionalApproximationLinfty}
    Let $d_U>0$ be an integer, $\gamma_U,\beta_U,L_U,L_f>0$ be constants and assume that 
    $U(d_U,\gamma_U,L_U,\beta_U)$ 
    satisfies Assumption \ref{assumption:Main:assumptions:S4}. 
    Let $f:\{ u:\Omega_U \mapsto \bbR \spaceBar \Vert u \Vert_{\Lp{\infty}} \leq \beta_U \} \mapsto \bbR$
    be a functional such that \[
    \vert f(u_1) - f(u_2) \vert \leq L_f \Vert v_1 - v_2 \Vert_{\Lp{\infty}}
    \]
    for all $u_1,u_2 \in \{u: \Omega_U \spaceBar \Vert u \Vert_{\Lp{\infty}}(\Omega_U) \leq \beta_U \}$.
    
    There exists constants $C$ and $C_{\delta}$ depending on $\beta_U, L_f$ and $L_f,L_U$ respectively such that the following holds. For any $\varepsilon>0$, \begin{itemize}
        \item let $\delta=C_{\delta}\varepsilon$ and let $\{c_m\}_{m=1}^{n_{c_U}}\subset \Omega_U$ be points so that $\{\mathcal{B}_{\delta}(c_m) \}_{ m  = 1}^{n_{c_U}}$ is a cover of $\Omega_U$ for some $n_{c_U}$;
        \item let $H = C \sqrt{n_{c_U}} \eps^{-1}$ and consider the network class $\cF_{\rm NN}(n_{c_U}, 1 , L, p, K, \kappa, R)$ 
     with parameters scaling as
     \begin{align*}
     &L=\mathcal{O}\left(n_{c_U}^2\log(n_{c_U})+n_{c_U}^2\log(\varepsilon^{-1})\right),\quad  p = \mathcal{O}(1),\quad K = \mathcal{O}\left(n_{c_U}^2\log n_{c_U}+n_{c_U}^2\log(\varepsilon^{-1})\right), \\ &\kappa=\mathcal{O}(n_{c_U}^{n_{c_U}/2+1}\varepsilon^{-n_{c_U}-1}),\qquad \, R=1
    \end{align*}
    where the constants hidden in $\mathcal{O}$ depend on $\beta_U$ and $L_f$.
    \end{itemize} 
    Then, there exists networks
     $\{b_k\}_{k=1}^{H^{n_{c_U}}} \subset  \cF_{\rm NN}(n_{c_U}, 1, L, p, K, \kappa,R)$ and functions $\{u_k\}_{k=1}^{H^{n_{c_U}}} \subset \{ u:\Omega_U \mapsto \bbR \spaceBar \Vert u \Vert_{\Lp{\infty}} \leq \beta_U \}$ such that 
     \begin{align}
        \sup_{u\in U}\left|f(u)-\sum_{k=1}^{H^{n_{c_U}}} f(u_k) b_k(\bm{u})\right|\leq \varepsilon,
        \label{eq.functional.1}
    \end{align}
    where $\ub=(u(c_1), u(c_2),\dots,u(c_{n_{c_U}}))^\top$.
\end{theorem}

\begin{remark}[Uniform functional approximation]\label{rem:uniformFunctionalApproximation}
We can extend Theorem \ref{thm:back:functionalApproximationLinfty} to a set of functionals 
\[
    \{f^{(j)}:\{u:\Omega_{U^{(j)}} \mapsto \bbR \spaceBar \Vert u \Vert_{\Lp{\infty}} \leq \beta_U^{(j)} \} \mapsto \bbR \spaceBar \text{
    $\vert f^{(j)}(u_1) - f^{(j)}(u_2)\vert \leq L^{(j)} \Vert u_1 - u_2 \Vert_{\Lp{\infty}} $}  \}_{j\in \mathcal{J}}
    \]
where $\mathcal{J}$ is a (possibly uncountable) index set. For simplicity, we assume that $d_{U^{(j)}} = d_{U}$ for all $j \in \mathcal{J}$ and define \[H^{(j)} = H^{n_{c_{U^{(j)}}}}.\]

\paragraph{Case I: $U^{(j)}$ are distinct}
In the case of distinct $U^{(j)}$, we apply Theorem \ref{thm:back:functionalApproximationLinfty} for every $j$ separately and, for every $j$, obtain \eqref{eq.functional.1}. Then, we take the supremum over $j$ and have 
\begin{equation}\label{eq:rem:multipleFunctionalDifferentDomain}
\sup_{j \in \mathcal{J}} \sup_{u \in U^{(j)}} \left|f^{(j)}(u)-\sum_{k=1}^{H^{(j)}} f^{(j)}(u_k^{(j)}) b_k^{(j)}(\bm{u}^{(j)})\right|\leq \varepsilon
\end{equation} 
where $\{b_k^{(j)}\}_{k=1}^{H^{(j)}}$ are networks with $ b_k^{(j)} \in \cF_{\rm NN}(n_{c_{U^{(j)}}}, 1, L^{(j)}, p^{(j)}, K^{(j)}, \kappa^{(j)},R^{(j)})$ for any $1 \leq k \leq H^{(j)}$, $\{u_k^{(j)}\}_{k=1}^{H^{(j)}}$ are functions in  $\{ u:\Omega_{U^{(j)}} \mapsto \bbR \spaceBar \Vert u \Vert_{\Lp{\infty}} \leq \beta_U^{(j)} \}$ and $\ub^{(j)}=(u(c_1^{(j)}), u(c_2^{(j)}),\dots,u(c_{n_{c_{U^{(j)}}}}^{(j)}))^{T}$. All the constants with superscript $j$ are analogous to the ones defined in the statement of Theorem \ref{thm:back:functionalApproximationLinfty} using the appropriated quantities related to $U^{(j)}$.

\paragraph{Case II: $U^{(j)} = U$}
If we have $U^{(j)} = U$ for all $j$, the above simplifies. By inspecting the proof of Theorem \ref{thm:back:functionalApproximationLinfty}, we note that our functional approximation relies on the function approximation Theorem \ref{thm:back:functionApproximation}. In particular, the idea is to transform our functional $f:\{u:\Omega_U \mapsto \bbR\ \spaceBar \Vert u \Vert_{\Lp{\infty}} \leq \beta_U \} \mapsto \bbR$ into a Lipschitz function $\hat{f}:[-\beta_U,\beta_U]^{n_{c_U}} \mapsto \bbR$ contained in some class $V(n_{c_U},\beta_U,L_f,C_{\hat{f}})$. Then, we obtain the approximation result \[
    \sup_{x \in [-\beta_U,\beta_U]^{n_{C_U}}} \left\vert \hat{f}(x) - \sum_{k=1}^{H^{n_{c_U}}} \hat{f}(s_k) b_k(x)  \right\vert \leq \frac{\eps}{2}
    \]
    where the same networks $b_k$ and points $s_k$ can be chosen for any function in the class $V$. In particular, the only parameters in the class of functions $V$ and in the second part of the approximation (Eq. \eqref{eq:functionalApproximation:LipschitzBound}) that depend on $f$ are $L_f$ and $C_{\hat{f}}$.
    Therefore, if we consider a set of functionals
    \[
    \{f^{(j)}:\{u:\Omega_U \mapsto \bbR \spaceBar \Vert u \Vert_{\Lp{\infty}} \leq \beta_U \} \mapsto \bbR \spaceBar \text{
    %$\sup_{u} \vert f_j(u) \vert \leq C$ and
    $\vert f^{(j)}(u_1) - f^{(j)}(u_2)\vert \leq L_j \Vert u_1 - u_2 \Vert_{\Lp{\infty}} $}  \}_{j\in \mathcal{J}},
    \]
    the same argument can be repeated by replacing $L_f$ by $\sup_{j \in \mathcal{J}}L_j$ and $C_{\hat{f}}$ by $\sup_{j \in \mathcal{J}} C_{\hat{f}_j}$: 
     in fact, $\hat{f}^{(j)} \in V(n_{c_U},\beta_U,\sup_{j \in \mathcal{J}}L_j,\sup_{j \in \mathcal{J}} C_{\hat{f}_j})$. 
    We can conclude that
    \begin{equation} \label{eq:rem:multipleFunctionalSameDomain}
    \sup_{j\in \mathcal{J}} \sup_{u \in U} \left|f^{(j)}(u)-\sum_{k=1}^{H^{n_{c_U}}} f^{(j)}(u_k) b_k(\bm{u})\right|\leq \varepsilon.
        \end{equation}
    This result will only affect the choice of the constants in the statement of the theorem, none of the scalings. Obviously, this presupposes that $\sup_{j \in \mathcal{J}}L_j < \infty $ and $\sup_{j \in \mathcal{J}} C_{\hat{f}_j} < \infty$. In particular, we note that we can set $\sup_{j \in \mathcal{J}} C_{\hat{f}_j} = \sup_{j \in \mathcal{J}} \sup_{u \in \{u:\Omega_{U} \mapsto \bbR \spaceBar \Vert u \Vert_{\Lp{\infty}} \leq \beta_U^{(j)}\} } f^{(j)}(u)$. 

We also note that the case where some $U^{(j)}$ are distinct and some coincide is dealt with similarly, as a combination of Eqs. \eqref{eq:rem:multipleFunctionalDifferentDomain} and \eqref{eq:rem:multipleFunctionalSameDomain}. 
    
\end{remark}

With the explicit functional approximation provided by Theorem~\ref{thm:back:functionalApproximationLinfty}, we now establish a version of the operator scaling laws that includes explicit coefficients for the approximating network. The proof of the theorem is analogous to \cite[Theorem 1]{liu2024neuralscalinglawsdeep}, just substituting Theorem \ref{thm:back:functionalApproximationLinfty} for \cite[Theorem 6]{liu2024neuralscalinglawsdeep}. We recall the main steps in  Remark \ref{rem:uniformOperatorApproximation} and prove a very similar statement in Remark \ref{rem:balancingComplexity}.

\begin{theorem}[Single Operator Scaling Laws]\label{thm:back:operatorApproximation}
 Let $d_U,d_V>0$ be integers, $\gamma_U,\gamma_V,\beta_U,\beta_V, L_U,L_V,L_G>0$, and assume that $U(d_U,\gamma_U,L_U,\beta_U)$ and $V(d_V,\gamma_V,L_V,\beta_V)$ satisfy Assumption \ref{assumption:Main:assumptions:S4}. 
Let $G$ be an operator such that $G:\{ u:\Omega_U \mapsto \bbR \spaceBar \Vert u \Vert_{\Lp{\infty}} \leq \beta_U \} \mapsto V$. Furthermore, assume that $G$ satisfies \begin{equation} \label{eq:operatorApproximation:Lipschitz}
   \Vert G(u_1) - G(u_2) \Vert_{\Lp{\infty}(\Omega_V)} \leq L_G \Vert u_1 - u_2 \Vert_{\Lp{r}(\Omega_U)} 
\end{equation}

for some $r \geq 1$ and for any $u_1,u_2 \in \{ u:\Omega_U \mapsto \bbR \spaceBar \Vert u \Vert_{\Lp{\infty}} \leq \beta_U \}$. 

There exists constants $C$ depending on $\gamma_V,L_V$, $C_{\delta}$ depending on $L_G,d_U,\gamma_U,r,L_U$ and $C'$ depending on $\beta_U, L_G, d_U, \gamma_U, r$ such that the following holds. For any $\varepsilon>0$, \begin{itemize}
     
     \item let $N=2C \sqrt{d_V} \varepsilon^{-1}$ and consider the network class $\cF_1=\cF_{\rm NN}(d_V,1,L_1,p_1,K_1,\kappa_1,R_1)$  with parameters scaling as
    \begin{align*}
    &L_1 = \mathcal{O}\left(d_V^2\log d_V+d_V^2\log(\varepsilon^{-1})\right),\quad  p_1 = \mathcal{O}(1), \quad K_1 = \mathcal{O}\left(d_V^2\log d_V+d_V^2\log(\varepsilon^{-1})\right),\\
&\kappa_1=\mathcal{O}(d_V^{d_V/2+1}\varepsilon^{-d_V-1}),\qquad \qquad \quad R_1=1.
    \end{align*}
    where the constants hidden in $\mathcal{O}$ depend on $\gamma_V$ and $L_V$;

    \item let $\{v_\ell\}_{\ell=1}^{N^{d_V}} \subset \Omega_V$ be a uniform grid with spacing $2\gamma_V/N$ along each dimension;
    
    \item let $\delta=\frac{C_{\delta}\varepsilon^{1+d_V}}{2^{d_V+1}(C\sqrt{d_V})^{d_V}}$ and let $\{c_m\}_{m=1}^{n_{c_U}}\subset \Omega_U$ be points so that $\{\mathcal{B}_{\delta}(c_m) \}_{ m  = 1}^{n_{c_U}}$ is a cover of $\Omega_U$ for some $n_{c_U}$;
    
    \item let $H= 2^{d_V+1}C' \sqrt{n_{c_U}} (C \sqrt{d_V})^{d_V} \varepsilon^{-(d_V+1)}$ and consider the network class \newline$\cF_2=\cF_{\rm NN}(n_{c_U},1,L_2,p_2,K_2,\kappa_2,R_2)$ with parameters scaling as
\begin{align*}
        &L_2 = \mathcal{O}\left(n_{c_U}^2\log n_{c_U}+n_{c_U}^2(d_V +1)\log(\varepsilon^{-1}) + n_{c_U}^2 \log(2^{d_V +1} (C \sqrt{d_V})^{d_V} ) \right),\quad p_2=\mathcal{O}(1), \\
&K_2 = \mathcal{O}\left(n_{c_U}^2\log n_{c_U}+n_{c_U}^2(d_V +1)\log(\varepsilon^{-1})+n_{c_U}^2 \log( 2^{d_V +1} (C \sqrt{d_V})^{d_V})\right), \\ 
&\kappa_2=\mathcal{O}(n_{c_U}^{n_{c_U}/2+1}\varepsilon^{-(d_V+1)(n_{c_U}+1)}[ 2^{d_V +1} (C \sqrt{d_V})^{d_V} ]^{n_{c_U}+1}), \qquad R_2=1
    \end{align*}
    where the constants hidden in $\mathcal{O}$ depend on $\beta_U, L_G, d_U, \gamma_U,r$.
 \end{itemize} 
 Then, there exists networks $\{\tau_\ell\}_{\ell=1}^{N^{d_V}} \subset \cF_1$, networks $\{b_k\}_{k=1}^{H^{n_{c_U}}} \subset \mathcal{F}_2$ and functions $\{u_k\}_{k=1}^{H^{n_{c_U}}} \subset \{ u:\Omega_U \mapsto \bbR \spaceBar \Vert u \Vert_{\Lp{\infty}} \leq \beta_U \}$ such that 
\begin{align}
        \sup_{u\in U}\sup_{x \in \Omega_V}\left|G[u](x)- \sum_{\ell=1}^{N^{d_V}} \sum_{k=1}^{H^{n_{c_U}}} G[u_k](v_{\ell}) b_k(\ub) \tau_{\ell}(x)\right|\leq \varepsilon, \label{eq:back:operatorApproximation}
    \end{align}
    where  $\ub=(u(c_1), u(c_2),...,u(c_{n_{c_U}}))^\top$ is a discretization of $u$.
\end{theorem}

\begin{remark}[Uniform operator approximation]\label{rem:uniformOperatorApproximation}
Similarly to Remark \ref{rem:uniformFunctionalApproximation}, we extend Theorem \ref{thm:back:operatorApproximation} to a set of operators 
\[\{G^{(j)}:\{ u:\Omega_{U^{(j)}} \mapsto \bbR \spaceBar \Vert u \Vert_{\Lp{\infty}} \leq \beta_{U^{(j)}} \} \mapsto V^{(j)} \spaceBar \Vert G^{(j)}(u_1) - G^{(j)}(u_2) \Vert_{\Lp{\infty}} \leq L_{G^{(j)}} \Vert u_1 - u_2 \Vert_{\Lp{r^{(j)}}}\}_{j\in \mathcal{J}}
\]
where $\mathcal{J}$ is a (possibly uncountable) index set. 
For simplicity, we assume that $d_{U^{(j)}} = d_{U}$ and $d_{V^{(j)}} = d_{V}$  for all $j \in \mathcal{J}$ and define \[H^{(j)} = H^{n_{c_{U^{(j)}}}}.\] 

\paragraph{Case I: $U^{(j)}$ are distinct and $V^{(j)}$ are distinct}
If the $U^{(j)}$ and $V^{(j)}$ are distinct, we apply Theorem \ref{thm:back:operatorApproximation} first for each $j \in \mathcal{J}$ separately and then take the supremum over all $j$ to obtain: \[
\sup_{j \in \mathcal{J}} \sup_{u\in U^{(j)}}\sup_{x \in \Omega_{V^{(j)}}}\left|G^{(j)}[u](x)- \sum_{\ell=1}^{(N^{(j)})^{d_{V}}} \sum_{k=1}^{H^{(j)}} G^{(j)}[u_k^{(j)}](v_{\ell}^{(\ell)}) b_k^{(j)}(\ub^{(j)}) \tau_{\ell}^{(j)}(x)\right|\leq \varepsilon.
\]

For the rest of the cases, we need to recall the proof of Theorem \ref{thm:back:operatorApproximation}. Specifically, the idea is first to consider, for $u \in U$, the functions $G[u]: \Omega_V \mapsto \bbR$. The latter are all contained in $V$ and we can therefore apply the function approximation Theorem \ref{thm:back:functionApproximation} to deduce that for all $u \in U$, \[
\sup_{x \in \Omega_V} \left| G[u](x) - \sum_{\ell=1}^{N^{d_V}} G[u](v_\ell) \tau_\ell(x) \right| \leq \frac{\eps}{2}. 
\]
Then, we define the functionals $f_\ell(u) = G[u](v_\ell)$ and verify that they are $\Lp{\infty}$-Lipschitz on $\{ u:\Omega_U \mapsto \bbR \spaceBar \Vert u \Vert_{\Lp{\infty}} \leq \beta_U \}$ with Lipschitz constant $L_G \vert\Omega_V\vert^{1/r}$. 
As explained in Remark \ref{rem:uniformFunctionalApproximation}, the proof corresponds to the setting where we have $N^{d_V}$ functionals all defined for the same set of functions $U$, we can apply the formula in Eq. \eqref{eq:rem:multipleFunctionalSameDomain} to obtain that, for all $1 \leq \ell \leq N^{d_V}$,
\[
\sup_{u \in U} \left\vert f_\ell(u) - \sum_{k=1}^{H^{n_{c_U}}} f_\ell(u_k) b_{k}(\ub) \right\vert = \sup_{u \in U} \left\vert G[u](v_\ell) - \sum_{k=1}^{H^{n_{c_U}}} G[u_k](v_\ell) b_{k}(\ub) \right\vert\leq \eps_0. 
\]
Combining both estimates, we conclude with \begin{align*}
    &\sup_{u\in U}\sup_{x \in \Omega_V}\left|G[u](x)- \sum_{\ell=1}^{N^{d_V}} \sum_{k=1}^{H^{n_{c_U}}} G[u_k](v_{\ell}) b_k(\ub^{(j)}) \tau_{\ell}(x)\right| \\
    &\leq \sup_{u\in U}\sup_{x \in \Omega_V}\left|G[u](x)- \sum_{\ell=1}^{N^{d_V}} G[u](v_{\ell}) \tau_{\ell}(x)\right| \, +\,  \sup_{u\in U}\sup_{x \in \Omega_V} \sum_{\ell=1}^{N^{d_V}} \vert \tau_{\ell}(x) \vert \left| G[u](v_\ell) -  \sum_{k=1}^{H^{n_{c_U}}} G[u_k](v_{\ell}) b_k(\ub^{(j)}) \right| \\ 
    &\leq \frac{\eps}{2} + \eps_0 N^{d_V}.
\end{align*}
By picking $\eps_0 = \eps/(2N^{d_V}) = \mathcal{O}(\eps^{d_V +1})$, we obtain the result.

\paragraph{Case II: $U^{(j)} =U$ and $V^{(j)}$ are distinct}
Let us now assume that $U^{(j)} = U$. The first step of the proof can be repeated for every $j$ separately to obtain that for every $u \in U$, \[
\sup_{j \in J} \sup_{x \in \Omega_{V^{(j)}}} \left| G^{(j)}[u](x) - \sum_{\ell=1}^{(N^{(j)})^{d_{V}}} G[u](v_\ell^{(j)}) \tau^{(j)}_\ell(x) \right| \leq \frac{\eps}{2}. 
\]
Next, we can define the functionals 
$$f^{(j)}_\ell(u) = G^{(j)}[u](v^{(j)}_\ell)$$ 
and the latter are $\Lp{\infty}$-Lipschitz in $\{ u:\Omega_{U} \mapsto \bbR \spaceBar \Vert u \Vert_{\Lp{\infty}} \leq \beta_{U} \}$ with Lipschitz constant $\sup_{j \in \mathcal{J}} \vert \Omega_U \vert^{1/r^{(j)}} L_{G^{(j)}}$ if we assume that the latter is finite. If we further assume that 
\[
\sup_{j \in \mathcal{J}} \sup_{1 \leq \ell \leq N^{d_{V^{(j)}}}} f^{(j)}_\ell(u) = \sup_{j \in \mathcal{J}} \sup_{1 \leq \ell \leq N^{d_{V^{(j)}}}} G^{(j)}[u](v^{(j)}_\ell) \leq \sup_{j \in \mathcal{J}} \sup_{v \in V^{(j)}} \Vert v \Vert_{\Lp{\infty}(V^{(j)})} \leq \sup_{j \in \mathcal{J}} \beta_{V^{(j)}} < \infty
\]
then, the functionals satisfy all the assumptions in Eq. \eqref{eq:rem:multipleFunctionalSameDomain} and we obtain
\begin{align*}
\sup_{j \in \mathcal{J}} \sup_{u \in U} \sup_{1\leq \ell \leq (N^{(j)})^{d_{V}}} \left|f^{(j)}_\ell(u)-\sum_{k=1}^{H^{n_{c_U}}} f^{(j)}_\ell(u_k) b_k(\bm{u})\right| &= \sup_{j \in \mathcal{J}} \sup_{u \in U} \left|G^{(j)}[u](x)-\sum_{k=1}^{H^{n_{c_U}}} G^{(j)}[u_k](x) b_k(\bm{u})\right| \\ 
&\leq \frac{\varepsilon}{2\sup_{j\in \mathcal{J}} (N^{(j)})^{d_V}} =: \eps_0.
\end{align*}
This also requires that $\sup_{j\in \mathcal{J}} (N^{(j)})^{d_V} < \infty$ and we note a subtle point: in this setting, $b_k$, $u_k$ and $\ub$ can be chosen independently of $j$. The is possible since the latter are a function of $\varepsilon_0$ which is set to $\frac{\varepsilon}{2\sup_{j\in \mathcal{J}} (N^{(j)})^{d_V}}$, i.e. independent of $j$, and not to $\frac{\varepsilon}{2 (N^{(j)})^{d_V}}$, in which case they would both become dependent of $j$ again. By concluding as above, we obtain: 
\[
\sup_{j \in \mathcal{J}} \sup_{u\in U} \sup_{x \in \Omega_{V^{(j)}}}\left|G^{(j)}[u](x)- \sum_{\ell=1}^{(N^{(j)})^{d_{V}}} \sum_{k=1}^{H^{n_{c_U}}} G^{(j)}[u_k](v_{\ell}^{(j)}) b_k(\ub) \tau_{\ell}^{(j)}(x)\right|\leq \varepsilon.
\]

\paragraph{Case III: $U^{(j)}$ are distinct and $V^{(j)} = V$}
Next, we assume that $V_j = V$. Since for all $j \in \mathcal{J}$ and $u^{(j)} \in U^{(j)}$, we obtain that $G^{(j)}[u^{(j)}] \in V$, by repeating the first step of the proof, we can choose $\tau_\ell$ and $v_\ell$ to be independent of $j$ and obtain
\[
\sup_{j \in J} \sup_{u^{(j)} \in U^{(j)}} \sup_{x \in \Omega_{V}} \left| G^{(j)}[u^{(j)}](x) - \sum_{\ell=1}^{N^{d_{V}}} G^{(j)}[u^{(j)}](v_\ell) \tau_\ell(x) \right| \leq \frac{\eps}{2}. 
\]
We then define the functionals $f_\ell^{(j)}:\{ u^{(j)}:\Omega_{U^{(j)}} \mapsto \bbR \spaceBar \Vert u^{(j)} \Vert_{\Lp{\infty}} \leq \beta_{U^{(j)}} \}$ as $f_\ell^{(j)}(u^{(j)}) = G^{(j)}[u^{(j)}](v_\ell)$ and verify that they are $\Lp{\infty}$-Lipschitz with Lipschitz constant $\vert \Omega_{U^{(j)}} \vert^{1/r^{(j)}} L_G^{(j)}$. 
We then apply Eq. \eqref{eq:rem:multipleFunctionalDifferentDomain} and obtain that \begin{align*}
&\sup_{j \in \mathcal{J}} \sup_{u^{(j)} \in U^{(j)}} \sup_{1 \leq \ell \leq N^{d_V}} \left|f_\ell^{(j)}(u^{(j)})-\sum_{k=1}^{H^{(j)}} f_{\ell}^{(j)}(u_k^{(j)}) b_k^{(j)}(\bm{u}^{(j)})\right| \\
&= \sup_{j \in \mathcal{J}} \sup_{u^{(j)} \in U^{(j)}} \sup_{1 \leq \ell \leq N^{d_V}} \left|G^{(j)}[u^{(j)}](v_\ell)-\sum_{k=1}^{H^{(j)}} G^{(j)}[u_k^{(j)}](v_\ell) b_k^{(j)}(\bm{u}^{(j)})\right|\\
&\leq \frac{\eps}{2N^{d_V}}.
\end{align*}
Combining both estimates, we conclude that \[
\sup_{j \in J} \sup_{u^{(j)} \in U^{(j)}} \sup_{x \in \Omega_{V}} \left| G^{(j)}[u^{(j)}](x) - \sum_{\ell=1}^{N^{d_{V}}} \sum_{k=1}^{H^{(j)}} G^{(j)}[u_k^{(j)}](v_\ell) b_k^{(j)}(\bm{u}^{(j)}) \tau_\ell(x) \right| \leq \eps.
\]

\paragraph{Case IV: $U^{(j)} = U$ and $V^{(j)} = V$}
Finally, if $U_j = U$ and $V_j = V$, by combining both of the above, we obtain: 
\begin{equation} \label{eq:rem:multipleOperatorSameDomainSameRange}
\sup_{1\leq j \leq J} \sup_{u\in U} \sup_{x \in \Omega_{V}}\left|G^{(j)}[u](x)- \sum_{\ell=1}^{N^{d_{V}}} \sum_{k=1}^{H^{n_{c_U}}} G^{(j)}[u_k](v_{\ell}) b_k(\ub) \tau_{\ell}(x)\right|\leq \varepsilon.
\end{equation}
Similarly, we can deal with the case when some $U^{(j)}$ and $V^{(j)}$ are distinct, while some are equal. 

\end{remark}

\begin{remark}[Uncountable index set assumptions]
    Remarks \ref{rem:uniformFunctionalApproximation} and \ref{rem:uniformOperatorApproximation} have been formulated for uncountably many functionals and operators respectively. While these results are of interest on their own, they also require several assumptions on the finiteness of various constants. In practice, we will apply Remark \ref{rem:uniformOperatorApproximation} for finitely many operators: this significantly simplifies the necessary assumptions. 
\end{remark}

\begin{remark}[Alternative network for the operator approximation] \label{rem:alternative}
The network appearing in Eq. \eqref{eq:back:operatorApproximation} can be re-written in a slightly different manner. Specifically, we can define \begin{equation} \label{eq:rem:alternative}
    \sum_{\ell=1}^{N^{d_V}} \sum_{k=1}^{H^{n_{c_U}}} G[u_k](v_{\ell}) b_k(\ub) \tau_{\ell}(x) =: \sum_{k=1}^{H^{n_{c_U}}} b_k(\ub) \hat{\tau}_k(x) \text{ or } \sum_{\ell=1}^{N^{d_V}} \sum_{k=1}^{H^{n_{c_U}}} G[u_k](v_{\ell}) b_k(\ub) \tau_{\ell}(x) =: \sum_{\ell=1}^{N^{d_V}} \hat{b}_\ell(\ub) \tau_\ell(x)
\end{equation}
The networks $\hat{\tau}_k$ and $\hat{b}_\ell$ are in the classes $\cS_{N^{d_V}}\cF_1$ and $\cS_{H^{n_{c_U}}}\cF_2$ where $\cF_1$ and $\cF_2$ are defined in Theorem \ref{thm:back:operatorApproximation} and $\cS_{j} \cF$ denotes functions that are linear combinations of $j$ functions in the class $\cF$. We note that this is the convention chosen in \cite{liu2024neuralscalinglawsdeep}. These formulations are particularly well-suited for practical applications, as they replace the double summation in Eq. \eqref{eq:back:operatorApproximation} with a single inner product, thereby significantly simplifying implementation. 

We also want to consider a set of operators
\[
\{G^{(j)}:\{ u:\Omega_{U^{(j)}} \mapsto \bbR \spaceBar \Vert u \Vert_{\Lp{\infty}} \leq \beta_{U^(j)} \} \mapsto V^{(j)} \spaceBar \Vert G^{(j)}(u_1) - G^{(j)}(u_2) \Vert_{\Lp{\infty}} \leq L_{G}^{(j)} \Vert u_1 - u_2 \Vert_{\Lp{r^{(j)}}}\}_{j \in \mathcal{J}}.
\]
as in Remark \ref{rem:uniformOperatorApproximation}. We distinguish the same following four cases for the alternative network formulations of Eq. \eqref{eq:rem:alternative} which approximate all $G^{(j)}$. In particular, using the formulas derived in Remark \ref{rem:uniformOperatorApproximation}, we obtain: \begin{enumerate}
    \item \textbf{Case I: $U^{(j)}$ are distinct and $V^{(j)}$ are distinct} \[
 \sum_{k=1}^{H^{(j)}} b_k^{(j)}(\ub^{(j)}) \hat{\tau}^{(j)}_k(x) \quad \text{or} \quad \sum_{\ell=1}^{(N^{(j)})^{d_V}} \hat{b}_\ell^{(j)}(\ub^{(j)}) \tau^{(j)}_\ell(x);
    \]
    \item \textbf{Case II: $U^{(j)} =U$ and $V^{(j)}$ are distinct}\[
     \sum_{k=1}^{H^{n_{c_U}}} b_k(\ub) \hat{\tau}^{(j)}_k(x)\quad \text{or} \quad \sum_{\ell=1}^{(N^{(j)})^{d_V}} \hat{b}_\ell^{(j)}(\ub^{(j)}) \tau^{(j)}_\ell(x);
    \]
    \item \textbf{Case III: $U^{(j)}$ are distinct and $V^{(j)} = V$} \[
    \sum_{k=1}^{H^{(j)}} b_k^{(j)}(\ub^{(j)}) \hat{\tau}^{(j)}_k(x) \quad \text{or} \quad \sum_{\ell=1}^{N^{d_V}} \hat{b}_\ell^{(j)}(\ub^{(j)}) \tau_\ell(x);
    \]
    \item\textbf{Case IV: $U^{(j)} = U$ and $V^{(j)} = V$} \[
    \sum_{k=1}^{H^{n_{c_U}}} b_k(\ub) \hat{\tau}^{(j)}_k(x) \quad \text{or} \quad \sum_{\ell=1}^{N^{d_V}} \hat{b}_\ell^{(j)}(\ub) \tau_\ell(x).
    \]
\end{enumerate}
The above formulas provide a principled basis for selecting architectures when approximating several single operators simultaneously. In particular, depending on the structure of $U^{(j)}$ and $V^{(j)}$, certain architectures allow the dependence on $j$ to be absorbed into a single network component rather than appearing in multiple components simultaneously. This result also provides a unified framework that encompasses recent approaches such as D2NO~\cite{zhang2024d2no} and MODNO~\cite{zhang2024modno}.
\end{remark}

\begin{remark}[Learning several single operator versus multiple operator learning]
In practice, one often encounters settings where several distinct operators must be learned (as in Remarks \ref{rem:uniformOperatorApproximation} and \ref{rem:alternative}), either independently or with partial weight sharing—for instance, learning solution maps corresponding to different physical regimes or boundary conditions. At first glance, this may seem equivalent to learning a single parameterized operator, i.e. multiple operator learning. 

While both settings involve learning mappings between function spaces, they differ fundamentally in how the dependence on the index variable is treated. In the several single-operator setting, one considers an indexed family of operators $\{G^{(j)} : U^{(j)} \to V^{(j)}\}_{j \in J}$, where $J$ may be finite or uncountable, but the index $j$ does not explicitly enter the learning process. Each operator is represented or trained separately, and any shared structure across $j$ is imposed manually: Remark \ref{rem:alternative} guides the architecture choice in this regard. In contrast, multiple operator learning aims to approximate a single parameterized mapping $G : W \to \{U^{(\alpha)} \to V^{(\alpha)}\}$, where $\alpha \in W$ directly enters the model as an input. This formulation inherently captures how operators vary with $\alpha$, allowing a single network to interpolate across the entire parameter space rather than fitting a collection of independent mappings. We summarize the main comparison points in Table \ref{tab:single-vs-multiple}.
\end{remark}

\begin{table}[H]
\centering
\small
\renewcommand{\arraystretch}{1.5}
\setlength{\tabcolsep}{8pt}
\begin{tabularx}{\linewidth}{>{\bfseries}p{6cm} Y Y}
\toprule
& Several Single Operators & Multiple Operator Learning \\
\midrule
Formulation 
& $\{G^{(j)} : U^{(j)} \to V^{(j)}\}_{j \in J}$ 
& $G : W \to \{ U^{(\alpha)} \to V^{(\alpha)} \}_{\alpha \in W}$ \\
\midrule
Dependence on parameter/index 
& Dependence on $j$ is external to the model 
& Parameter $\alpha$ is an explicit input to the network \\
\midrule
Coupling between operators 
& Optional, via shared structure of the architecture 
& Intrinsic, through a single network \\
\midrule
Generalization capability 
& Limited to operators seen during training 
& Enables interpolation and extrapolation across $\alpha \in W$ \\
\midrule
Practical interpretation 
& Independent or weakly coupled tasks 
& Unified model for a continuum of related tasks \\
\bottomrule
\end{tabularx}
\caption{Comparison between the settings of several single operators and multiple operator learning. 
The latter treats the parameter (or index) as an explicit input, enabling a single network to represent a continuously parameterized family of operators.}
\label{tab:single-vs-multiple}
\end{table}

\begin{remark}[Balancing functional and spatial scaling complexity in the approximating architecture]\label{rem:balancingComplexity}

Theorem~\ref{thm:back:operatorApproximation} establishes specific scaling relations for the space-approximation networks $\tau_\ell$ and the function-approximation networks $b_k$. These scalings arise naturally from the order of approximation adopted in the proof, namely, approximating functions first and functionals second as recalled in Remark \ref{rem:uniformOperatorApproximation}. If the order is reversed, the resulting derivation yields a different scaling behavior, illustrating that the overall approximation complexity can be redistributed between the two components of the network. This observation highlights a fundamental flexibility in the design of operator-learning architectures: the computational burden can be shifted from one subnetwork to another without altering the expressive power of the overall model.

More precisely, when the order of approximation is inverted, the proof follows the steps outlined below. For $x \in \Omega_V$, we start by defining the functional $f_x: \{ u:\Omega_U \mapsto \bbR \spaceBar \Vert u \Vert_{\Lp{\infty}} \leq \beta_U \} \mapsto \bbR$ as \[
    f_x(u) = G[u](x). 
    \]
In particular, we have that \begin{align}
        \vert f_x(u_1) - f_x(u_2) \vert &= \vert G[u_1](x) - G[u_2](x) \vert \notag \\
        &\leq L_G \Vert u_1 - u_2 \Vert_{\Lp{r}(\Omega_U)} \label{eq:operatorApproximation:Lipschitz1} \\
        &\leq L_G \vert \Omega_U \vert^{1/r} \Vert u_1 - u_2 \Vert_{\Lp{\infty}(\Omega_U)} \notag 
    \end{align}
where we use \eqref{eq:operatorApproximation:Lipschitz} for \eqref{eq:operatorApproximation:Lipschitz1}.
Therefore, we can apply Theorem \ref{thm:back:functionalApproximationLinfty}. Specifically, for any  $\varepsilon_0>0$, there exists constants $C'$ and $C_{\delta}$ depending on $\beta_U, L_G \vert \Omega_U \vert^{1/r}$ and $L_G \vert \Omega_U \vert^{1/r},L_U$ respectively such that the following holds. There exists 
\begin{itemize}
    \item a constant $\delta:=C_{\delta}\varepsilon_0$ and points $\{c_m\}_{m=1}^{n_{c_U}}\subset \Omega_U$ so that $\{\mathcal{B}_{\delta}(c_m) \}_{ m  = 1}^{n_{c_U}}$ is a cover of $\Omega_U$ for some $n_{c_U}$,
    \item a network class $\cF_2 = \cF_{\rm NN}(n_{c_U},1,L_2,p_2,K_2,\kappa_2,R_2)$ whose parameters scale as \begin{align*}
    &L_2=\mathcal{O}\left(n_{c_U}^2\log(n_{c_U})+n_{c_U}^2\log(\varepsilon_0^{-1})\right),\quad  p_2 = \mathcal{O}(1),\quad K_2 = \mathcal{O}\left(n_{c_U}^2\log n_{c_U}+n_{c_U}^2\log(\varepsilon_0^{-1})\right), \\ &\kappa_2=\mathcal{O}(n_{c_U}^{n_{c_U}/2+1}\varepsilon_0^{-n_{c_U}-1}),\qquad \, R_2=1
    \end{align*}
    where the constants hidden in $\mathcal{O}$ depend on $\beta_U$ and $L_G \vert \Omega_U \vert^{1/r}$,
    \item networks $\{b_k\}_{k=1}^{H^{n_{c_U}}} \subset \cF_2$ with $H := C' \sqrt{n_{c_U}} \eps_0^{-1}$ and
    \item functions $\{u_k\}_{k=1}^{H^{n_{c_U}}} \subset \{ u:\Omega_U \mapsto \bbR \spaceBar \Vert u \Vert_{\Lp{\infty}} \leq \beta_U \}$
\end{itemize}
such that 
    \begin{align} \label{eq:operatorApproximation:functionalApproximationTrue}
        \sup_{u\in U}\left|f_x(u)-\sum_{k=1}^{H^{n_{c_U}}} f_x(u_k) b_k\left(P_{\mathcal{C}_U}(u))\right)\right| = \sup_{u\in U}\left|G[u](x)-\sum_{k=1}^{H^{n_{c_U}}} G[u_k](x) b_k\left(P_{\mathcal{C}_U}(u))\right)\right| \leq \varepsilon_0.
    \end{align}
where $P_{\mathcal{C}_U}(u)$ is defined in the proof of Theorem \ref{thm:back:functionalApproximationLinfty}.

By assumption, $G[u_k] \in V$ for all $1 \leq k \leq H^{n_{c_U}}$ and we can apply Theorem \ref{thm:back:functionApproximation} to approximate all these functions simultaneously. Specifically, for any  $\varepsilon_1>0$, there exists a constant $C$ depending on $\gamma_V$ and $L_V$ such that the following holds. There exists 
\begin{itemize}
    \item a constant $N = C\sqrt{d_V}\varepsilon_1^{-1}$ and points $\{c_k\}_{k=1}^{N^{d_V}}$ which form a uniform grid on $\Omega_V$ with spacing $2\gamma_V/N$ along each dimension,
    \item a network class $\cF_1 = \cF_{\rm NN}(d_V,1,L_1,p_1,K_1,\kappa_1,R_1)$ whose parameters scale as \begin{align*}
    &L_1 = \mathcal{O}\left(d_V^2\log d_V+d_V^2\log(\varepsilon_1^{-1})\right),\quad p_1 = \mathcal{O}(1),\quad K_1 = \mathcal{O}\left(d_V^2\log d_V+d_V^2\log(\varepsilon_1^{-1})\right),\\
&\kappa_1=\mathcal{O}(d_V^{d_V/2+1}\varepsilon_1^{-d_V-1}),\qquad \qquad \quad R_1=1
    \end{align*}
where the constants hidden in $\mathcal{O}$ depend on $\gamma_V,L_V$ and 
    \item networks $\{\tau_\ell\}_{\ell=1}^{N^{d_V}} \subset \cF_2$ 
\end{itemize}
such that, for every $1 \leq k \leq N^{d_V}$:
\begin{equation} \label{eq:operatorApproximation:functionalApproximation}
		\sup_{x \in \Omega_V} \left\vert G[u_k](x) - \sum_{\ell=1}^{N^{d_V}} G[u_k](v_{\ell}) \tau_{\ell}(x) \right\vert \leq \eps_1.
\end{equation}

Combining both of our bounds Eqs. \eqref{eq:operatorApproximation:functionalApproximationTrue} and \eqref{eq:operatorApproximation:functionalApproximation}, we obtain: \begin{align}
	&\sup_{x\in \Omega_V, \, u \in U}\left\vert G[u](x) - \sum_{k=1}^{H^{n_{c_U}}} \sum_{\ell=1}^{N^{d_V}} G[u_k](v_{\ell}) b_k(P_{\mathcal{C}_U}(u)) \tau_{\ell}(x) \right\vert \notag \\
	 &\leq \sup_{x\in \Omega_V, \, u \in U} \left| G[u](x) - \sum_{k=1}^{H^{n_{c_U}}} G[u_k](x)b_k(P_{\mathcal{C}_U}(u)) \right| \notag \\
	&+  \sup_{x\in \Omega_V, \, u \in U} \sum_{k=1}^{H^{n_{c_U}}} \left| G[u_k](x) - \sum_{\ell=1}^{N^{d_V}} G[u_k](v_{\ell})  \tau_{\ell}(x) \right| \vert b_k(P_{\mathcal{C}_U}(u)) \vert \notag \\
	&\leq \eps_0 + \sup_{x\in \Omega_V, \, u \in U} \sum_{k=1}^{H^{n_{c_U}}} \left| G[u_k](x) - \sum_{\ell=1}^{N^{d_V}} G[u_k](v_{\ell})  \tau_{\ell}(x) \right| \label{eq:operatorApproximation:estimate1} \\
	&\leq \eps_0 + H^{n_{c_U}} \eps_1. \label{eq:operatorApproximation:estimate2}
\end{align}
where we use \eqref{eq:operatorApproximation:functionalApproximationTrue} and the fact that $\vert b_k(P_{\mathcal{C}_U}(u)) \vert \leq 1$ in \eqref{eq:operatorApproximation:estimate1} and \eqref{eq:operatorApproximation:functionalApproximation} for \eqref{eq:operatorApproximation:estimate2}.

We conclude by picking $\eps_0 = \eps/2$ and $\eps_1 = \eps/(2 H^{n_{c_U}})$. 
Since $\eps_1 = \eps^{n_{c_U}+1} (C'\sqrt{n_{c_U}})^{-n_{c_U}}2^{-(1+n_{c_U})}$, the final scalings for $\cF_1$ are
\begin{align*}
    &L_1 = \mathcal{O}\left(d_V^2\log d_V+d_V^2\log(\varepsilon^{-(n_{c_U}+1)}) +d_{V}^2 \log( 2^{n_{c_U} +1} (C' \sqrt{n_{c_U}})^{n_{c_U}}) \right),\quad p_1 = \mathcal{O}(1),\\
    &K_1 = \mathcal{O}\left(d_V^2\log d_V+d_V^2\log(\varepsilon^{-(n_{c_U}+1)}) + d_{V}^2 \log( 2^{n_{c_U} +1} (C' \sqrt{n_{c_U}})^{n_{c_U}})\right),\\
&\kappa_1=\mathcal{O}(d_V^{d_V/2+1}\varepsilon^{-(d_V+1)(n_{c_U}+1)} \ls 2^{n_{c_U} +1} (C' \sqrt{n_{c_U}})^{n_{c_U}} \rs^{-1}),\qquad R_1=1,\\
&N = 2^{n_{c_U}+1}C\sqrt{d_V}(C'\sqrt{n_{c_U}})^{n_{c_U}} \eps^{-(1+n_{c_U})}).
    \end{align*}

We summarize the differences in scaling for the various networks in Table \ref{tab:balancingComplexity}. Reversing the approximation order shifts the computational cost: in the original formulation, most complexity lies in the function-approximation networks, whereas in the reversed case, it is transferred to the space-approximation networks.
\end{remark}

\begin{remark}
    [Total number of parameters for operator learning] \label{rem:totalParametersOperatorLearning}
    We now express the approximation error of the network in Eq. \eqref{eq:back:operatorApproximation} as a function of the total number of parameters, $N_\# := N^{d_V}K_1 + H^{n_{c_U}}K_2$. We note that $n_{c_U} = \mathcal{O}(\eps^{-{(1+d_V)d_U}})$, by \cite[Lemma 2]{liu2024neuralscalinglawsdeep}. We compute as follows: \begin{align*}
        N_\# &= N^{d_V}K_1 + H^{n_{c_U}}K_2 \\
        &\asymp \eps^{-d_V} \log(\eps^{-1}) + \ls \sqrt{n_{c_U}} \eps^{-(1+d_V)}\rs^{n_{c_U}} \l n_{c_U}^2\log n_{c_U}+n_{c_U}^2(d_V +1)\log(\varepsilon^{-1}) \r  \\
        &\asymp \eps^{-d_V} \log(\eps^{-1}) + \eps^{-[\frac{(1+d_V)d_U}{2} + (1+d_V)]\eps^{-{(1+d_V)d_U}}} \eps^{-{2d_U(1+d_V)}}\log(\eps^{-1}) \ls (1+d_V)d_U + d_V \rs \\
        &\asymp \eps^{-[\frac{(1+d_V)d_U}{2} + (1+d_V)]\eps^{-{(1+d_V)d_U}} - 2d_U(1+d_V)} \log(\eps^{-1})\ls (1+d_V)d_U + d_V \rs.
    \end{align*}
Taking logarithms on each side leads to:\begin{align*}
    \log(N_\#) &\asymp \l \ls \frac{(1+d_V)d_U}{2} + (1+d_V) \rs \eps^{-{(1+d_V)d_U}} + 2d_U(1+d_V)\r \log(\eps^{-1}) + \log(\log(\eps^{-1})) \\
    &\asymp \l \ls \frac{(1+d_V)d_U}{2} + (1+d_V) \rs \eps^{-{(1+d_V)d_U}} \r \log(\eps^{-1}) \\
    &=: \theta \eps^{-\gamma} \log(\eps^{-1}). 
\end{align*}
In fact, this is equivalent to \[
\frac{\gamma}{\theta} \log(N_\#) \asymp \log(\eps^{-\gamma}) \eps^{-\gamma}
\]
and, with the change of variable $t = \log(\eps^{-1})$, we obtain \[
\frac{\gamma}{\theta} \log(N_\#) \asymp \gamma t e^{\gamma t}.
\]
Using the Lambert $W$ function \cite{Mezo2022LambertW} (defined by $W(z)\,e^{W(z)}=z$), we obtain that 
\[
t \asymp \frac{1}{\gamma} W\l\frac{\gamma}{\theta} \log\l N_\#\r\r
\]
which leads to \[
\eps \asymp \exp\l -\frac{1}{\gamma} W\l\frac{\gamma}{\theta} \log\l N_\#\r\r \r.
\]
For large arguments, $W(z) = \log(z) - \log(\log z) + o(1)$ (see \cite[Section 4.1.4]{Mezo2022LambertW}), thus
\[
W\l\frac{\gamma}{\theta} \log\l N_\#\r\r
= \log(\log (N_\#)) - \log(\log(\log (N_\#))) + o(1),
\]
and hence
\[
\varepsilon
\asymp
\Bigg(\frac{\log N_\#}{\log\log N_\#}\Bigg)^{-1/\gamma}
=
\Bigg(\frac{\log N_\#}{\log\log N_\#}\Bigg)^{-\frac{1}{(1+d_V)d_U}}.
\]  
\end{remark}

\begin{remark}[Dependence of parameter complexity on approximation order] \label{rem:complexityOrder}
    Similarly to Remark \ref{rem:totalParametersOperatorLearning}, we now express the approximation error of the network in Eq. \eqref{eq:back:operatorApproximation} as a function of the total number of parameters, but using the alternative approximation order of Remark \ref{rem:balancingComplexity}. We note that $n_{c_U} = \mathcal{O}(\eps^{-d_U})$ by Lemma \cite[Lemma 2]{liu2024neuralscalinglawsdeep}.

    In particular, we have \begin{align*}
        \log(N^{d_V}) &= d_V \ls (n_{c_U}+1)\log(2) + \log(C \sqrt{d_V}) + n_{c_U} \log(C') + \frac{n_{c_U}}{2}\log(n_{c_U}) + (1 + n_{c_U})\log(\eps^{-1}) \rs \\
        &\asymp d_V \ls \eps^{-d_U} \frac{d_U}{2} \log(\eps^{-1}) + \eps^{-d_U} \log(\eps^{-1}) \rs \\
        &= \eps^{-d_U}\log(\eps^{-1}) \l \frac{d_U}{2} + 1 \r d_V
    \end{align*}
which implies 
\[
N^{d_V} \asymp \eps^{-\eps^{-d_U}d_V \l \frac{d_U}{2} + 1 \r}.
\]
For $K_1$, we have: \begin{align*}
K_1 &\asymp d_V^2 (n_{c_U}+1) \log(\eps^{-1}) + d_V^2 (n_{c_U}+1) \log(2) + d_V^2 n_{c_U} \log(C') + d_V^2 \frac{n_{c_U}}{2} \log(n_{c_U}) \\
&\asymp d_V^2 \eps^{-d_U} \log(\eps^{-1}) + d_V^2 \eps^{-d_U} \frac{d_U}{2} \log(\eps^{-1}) \\
&= \eps^{-d_U} \log(\eps^{-1}) d_V^2 \l 1 + \frac{d_U}{2} \r
\end{align*}
from which we deduce that \[
N^{d_V} K_1 \asymp \eps^{-\eps^{-d_U}d_V \l \frac{d_U}{2} + 1 \r - d_U} \log(\eps^{-1}) d_V^2 \l 1 + \frac{d_U}{2} \r.
\]
Similarly, we have \begin{align*}
    \log(H^{n_{c_U}}) &\asymp n_{c_U} \ls \log(\sqrt{n_{c_U}}) + \log(\eps^{-1}) \rs \\
    &= \eps^{-d_U} \ls \frac{d_U}{2} \log(\eps^{-1}) + \log(\eps^{-1}) \rs \\
    &= \eps^{-d_U} \log(\eps^{-1}) \l 1 + \frac{d_U}{2} \r
\end{align*}
hence, \[
H^{n_{c_U}} \asymp \eps^{-\eps^{-d_U}\l 1 + \frac{d_U}{2} \r}.
\]
This implies \begin{align*}
    H^{n_{c_U}} K_2 &\asymp \eps^{-\eps^{-d_U}\l 1 + \frac{d_U}{2} \r} \ls \eps^{-2d_u} d_U \log(\eps^{-1}) + \eps^{-2d_U} \log(\eps^{-1}) \rs \\
    &= \eps^{-\eps^{-d_U}\l 1 + \frac{d_U}{2} \r - 2d_U} \log(\eps^{-1}) \l 1 + d_U \r  
\end{align*}
and consequently: 
\[
N_\# \asymp \begin{cases} \eps^{-\eps^{-d_U}d_V \l \frac{d_U}{2} + 1 \r - d_U} \log(\eps^{-1}) d_V^2 \l 1 + \frac{d_U}{2} \r & \text{if $d_V > 1$} \\
\eps^{-\eps^{-d_U}\l 1 + \frac{d_U}{2} \r - 2d_U} \log(\eps^{-1}) \l 1 + d_U \r & \text{if $d_V = 1$.}
\end{cases}
\]

As our analysis in Remark \ref{rem:totalParametersOperatorLearning} shows, the only parameter determining the leading order of the final rates is the power of the $\eps$-power term, i.e. $d_U$. We conclude that \[
\eps \asymp \Bigg(\frac{\log N_\#}{\log\log N_\#}\Bigg)^{-\frac{1}{d_U}}.
\]
We note that reversing the approximation order yields a more favorable parameter scaling, as reflected by the improved rate, since $1/d_U > 1/(d_U(1+d_V))$. This observation underscores the importance of architectural design choices in determining the overall efficiency of operator approximation. We summarize these observations in Table \ref{tab:balancingComplexity}.
\end{remark}

\begin{table}[]
\centering
\small
\renewcommand{\arraystretch}{1.5}
\setlength{\tabcolsep}{8pt}
\begin{tabularx}{\linewidth}{>{\bfseries}p{2cm} Y Y}
\toprule
& Theorem \ref{thm:back:operatorApproximation} & Remark \ref{rem:balancingComplexity} \\
\midrule
\textbf{Approximation goal} 
& \multicolumn{2}{>{\raggedright\arraybackslash}p{0.65\linewidth}}{Establish scaling laws for a Lipschitz operator $G:U\mapsto V$} \\
\midrule
\textbf{Approximating architecture} 
& \multicolumn{2}{>{\raggedright\arraybackslash}p{0.65\linewidth}}{$$\sum_{\ell=1}^{N^{d_{V}}} \sum_{k=1}^{H^{n_{c_U}}} G^{(j)}[u_k](v_{\ell}) b_k(\ub) \tau_{\ell}(x)$$} \\
\midrule
Approximation order 
& Function, then Functional 
& Functional, then Function \\
\midrule
Value of $n_{c_{U}}$
& $\mathcal{O}(\eps^{-(1+d_V)d_U})$  
& $\mathcal{O}(\eps^{-d_U})$ \\
\midrule
Value of $N$
& $2C \sqrt{d_V} \eps^{-1}$  
& $2^{n_{c_U}+1}C\sqrt{d_V}(C'\sqrt{n_{c_U}})^{n_{c_U}} \eps^{-(1+n_{c_U})}$ \\
\midrule
Network class for $\tau_\ell$ 
& $\cF_{\rm NN}(d_V,1,L,p,K,\kappa,R)$  with parameters scaling as
\[
\begin{aligned}
L &= \mathcal{O}(d_V^2\log d_V+d_V^2\log(\varepsilon^{-1})),\\
p &= \mathcal{O}(1),\\
K &= \mathcal{O}(d_V^2\log d_V+d_V^2\log(\varepsilon^{-1})),\\
\kappa &= \mathcal{O}(d_V^{d_V/2+1}\varepsilon^{-d_V-1}),\\
R &= 1
\end{aligned}
\]
& $\cF_{\rm NN}(d_V,1,L,p,K,\kappa,R)$ with parameters scaling as \[\begin{aligned}
L &= \mathcal{O}(d_V^2\log d_V+d_V^2\log(\varepsilon^{-(n_{c_U}+1)}) \\
&\qquad+d_{V}^2 \log( 2^{n_{c_U} +1} (C' \sqrt{n_{c_U}})^{n_{c_U}}) ),\\
p &= \mathcal{O}(1),\\
K &= \mathcal{O}(d_V^2\log d_V+d_V^2\log(\varepsilon^{-(n_{c_U}+1)}) \\
&\qquad+d_{V}^2 \log( 2^{n_{c_U} +1} (C' \sqrt{n_{c_U}})^{n_{c_U}}) ),\\
\kappa &= \mathcal{O}(d_V^{d_V/2+1}\varepsilon^{-(d_V+1)(n_{c_U}+1)} \\
&\qquad \times [2^{n_{c_U} +1} (C' \sqrt{n_{c_U}})^{n_{c_U}} ]^{-1}),\\
R &= 1
\end{aligned}\]
    \\
\midrule
Value of $H$
& $2^{d_V+1}C' \sqrt{n_{c_U}} (C \sqrt{d_V})^{d_V} \varepsilon^{-(d_V+1)}$  
& $2C'\sqrt{n_{c_U}}\eps^{-1}$ \\
\midrule
Network class for $b_k$
& $\cF_{\rm NN}(n_{c_U},1,L,p,K,\kappa,R)$ with parameters scaling as
\[\begin{aligned}
L &= \mathcal{O}(n_{c_U}^2\log n_{c_U}+n_{c_U}^2(d_V +1)\log(\varepsilon^{-1}) \\
&\qquad+ n_{c_U}^2 \log(2^{d_V +1} (C \sqrt{d_V})^{d_V} )),\\
p &= \mathcal{O}(1),\\
K &= \mathcal{O}(n_{c_U}^2\log n_{c_U}+n_{c_U}^2(d_V +1)\log(\varepsilon^{-1})\\
&\qquad+ n_{c_U}^2 \log(2^{d_V +1} (C \sqrt{d_V})^{d_V} )),\\
\kappa &= \mathcal{O}(n_{c_U}^{n_{c_U}/2+1}\varepsilon^{-(d_V+1)(n_{c_U}+1)}\\
&\qquad \times[ 2^{d_V +1} (C \sqrt{d_V})^{d_V} ]^{n_{c_U}+1}),\\
R &= 1
\end{aligned}\]
& $\cF_{\rm NN}(n_{c_U},1,L,p,K,\kappa,R)$ with parameters scaling as \[
\begin{aligned}
L &= \mathcal{O}(n_{c_U}^2\log(n_{c_U})+n_{c_U}^2\log(\varepsilon^{-1})),\\
p &= \mathcal{O}(1),\\
K &= \mathcal{O}(n_{c_U}^2\log(n_{c_U})+n_{c_U}^2\log(\varepsilon^{-1})),\\
\kappa &= \mathcal{O}(n_{c_U}^{n_{c_U}/2+1}\varepsilon^{-(n_{c_U}+1)}),\\
R &= 1
\end{aligned}
\] \\
\midrule
Total parameters $N_\#$ scaling
& $$\Bigg(\frac{\log N_\#}{\log\log N_\#}\Bigg)^{-\frac{1}{(1+d_V)d_U}}$$  
& $$\Bigg(\frac{\log N_\#}{\log\log N_\#}\Bigg)^{-\frac{1}{d_U}}$$ \\
\bottomrule
\end{tabularx}
\caption{Comparison of scaling behaviors for the space-approximation networks $\tau_\ell$, function-approximation networks $b_k$ and total number of parameters under different approximation orders for operator learning. We have $n_{c_U} = \mathcal{O}(\delta^{-d_U})$ by \cite[Lemma 2]{liu2024neuralscalinglawsdeep}. The results illustrate that (1) scaling complexity can be redistributed between the subnetworks without affecting the expressive power of the overall architecture, and (2) the chosen approximation order directly impacts the scaling of the total number of parameters.}
\label{tab:balancingComplexity}
\end{table}

Combining previous results, we conclude with the scaling laws for the multiple operator approximation problem. In particular, the proof reduces multiple operator learning to learning a finite amount of single operators.

\begin{theorem}[Multiple Operator Scaling Laws]\label{thm:main:multipleOperatorApproximation}
Let $d_W,d_U,d_V>0$ be integers, \[
\gamma_W, \gamma_U, \gamma_V, \beta_W, \beta_U,\beta_V,L_W,L_U,L_V,L_G,L_{\mathcal{G}} > 0 \qquad \text{and} \qquad r_G, r_{\mathcal{G}} \geq 1\]
and assume that $W(d_W,\gamma_W,L_W,\beta_W)$,
 $U(d_U,\gamma_U,L_U,\beta_U)$ and $V(d_V,\gamma_V,L_V,\beta_V)$ satisfy Assumption \ref{assumption:Main:assumptions:S4}. %Without loss of generality, let $\gamma_W \leq \beta_W$, $\gamma_U \leq \beta_U$ and $\gamma_V \leq \beta_V$. 
 Let $G$ be a map such that \begin{align*}
 	&G:\{ \alpha:\Omega_W \mapsto \bbR \spaceBar \Vert \alpha \Vert_{\Lp{\infty}} \leq \beta_W \} \mapsto \mathcal{G} \qquad \text{where }\\
 	&\mathcal{G} = \Big\{ G[\alpha]\spaceBar G[\alpha]:\{ u:\Omega_U \mapsto \bbR \spaceBar \Vert u \Vert_{\Lp{\infty}} \leq \beta_U \} \mapsto V \text{ and } \\
 	&\quad \Vert G[\alpha][u_1] - G[\alpha][u_2]\Vert_{\Lp{\infty}(\Omega_V)} \leq L_{\mathcal{G}}\Vert u_1 - u_2 \Vert_{\Lp{r_\mathcal{G}}(\Omega_U)} \Big\} 
 \end{align*}
Furthermore, assume that $G$ satisfies \begin{equation} \label{eq:multipleOperatorApproximation:Lipschitz}
   \Vert G(\alpha_1) - G(\alpha_2) \Vert_{\Lp{\infty}(\{ u:\Omega_U \mapsto \bbR \spaceBar \Vert u \Vert_{\Lp{\infty}} \leq \beta_U \} \times \Omega_V)} \leq L_G \Vert \alpha_1 - \alpha_2 \Vert_{\Lp{r_G}(\Omega_U)}  
\end{equation}
for $\alpha_1,\alpha_2 \in \{ \alpha:\Omega_W \mapsto \bbR \spaceBar \Vert \alpha \Vert_{\Lp{\infty}} \leq \beta_W \}$.

There exists constants $C$ depending on $\gamma_V,L_V$, $C_{\delta}$ depending on $L_{\mathcal{G}},d_U,\gamma_U,r_{\mathcal{G}},L_U$, $C'$ depending on $\beta_U, L_{\mathcal{G}}, d_U, \gamma_U, r_{\mathcal{G}}$, $C_\zeta$ depending on $L_G, d_W, \gamma_W, r_G,L_W$ and $C''$ depending on $\beta_W, L_G, d_W, \gamma_W, r_G$ such that the following holds. For any $\varepsilon>0$, \begin{itemize}
     
     \item let $N=  2^{n_{c_W} + 2} C \sqrt{d_V} (C'' \sqrt{n_{c_W}})^{n_{c_W}} \varepsilon^{-(n_{c_W}+1)}$ and consider the network class \newline $\cF_1=\cF_{\rm NN}(d_V,1,L_1,p_1,K_1,\kappa_1,R_1)$  with parameters scaling as
    \begin{align*}
    &L_1 = \mathcal{O}\left(d_V^2\log d_V+d_V^2(n_{c_W}+1)\log(\varepsilon^{-1}) +d_V^2\log(2^{n_{c_W}+1}(C''\sqrt{n_{c_W}})^{n_{c_W}})\right) ,\quad  p_1 = \mathcal{O}(1),\\
    &K_1 = \mathcal{O}\left(d_V^2\log d_V+d_V^2(n_{c_W}+1)\log(\varepsilon^{-1}) +d_V^2\log(2^{n_{c_W}+1}(C''\sqrt{n_{c_W}})^{n_{c_W}})\right)\\
&\kappa_1=\mathcal{O}(d_V^{d_V/2+1}\varepsilon^{-(d_V+1)(n_{c_W}+1)} \ls 2^{n_{c_W}+1}(C''\sqrt{n_{c_W}})^{n_{c_W}}\rs^{(d_V+1)}),\qquad R_1=1
    \end{align*}
    where the constants hidden in $\mathcal{O}$ depend on $\gamma_V$ and $L_V$;

    \item let $\{v_\ell\}_{\ell=1}^{N^{d_V}} \subset \Omega_V$ be a uniform grid with spacing $2\gamma_V/N$ along each dimension;
    
    \item let $\delta=\frac{C_{\delta}\varepsilon^{(1+d_V)(1+n_{c_W})}}{2^{d_V+n_{c_W}+2}(C\sqrt{d_V})^{d_V}(C''\sqrt{n_{c_W}})^{n_{c_W}}}$ and let $\{c_m\}_{m=1}^{n_{c_U}}\subset \Omega_U$ be points so that $\{\mathcal{B}_{\delta}(c_m) \}_{ m  = 1}^{n_{c_U}}$ is a cover of $\Omega_U$ for some $n_{c_U}$;
    
    \item let $H = 2^{(d_V+1)(n_{c_W}+2)}C' \sqrt{n_{c_U}} (C \sqrt{d_V})^{d_V} (C'' \sqrt{n_{c_W}})^{n_{c_W}(d_V +1)} \eps^{-(d_V+1)(1+n_{c_W})}$ and consider the network class $\cF_2=\cF_{\rm NN}(n_{c_U},1,L_2,p_2,K_2,\kappa_2,R_2)$ with parameters scaling as
\begin{align*}
        &L_2 = \mathcal{O}\big(n_{c_U}^2\log n_{c_U}+n_{c_U}^2(d_V +1)(n_{c_W}+1)\log(\varepsilon^{-1}) + n_{c_U}^2 \log(2^{d_V +1} (C \sqrt{d_V})^{d_V} )\\
        &\qquad + n_{c_U}^2(d_V +1)\log(2^{n_{c_W}+1}(C''\sqrt{n_{c_W}})^{n_{c_W}})\big),\quad p_2=\mathcal{O}(1), \\
&K_2 = \mathcal{O}\big(n_{c_U}^2\log n_{c_U}+n_{c_U}^2(d_V +1)(n_{c_W}+1)\log(\varepsilon^{-1}) + n_{c_U}^2 \log(2^{d_V +1} (C \sqrt{d_V})^{d_V} )\\
        &\qquad + n_{c_U}^2(d_V +1)\log(2^{n_{c_W}+1}(C''\sqrt{n_{c_W}})^{n_{c_W}})\big), \\ 
&\kappa_2=\mathcal{O}(n_{c_U}^{n_{c_U}/2+1}\varepsilon^{-(d_V+1)(n_{c_U}+1)(n_{c_{W}}+1)}[ 2^{d_V +1} (C \sqrt{d_V})^{d_V} ]^{n_{c_U}+1} \ls 2^{d_V +1} (C \sqrt{d_V})^{d_V} \rs^{(d_V+1)(n_{c_U}+1)}), \\
&R_2=1
    \end{align*}
    where the constants hidden in $\mathcal{O}$ depend on $\beta_U, L_{\mathcal{G}}, d_U, \gamma_U,r_{\mathcal{G}}$;

\item let $\zeta=C_{\zeta}\varepsilon$ and let $\{y_m\}_{m=1}^{n_{c_W}}\subset \Omega_W$ be points so that $\{\mathcal{B}_{\zeta}(y_m) \}_{ m = 1}^{n_{c_W}}$ is a cover of $\Omega_W$ for some $n_{c_W}$;

\item let $P = 2C'' \sqrt{n_{c_W}} \eps^{-1}$ and consider the network class $\cF_3=\cF_{\rm NN}(n_{c_W},1,L_3,p_3,K_3,\kappa_3,R_3)$ with parameters scaling as
\begin{align*}
       &L_3=\mathcal{O}\left(n_{c_W}^2\log(n_{c_W})+n_{c_W}^2\log(\varepsilon^{-1})\right),\quad  p_3 = \mathcal{O}(1),\quad K_3 = \mathcal{O}\left(n_{c_W}^2\log n_{c_W}+n_{c_W}^2\log(\varepsilon^{-1})\right), \\ &\kappa_3=\mathcal{O}(n_{c_W}^{n_{c_W}/2+1}\varepsilon^{-n_{c_W}-1}),\qquad \, R_3=1
    \end{align*}
    where the constants hidden in $\mathcal{O}$ depend on $\beta_W, L_G, d_W, \gamma_W,r_G$.
    
 \end{itemize} 
 Then, there exists networks $\{\tau_\ell\}_{\ell=1}^{N^{d_V}} \subset \cF_1$, networks $\{b_k\}_{k=1}^{H^{n_{c_U}}} \subset \mathcal{F}_2$, networks $\{l_p\}_{p=1}^{P} \subset \mathcal{F}_3$, functions $\{u_k\}_{k=1}^{H^{n_{c_U}}} \subset \{ u:\Omega_U \mapsto \bbR \spaceBar \Vert u \Vert_{\Lp{\infty}} \leq \beta_U \}$ and functions $\{\alpha_p\}_{p=1}^P \subset \{ \alpha:\Omega_W \mapsto \bbR \spaceBar \Vert \alpha \Vert_{\Lp{\infty}} \leq \beta_W \}$ such that 
 \begin{align}
        \sup_{\alpha\in W}\sup_{u\in U}\sup_{x \in \Omega_V}\left|G[\alpha][u](x)-\sum_{p=1}^{P^{n_{c_W}}}\sum_{k=1}^{H^{n_{c_U}}} \sum_{\ell=1}^{N^{d_V}} G[\alpha_p][u_k](v_{\ell}) l_p(\bm{\alpha}) b_k(\ub) \tau_{\ell}(x)\right|\leq \varepsilon, \label{eq:main:multipleOperatorApproximation}
    \end{align}
where $\bm{\alpha}=(\alpha(y_1), \alpha(y_2),...,\alpha(y_{n_{c_W}}))^\top$ is a discretization of $\alpha$ and $\ub=(u(c_1), u(c_2),...,u(c_{n_{c_U}}))^\top$ is a discretization of $u$.

\end{theorem}

The proof of the theorem is presented in Section~\ref{sec_main_multipleOperatorApproximation}.
As in the operator learning setting of Theorem~\ref{thm:back:operatorApproximation}, the argument follows an inherently sequential structure, proceeding through successive approximation stages: first for the mapping $\alpha \mapsto G[\alpha]$, then for $u \mapsto G[\alpha][u]$, and finally for $x \mapsto G[\alpha][u](x)$. As a result, the scaling behavior deteriorates progressively, since each stage of the approximation inherits and compounds the error and complexity of the preceding one.
This increasing scaling complexity is reflected in the growth of the network classes $\mathcal{F}_3$, $\mathcal{F}_1$, and $\mathcal{F}_2$.

\begin{remark}[Total number of parameter for multiple operator learning]
Similarly to Remark \ref{rem:totalParametersOperatorLearning}, we now express the approximation error of the network in Eq. \eqref{eq:main:multipleOperatorApproximation} as a function of the total number of parameters $N_\# = P^{n_{c_W}} K_3 + N^{d_V} K_1 + H^{n_{c_U}} K_2$. We note that $n_{c_W} = \mathcal{O}(\eps^{-d_W})$ and $n_{c_U} = \mathcal{O}(\delta^{-d_U})$ by \cite[Lemma 2]{liu2024neuralscalinglawsdeep}. For the latter, we compute \begin{align*}
    \log(\delta^{-d_U}) &= -d_U \log\l \frac{C_{\delta}\varepsilon^{(1+d_V)(1+n_{c_W})}}{2^{d_V+n_{c_W}+2}(C\sqrt{d_V})^{d_V}(C''\sqrt{n_{c_W}})^{n_{c_W}}} \r \\
    &= -d_u \log(C_\delta) + d_U (1+d_V)(1+n_{c_W}) \log(\eps^{-1}) + d_U(d_V + n_{c_W} + 2)\log(2)\\
    &+ d_U n_{c_W}\log(C'') + d_U \frac{n_{c_W}}{2}\log(n_{c_W}) \\
    &\asymp d_U (1+d_V)(1+\eps^{-d_W})\log(\eps^{-1}) + d_U \eps^{-d_W} + \frac{d_Ud_W}{2}\eps^{-d_W} \log(\eps^{-1}) \\
    &\asymp \log(\eps^{-1})\eps^{-d_W}d_U\l(1 + d_V) + \frac{d_W}{2}\r
\end{align*}
which leads to \[
n_{c_U} \asymp \eps^{-\eps^{-d_W}d_U\l(1 + d_V) + \frac{d_W}{2}\r}.
\]

We now consider: \begin{align}
    P^{n_{c_W}} K_3 & \asymp n_{c_W}^{\sfrac{n_{c_W}}{2}}\eps^{-n_{c_W}} \left(n_{c_W}^2\log n_{c_W}+n_{c_W}^2\log(\varepsilon^{-1})\right) \notag \\
    &= \eps^{-\eps^{-d_W}\l1+\frac{d_W}{2}\r} \l d_W \eps^{-2d_W}\log(\eps^{-1}) + \eps^{-2d_W}\log(\eps^{-1}) \r \notag \\
    &\asymp \eps^{-\eps^{-d_W}\l1+\frac{d_W}{2}\r -2d_W}\log(\eps^{-1})\l 1 + d_W \r. \label{eq:rem:totalMultiple:PK3}
\end{align}
Next, we note that \begin{align*}
    \log(N^{d_V}) &= d_V \log(2^{n_{c_W} + 2} C \sqrt{d_V} (C'' \sqrt{n_{c_W}})^{n_{c_W}} \varepsilon^{-(n_{c_W}+1)}) \\
    &= d_V \ls (n_{c_W +1})\log(2) + \log(C\sqrt{d_V}) + n_{c_W}\log(C'') + \frac{n_{c_W}}{2}\log(n_{c_W}) + (n_{c_W}+1)\log(\eps^{-1})\rs \\
    &\asymp d_V \ls \eps^{-d_V} (\log(2) + C'') + \eps^{-d_W} \log(\eps^{-1})\l\frac{d_W}{2} +1\r + \log(\eps^{-1}) \rs \\
    &= d_V \eps^{-d_W} \log(\eps^{-1})\l\frac{d_W}{2} +1\r
\end{align*}
which implies $$N^{d_V} \asymp \eps^{-\eps^{-d_W}d_V\l \frac{d_W}{2} +1 \r}.$$ 
Using this, we have \begin{align}
    N^{d_V} K_1 &\asymp \eps^{-\eps^{-d_W}d_V\l \frac{d_W}{2} +1\r} \ls d_V^2\log d_V+d_V^2(n_{c_W}+1)\log(\varepsilon^{-1}) +d_V^2\log(2^{n_{c_W}+1}(C''\sqrt{n_{c_W}})^{n_{c_W}}) \rs \notag \\
    &\asymp \eps^{-\eps^{-d_W}d_V\l \frac{d_W}{2} +1\r} \bigg[ d_V^2\eps^{-d_W}\log(\varepsilon^{-1}) + d_V^2 \log(\eps^{-1}) +d_V^2(n_{c_W}+1)\log(2) \notag \\
    &+ d_V^2 n_{c_W} \log(C'') +d_V^2 \frac{n_{c_W}}{2} \log( n_{c_W})) \bigg] \notag \\
    &\asymp \eps^{-\eps^{-d_W}d_V\l \frac{d_W}{2} +1\r} \ls d_V^2\eps^{-d_W}\log(\varepsilon^{-1}) + d_V^2 \frac{d_W}{2} \eps^{-d_W} \log(\eps^{-1}) \rs \notag \\
    &= \eps^{-\eps^{-d_W}d_V\l \frac{d_W}{2} +1\r - d_W} \log(\eps^{-1}) d_V^2\l 1 + \frac{d_W}{2} \r. \label{eq:rem:totalMultiple:NK1}
\end{align}
Then, we consider \begin{align*}
    &\log(H^{n_{c_U}}) = n_{c_U} \log(2^{(d_V+1)(n_{c_W}+2)}C' \sqrt{n_{c_U}} (C \sqrt{d_V})^{d_V} (C'' \sqrt{n_{c_W}})^{n_{c_W}(d_V +1)} \eps^{-(d_V+1)(1+n_{c_W})}) \\
    &= n_{c_U} \bigg[(d_V +1)(n_{c_W}+2)\log(2) + \log(C'(C\sqrt{d_V})^{d_V}) + \frac{1}{2}\log(n_{c_u}) + n_{c_W}(d_V +1) \log(C'') \\
    &+ n_{c_W}\frac{(d_V+1)}{2} \log(n_{c_W}) + (d_V+1)(1+n_{c_W})\log(\eps^{-1}) 
    \bigg] \\
    &\asymp n_{c_U} \bigg[(d_V+1)\eps^{-d_W}\log(2) + d_U\l(1 + d_V) + \frac{d_W}{2}\r \eps^{-d_W}\log(\eps^{-1}) + \eps^{-d_W}(d_V+1)\log(C'') \\
    &+d_W\frac{(d_V+1)}{2}\eps^{-d_W}\log(\eps^{-1}) + (d_V+1)\eps^{-d_W}\log(\eps^{-1})\bigg] \\
    &\asymp n_{c_U} \ls \eps^{-d_W}\log(\eps^{-1})\l d_U\l(1 + d_V) + \frac{d_W}{2}\r + d_W\frac{(d_V+1)}{2} + (d_V+1)  \r \rs \\
    &\asymp \eps^{-\eps^{-d_W}d_U\l(1 + d_V) + \frac{d_W}{2}\r -d_W} \log(\eps^{-1}) \ls  d_U\l(1 + d_V) + \frac{d_W}{2}\r + d_W\frac{(d_V+1)}{2} + (d_V+1)   \rs 
\end{align*}
which implies that \[
H^{n_{c_U}} \asymp \eps^{-\eps^{-\eps^{-d_W}d_U\l(1 + d_V) + \frac{d_W}{2}\r -d_W} \ls \l d_U\l(1 + d_V) + \frac{d_W}{2}\r + d_W\frac{(d_V+1)}{2} + (d_V+1)  \r \rs }.
\]
We also note that \begin{align*}
    K_2 &\asymp n_{c_U}^2\log n_{c_U}+n_{c_U}^2(d_V +1)(n_{c_W}+1)\log(\varepsilon^{-1}) + n_{c_U}^2 \log(2^{d_V +1} (C \sqrt{d_V})^{d_V} )\\
        &+ n_{c_U}^2(d_V +1)\log(2^{n_{c_W}+1}(C''\sqrt{n_{c_W}})^{n_{c_W}}) \\
        &\asymp n_{c_U}^2\log n_{c_U}+n_{c_U}^2(d_V +1)n_{c_W}\log(\varepsilon^{-1}) + n_{c_U}^2(d_V+1)(n_{c_W}+1)\log(2) \\
        &+ n_{c_U}^2(d_V+1)n_{c_W}\log(C'') + n_{c_U}^2(d_V+1)\frac{n_{c_W}}{2}\log(n_{c_W}) \\
        &\asymp n_{c_U}^2\log n_{c_U}+n_{c_U}^2(d_V +1)\eps^{-d_W}\log(\varepsilon^{-1}) + n_{c_U}^2(d_V+1)\frac{d_W}{2}\eps^{-d_W}\log(\eps^{-1})\\
        &= n_{c_U}^2\log n_{c_U} + n_{c_U}^2\eps^{-d_W}\log(\eps^{-1})(d_V+1)\l 1 + \frac{d_W}{2} \r \\
        &\asymp \eps^{-2\eps^{-d_W}d_U\l(1 + d_V) + \frac{d_W}{2}\r - d_W} \log(\eps^{-1}) d_U\l(1 + d_V) + \frac{d_W}{2}\r \\
        &+ \eps^{-2\eps^{-d_W}d_U\l(1 + d_V) + \frac{d_W}{2}\r -d_W}\log(\eps^{-1}) (d_V+1) \l 1 + \frac{d_W}{2} \r \\
        &= \eps^{-2\eps^{-d_W}d_U\l(1 + d_V) + \frac{d_W}{2}\r -d_W}\log(\eps^{-1}) \ls d_U\l(1 + d_V) + \frac{d_W}{2}\r + (d_V+1) \l 1 + \frac{d_W}{2} \r  \rs
\end{align*}
which therefore yields:\begin{align}
    H^{n_{c_U}} K_2 &\asymp \eps^{-\eps^{-\eps^{-d_W}d_U\l(1 + d_V) + \frac{d_W}{2}\r -d_W} \ls \l d_U\l(1 + d_V) + \frac{d_W}{2}\r + d_W\frac{(d_V+1)}{2} + (d_V+1)  \r \rs -2\eps^{-d_W}d_U\l(1 + d_V) + \frac{d_W}{2}\r -d_W} \notag \\
    &\times \log(\eps^{-1}) \ls d_U\l(1 + d_V) + \frac{d_W}{2}\r + (d_V+1) \l 1 + \frac{d_W}{2} \r  \rs \label{eq:rem:totalMultiple:HK2}.
\end{align}
Combining Eqs. \eqref{eq:rem:totalMultiple:PK3}, \eqref{eq:rem:totalMultiple:NK1} and \eqref{eq:rem:totalMultiple:HK2}, we conclude that \begin{align*}
    N_\# &\asymp \eps^{-\eps^{-\eps^{-d_W}d_U\l(1 + d_V) + \frac{d_W}{2}\r -d_W} \ls \l d_U\l(1 + d_V) + \frac{d_W}{2}\r + d_W\frac{(d_V+1)}{2} + (d_V+1)  \r \rs -2\eps^{-d_W}d_U\l(1 + d_V) + \frac{d_W}{2}\r -d_W} \notag \\
    &\times \log(\eps^{-1}) \ls d_U\l(1 + d_V) + \frac{d_W}{2}\r + (d_V+1) \l 1 + \frac{d_W}{2} \r  \rs \\
    &=: \eps^{-\gamma_2 \eps^{-\gamma_1 \eps^{-d_W} - d_W} - \gamma_3 \eps^{-d_W} - d_W} \log(\eps^{-1}) \gamma_4. 
\end{align*}

We therefore have \[
\log(N_\#) \asymp \l \gamma_2 \eps^{-\gamma_1 \eps^{-d_W}-d_W} + \gamma_3 \eps^{-d_W} - d_W \r + \log(\log(\eps^{-1})) + \log(\gamma_4)  \asymp \gamma_2 \eps^{-\gamma_1 \eps^{-d_W}-d_W}.
\]
Taking an additional logarithm, we obtain \[
\log(\log(N_\#)) \asymp \log(\gamma_2) + \l \gamma_1 \eps^{-d_W} - d_W \r \log(\eps^{-d}) \asymp \gamma_1 \eps^{-d_W} \log(\eps^{-d}).
\]
Proceeding as in Remark \ref{rem:totalParametersOperatorLearning}, with the Lambert function inversion, this yields the final scaling \[
\eps \asymp  \l \frac{\log \log N_\# }{ \log \log \log N_\#} \r^{-1/d_W}.
\]
As expected, moving from the single operator to the multiple operator setting incurs a less favorable scaling of the total number of parameters (see Table \ref{tab:balancingComplexity}), consistent with the higher representational complexity required.
\end{remark}

\begin{remark}[Improved rates with additional low-dimensional structure]
    If the input function spaces $W$ and $U$ admit a finite orthonormal basis representations and the discretization grids satisfy stable linear reconstruction properties as in \cite[Assumption 4]{liu2024neuralscalinglawsdeep}, one can expect substantially improved approximation rates. In particular, under these additional structural assumptions, one should observe at least a transition from double–iterated to single–iterated logarithmic convergence, potentially recovering the rates observed for single–operator learning in the general setting (see Remarks~\ref{rem:totalParametersOperatorLearning} and~\ref{rem:complexityOrder}, and Table~\ref{tab:balancingComplexity}).
\end{remark}

\begin{remark}[A review of different multiple operator network architectures] \label{rem:mnoScalingLaws}
By combining the proof of Theorem \ref{thm:main:multipleOperatorApproximation}, Remark \ref{rem:uniformOperatorApproximation} and Remark \ref{rem:alternative}, we can prove scaling laws for various network architectures depending on the assumptions we make on the map $G$. In particular, we replicate the proof of Theorem \ref{thm:main:multipleOperatorApproximation} for $G$ satisfying \begin{align*}
 	&G:\{ \alpha:\Omega_W \mapsto \bbR \spaceBar \Vert \alpha \Vert_{\Lp{\infty}} \leq \beta_W \} \mapsto \mathcal{G} \qquad \text{where }\\
 	&\mathcal{G} = \Big\{ G[\alpha]\spaceBar G[\alpha]:\{ u:\Omega_U^{(\alpha)} \mapsto \bbR \spaceBar \Vert u \Vert_{\Lp{\infty}} \leq \beta_{U^{(\alpha)}} \} \mapsto V^{(\alpha)} \text{ and } \\
 	&\quad \Vert G[\alpha][u_1] - G[\alpha][u_2]\Vert_{\Lp{\infty}(\Omega_{V^{(\alpha)}})} \leq L^{(\alpha)}\Vert u_1 - u_2 \Vert_{\Lp{r^{\alpha}}(\Omega_{U^{\alpha}})} \Big\}. 
 \end{align*}
We start by considering the set \(
\mathcal{S} := \bigcup_{\alpha \in W} (U^{\alpha} \times \Omega_{V^{(\alpha)}})
\)
where $(U^{\alpha} \times V^{(\alpha)})$ are such that $G[\alpha]:\{ u:\Omega_U^{(\alpha)} \mapsto \bbR \spaceBar \Vert u \Vert_{\Lp{\infty}} \leq \beta_{U^{(\alpha)}} \} \mapsto V^{(\alpha)}$ for $\alpha \in W$. Furthermore, assume that $G$ satisfies \[
\Vert G(\alpha_1) - G(\alpha_2) \Vert_{\Lp{\infty}(\mathcal{S})} \leq L_G \Vert \alpha_1 - \alpha_2 \Vert_{\Lp{r_G}(\Omega_U)} \]
for $\alpha_1,\alpha_2 \in \{ \alpha:\Omega_W \mapsto \bbR \spaceBar \Vert \alpha \Vert_{\Lp{\infty}} \leq \beta_W \}$. By the axiom of choice, we can select an element $s \in \mathcal{S}$ such that $s^{(\alpha)} = (u^{(\alpha)},x^{(\alpha)}) \in (U^{\alpha} \times \Omega_{V^{(\alpha)}})$. Then, we define the functional $f(\alpha) = G[\alpha][u^{(\alpha)}](x^{\alpha})$ and, by the above assumption on Lipschitz continuity of $G$, we deduce that $f:\{ \alpha:\Omega_W \mapsto \bbR \spaceBar \Vert \alpha \Vert_{\Lp{\infty}} \leq \beta_W \} \mapsto \bbR$ is Lipschitz. We apply Theorem \ref{thm:back:functionalApproximationLinfty} to obtain that \[
\sup_{\alpha \in W} \left| f(\alpha) - \sum_{p=1}^{P^{n_{c_W}}} f(\alpha_k)l_p(P_{\mathcal{C}_W}(\alpha)) \right| = \sup_{\alpha \in U} \left| G[\alpha][u^{(\alpha)}](x^{(\alpha)}) - \sum_{p=1}^{P^{n_{c_W}}} G[\alpha_p][u^{(\alpha_p)}](x^{(\alpha_p)}) l_p(P_{\mathcal{C}_W}(\alpha)) \right| \leq \frac{\eps}{2}. 
\]
It now remains to approximate the $p$ operators $G[\alpha_p]$ by any of the architectures in Remark \ref{rem:uniformOperatorApproximation} or Remark \ref{rem:alternative}. We summarize the final multiple operator architectures in Tables \ref{tab:architectures} and \ref{tab:architectures2} (only for the first alternative formulation). From the latter, we note that our scaling laws can be transferred, in particular, to the $\ScalingNetwork$ and MIONet \cite{mionet} architectures.

\begin{table}[h]
\centering
\renewcommand{\arraystretch}{2.8}
\setlength{\tabcolsep}{6pt}
\begin{adjustbox}{max width=0.9\paperwidth,center}
\begin{tabularx}{\paperwidth}{>{\bfseries}p{2.6cm} Y @{\hskip 4em} Y}
\toprule
Network type \textbackslash\ Assumptions 
& $U^{(\alpha_p)}$ and $V^{(\alpha_p)}$ distinct 
& $U^{(\alpha_p)} = U$ \\
\midrule
Exact 
& {$\displaystyle 
   \sum_{p=1}^{P^{n_{c_W}}}
   \sum_{\ell=1}^{(N^{(p)})^{d_{V^{(p)}}}} \sum_{k=1}^{H^{(p)}} 
   G[\alpha_p][u_k^{(p)}](v_{\ell}^{(p)}) \,
   l_p(\bm{\alpha}) \, b_{pk}(\ub^{(p)}) \, \tau_{p\ell}(x)$}
& {$\displaystyle 
   \sum_{p=1}^{P^{n_{c_W}}}
   \sum_{\ell=1}^{(N^{(p)})^{d_{V^{(p)}}}} \sum_{k=1}^{H^{n_{c_U}}} 
   G[\alpha_p][u_k](v_{\ell}^{(p)}) \,
   l_p(\bm{\alpha}) \, b_{k}(\ub) \, \tau_{p\ell}(x)$} \\

Alternative
& {$\displaystyle 
   \sum_{p=1}^{P^{n_{c_W}}} \sum_{k=1}^{H^{(p)}} 
   l_p(\bm{\alpha}) \, b_{pk}(\ub^{(p)}) \, \hat{\tau}_{pk}(x)$}
& {$\displaystyle 
   \sum_{p=1}^{P^{n_{c_W}}} \sum_{k=1}^{H^{n_{c_U}}} 
   l_p(\bm{\alpha}) \, b_{k}(\ub) \, \hat{\tau}_{pk}(x)$} \\

\bottomrule
\end{tabularx}
\end{adjustbox}
\caption{Multiple operator network architectures with distinct or partially fixed $U^{(\alpha_p)}$, $V^{(\alpha_p)}$. We write $H^{(p)} = H^{n_{c_{U^{(\alpha_p)}}}}$ and $N^{(p)} = N^{(\alpha_p)}$.}
\label{tab:architectures}
\end{table}

\begin{table}[h]
\centering
\renewcommand{\arraystretch}{2.8}
\setlength{\tabcolsep}{6pt}
\begin{adjustbox}{max width=0.9\paperwidth,center}
\begin{tabularx}{\paperwidth}{>{\bfseries}p{2.6cm} Y @{\hskip 4em} Y}
\toprule
Network type \textbackslash\ Assumptions 
& $V^{(\alpha_p)} = V$ 
& $U^{(\alpha_p)} = U$ and $V^{(\alpha_p)} = V$ \\
\midrule
Exact 
& {$\displaystyle 
   \sum_{p=1}^{P^{n_{c_W}}}\sum_{k=1}^{H^{(p)}} 
   \sum_{\ell=1}^{N^{d_{V}}} 
   G[\alpha_p][u_k^{(p)}](v_{\ell}) \,
   l_p(\bm{\alpha}) \, b_{pk}(\ub^{(p)}) \, \tau_{\ell}(x)$}
& {$\displaystyle 
   \sum_{p=1}^{P^{n_{c_W}}}\sum_{k=1}^{H^{n_{c_U}}} 
   \sum_{\ell=1}^{N^{d_{V}}} 
   G[\alpha_p][u_k](v_{\ell}) \,
   l_p(\bm{\alpha}) \, b_{k}(\ub) \, \tau_{\ell}(x)$} \\

Alternative 
& {$\displaystyle 
   \sum_{p=1}^{P^{n_{c_W}}} \sum_{k=1}^{H^{(p)}} 
   l_p(\bm{\alpha}) \, b_{pk}(\ub^{(p)}) \, \hat{\tau}_{pk}(x)$}
& {$\displaystyle 
   \sum_{p=1}^{P^{n_{c_W}}} \sum_{k=1}^{H^{n_{c_U}}} 
   l_p(\bm{\alpha}) \, b_{k}(\ub) \, \hat{\tau}_{pk}(x)$} \\
\bottomrule
\end{tabularx}
\end{adjustbox}
\caption{Multiple operator network architectures with partially or fully fixed $U^{(\alpha_p)}$, $V^{(\alpha_p)}$. We write $H^{(p)} = H^{n_{c_{U^{(\alpha_p)}}}}$ and $N^{(p)} = N^{(\alpha_p)}$.}
\label{tab:architectures2}
\end{table}

\end{remark}

\section{Proofs} \label{sec:proofs}
In this section, we present the proofs. We first establish two versions of the universal approximation theorem for multiple nonlinear operators, then address scaling laws for functional and single operator approximation. Finally, we conclude with the proof of the scaling laws in the multiple operator approximation setting.

\subsection{Proof of Theorem \ref{thm:mainResult:universalApproximationI}}
\label{sec_mainResult_universalApproximationI}
In this section, we prove the universal approximation property for the $\ScalingNetwork$ and $\UAPNetwork$ networks in $\Lp{\infty}$. The intuition behind the proof of Theorem \ref{thm:mainResult:universalApproximationI} parallels our earlier discussion in Section~\ref{sec:back:operator}. The key step is to sequentially separate the input variables of the operator $G$, thereby transforming the operator approximation problem into a sequence of function and functional approximation problems.

\begin{proof}[Proof of Theorem \ref{thm:mainResult:universalApproximationI}] 

We first observe that any multiple operator network of the form
\[
\sum_{k=1}^N \sum_{\ell=1}^{M} \tau_k(x)\, b_{k\ell}(u)\, L_{k\ell}(\alpha)
\]
can be re-indexed into a $\ScalingNetwork$ of the form
\[
\sum_{p=1}^{P} \sum_{k=1}^{H^{(p)}} l_p(\alpha)\, b_{pk}(u)\, \tau_{pk}(x),
\]
where $l_p$, $b_{pk}$, and $\tau_{pk}$ retain the same structure as in Definition~\ref{def:uapNetwork}. Since the universal approximation statement in \eqref{eq:thm4} does not rely on explicit scalings for $N$ and $M$, this re-indexing leaves the result unaffected. Consequently, Theorem~\ref{thm:mainResult:universalApproximationI} may be established for one representation and inferred for the other. Without loss of generality, we therefore present the proof using the $\UAPNetwork$ architecture.

Let $\varepsilon> 0$. For an arbitrary, fixed $\alpha \in W$ and $u\in V$, consider the function $f(x) = G[\alpha][u](x)$. By assumption, $f \in \mathcal{F}:=G[W][U]$ which is a compact subset of $V = \Ck{0}(\Omega_V)$ and thus by Theorem~\ref{theory_1},  we can find $N\in\mathbb{N},\eta_k\in\mathbb{R}, \omega_k\in\mathbb{R}^n$ such that:
\begin{align*}
    \left|f(x) -\sum_{k=1}^N c_k( f(\cdot))\sigma(\omega_k\cdot x+\zeta_k)\right|<\varepsilon /3,
\end{align*}
where $c_i(f)$ are continuous linear functions. The approximation results hold for all $f \in \mathcal{F}$, and since $f(\cdot)=G[\alpha][u](\cdot)$ the coefficients $c_k( f(\cdot)) = c_k(G[\alpha][u](\cdot))$ are continuous functionals mapping $W \times U \rightarrow \mathbb{R}$. 

For each $k$ and arbitrary fixed $\alpha$,  define $F^{(k)}:U \rightarrow \mathbb{R}$ by $$F^{(k)}(u):=c_k( G[\alpha][u])$$ which is a continuous functional with respect to $u$.  Similarly to \cite{ChenChen1995}, by the Tietze Extension theorem  the functionals $F^{(k)}$ are extended to continuous functionals $F_*^{(k)}$  on all of $U^*$ from Lemma \ref{lemma2}, so that $F_*^{(k)}(u) =F^{(k)}(u)$ for all $u\in U$. Since $U^*$ is a compact set, there exists $\delta>0$ such that 
$$\left|F_*^{(k)}(u_1) - F_*^{(k)}(u_2) \right| < \frac{\varepsilon}{6L_1}$$ for all $u_1,u_2\in U^*$ with $\|u_1-u_2\|_{\Ck{0}(\Omega_U)}<\delta$, and $L_1 = \sum_{k=1}^N\sup\limits_{x\in \Omega_V}  \left| \sigma(\omega_k \cdot x + \zeta_k) \right|.$ The extension is needed since the construction of the functions on the $\eta_k$-net by Equation~\ref{u_etak} may reside  in $U^*\setminus U$.

Let  $\delta_{k}<\delta$,  where $\delta_{k}$ is defined in~\eqref{delta_k}. (Abusing notation, we select $\delta_k$ satisfying $\delta_k < \delta$, independently of the $k^{th}$ element in the sequence of~\eqref{delta_k}, while retaining the subscript for simplicity.) Then by Lemma~\ref{lemma2}(2), there exists $u_{\eta_{k}}\in U_{\eta_{k}} \subseteq U^*$ with 
$ \| u - u_{\eta_{k}} \|_{\Ck{0}(\Omega_U)} < \delta_{k}<\delta$ which implies 
% $$\left|F_*^{(k)}(u) - F_*^{(k)}( u_{\eta_{k}} ) \right| < \frac{\varepsilon}{6L_1}.$$ 
$$\left|F_*^{(k)}(u) - F_*^{(k)}( u_{\eta_{k}} ) \right| < \frac{\varepsilon}{6L_1}.$$ 
By Lemma~\ref{lemma2}(1), $F_*^{(k)}( u_{\eta_{k}} )$ is a continuous functional defined on the compact set $U_{\eta_{k}}$ with dimension $n(\eta_{k})$ and thus is equivalent to a continuous function (abusing notation) $F_*^{(k)}:\mathbb{R}^{n(\eta_{k})} \rightarrow \mathbb{R}$. Therefore, by Theorem~\ref{theory_1}, we can find $M\in\mathbb{N}$,  $\xi_{kil},\theta_{ki}$ and $x_l\in\eta_{k}-$ net as defined in~\eqref{etaknet} on $\Omega_U$, such that  
\begin{equation*}
    \left|F_*^{(k)}( u_{\eta_{k}} )-\sum_{i=1}^M c_{ki}(F_*^{(k)}(\cdot)) \sigma \left( \sum_{l=1}^{n({\eta_{k}})} \xi_{kil} u_{\eta_{k}}(x_l) + \theta_{ki} \right) \right|<\frac{\varepsilon}{6L_1}
\end{equation*}
which holds for all $k$. Setting the value $m=n(\eta_{k})$ and 
since $ u_{\eta_{k}}(x_l)  =  u(x_l) $ by~\eqref{u_etak}, we have that \begin{align*}
 &\left|F^{(k)}( u)-\sum_{i=1}^M c_{ki}(F_*^{(k)}(\cdot)) g \left( \sum_{l=1}^{m} \xi_{kil} u(x_l) + \theta_{ki} \right) \right|\\
  &\leq \left|F^{(k)}( u)- F_*^{(k)}( u_{\eta_{k}} ) \right| +\left| F_*^{(k)}( u_{\eta_{k}} )-\sum_{i=1}^M c_{ki}(F_*^{(k)}(\cdot)) \sigma \left( \sum_{l=1}^{m} \xi_{kil} u(x_l) + \theta_{ki} \right) \right|\\
     &<\frac{\varepsilon}{3L_1},
\end{align*}
where we also use the fact that $F^{(k)}( u)=F_*^{(k)}(u)$ since $u\in V$.

The extension operator $E$, from the Tietze extension theorem, is a continuous operator~\cite{dugundji1951extension}, and thus the coefficient of the expansions,  
$$c_{ki}( F_*^k(\cdot)) = c_{ki}(E(c_k(F^k(\cdot))))=c_{ki}(E(c_k(G[\alpha][\cdot]))),$$ are a continuous functionals depending on $\alpha$. Hence, we can provide a similar argument for the approximation for $H^{(i,k)}:W \rightarrow \mathbb{R}$ defined by  
$$H^{(k,i)}(\alpha):=c_{ki}(E(c_k(G[\alpha,\cdot]))).$$
By the Tietze extension theorem, we extend $H^{(i,k)}$ to a continuous functional $H^{(k,i)}_*$  on all of $W^*$ (defined in Lemma \ref{lemma2}) with $H^{(k,i)}_*(\alpha) =H^{(k,i)}(\alpha)$ for all $\alpha \in W$. Since $W^*$ is a compact set, there exists $\delta'>0$ such that 
\begin{equation*}
    \left|H^{(k,i)}_*(\alpha_1)-H^{(k,i)}_*(\alpha_2)\right|<\frac{\varepsilon}{6L_2}
\end{equation*}
 for all $\alpha_1,\alpha_2\in\mathcal{A}^*$ with $\|\alpha_1-\alpha_2\|_{C(\Omega_W)}<\delta'$, and 
 \begin{equation*}
     L_2 = \sum_{k=1}^N\sum_{i=1}^M\sup_{u \in U, x\in \Omega_V} \left|\sigma \left( \sum_{l=1}^m \xi_{kil} u(x_l) + \theta_{ki} \right) \cdot \sigma(\omega_k \cdot x + \zeta_k) \right|,   
\end{equation*}
where $L_2$ is finite since the terms in the absolute value are continuous functions and the sets are compact.
We can find an $\eta_{i}$-net defined on $\Omega_W$, with $\delta'_{{i}}<\delta'$ and by Lemma~\ref{lemma2} there exists $\alpha_{\eta_{i}}\in W_{\eta_{i}} \subseteq W^*$ (where $W_{\eta_i} = \{\alpha_{\eta_i}: \alpha \in W\}$) with $\|\alpha-\alpha_{\eta_{i}}\|<\delta_{i}<\delta$, which implies  
\begin{equation*}
  \left|H_*^{(k,i)}(\alpha)-H_*^{(k,i)}(\alpha_{\eta_{i}})\right|<\frac{\varepsilon}{6L_2}.  
\end{equation*}
The functionals $H_*^{(k,i)}(\alpha_{\eta_{i}})$ are equivalent to continuous functions defined on the compact set $W_{\eta_{i}}$ of dimension $n(\eta_{i})$, i.e., $H_*^{(k,i)}: \mathbb{R}^{n(\eta_{i})}\rightarrow \mathbb{R}$. By Theorem~\ref{theory_1}, we can find $P\in\mathbb{N}$, $c_{kij},\varphi_{kijh},\rho_{kij}$  and $z_h\in \eta_{i}$- net defined on $\Omega_W$, such that 
\begin{equation*}
    \left|H_*^{(k,i)}(\alpha_{\eta_{i}})-\sum_{j=1}^P c_{kij} \,  \sigma \left( \sum_{h=1}^{n(\eta_{i})} \varphi_{kijh} \alpha_{\eta_{i}}(z_h) + \rho_{kij} \right)  \right|<\frac{\varepsilon}{6L_2},
\end{equation*}
which holds for all $(k,i)$. Taking $p=n(\eta_{i})$ and recalling that $\alpha_{\eta_{i}}(z_h)  = \alpha(z_h)$, we have
\begin{equation*}
        \left|H_*^{(k,i)}(\alpha)-\sum_{j=1}^P c_{kij} \,  \sigma \left( \sum_{h=1}^{p} \varphi_{kijh} \alpha(z_h) + \rho_{kij} \right)  \right|<\frac{\varepsilon}{3L_2},
\end{equation*}
for all $\alpha \in W$.

Altogether,  the following holds for all  $(\alpha, u, y) \in W \times U \times \Omega_V$:
\begin{align*}
\left| G[\alpha][u](x) - \sum_{k=1}^N \sum_{i=1}^M \sum_{j=1}^P c_{kij} \,  \sigma \left( \sum_{h=1}^p \varphi_{kijh} \alpha(z_h) + \rho_{kij} \right) \cdot \sigma \left( \sum_{l=1}^m \xi_{kil} u(x_l) + \theta_{ki} \right) \cdot \sigma(\omega_k \cdot x + \zeta_k) \right| < \varepsilon,
\end{align*}
which concludes the proof.

\end{proof}

If $\alpha$ is finite dimensional, then the proof simplifies and the following corollary holds. 
\begin{corollary}\label{theory_4c}
Assume the same setting as in Theorem \ref{thm:mainResult:universalApproximationI}.
Then, for any \( \varepsilon > 0 \), there exist a network as defined in Equation~\ref{eq:mole2}, such that
\begin{equation*}%\label{eq:thm4c}
    \left| G[\alpha][u] (y) - \UAPNetworkFinite[\alpha][u](x) \right| < \varepsilon
\end{equation*}
holds for all \(\alpha\in \Omega_W,  u \in U \) and \( x \in \Omega_V \).
\end{corollary}

\subsection{Proof of Theorem \ref{thm:mainResult:universalApproximationII}}
\label{sec_mainResult_universalApproximationII}
In this section, we prove the universal approximation property for the $\ScalingNetwork$ and $\UAPNetwork$ networks in $\Lp{\infty}$. We start with the next lemma which allows us to reformulate $\UAPNetwork$ with orthonormal components.

\begin{lemma}[$\UAPNetwork$ network with orthonormal trunk and branch networks] \label{lem:orthogonal}
    Assume that Assumptions \ref{assumption:Main:assumptions:A2}, \ref{assumption:Main:assumptions:S1}, \ref{assumption:Main:assumptions:S2}, \ref{assumption:Main:assumptions:S3}, \ref{assumption:Main:assumptions:M1} and \ref{assumption:Main:assumptions:M2} hold. Then, we can re-write any $\UAPNetwork$ network defined in \eqref{eq:mole1} as 
    \[
    \UAPNetwork[\alpha][u](x) = \sum_{k=1}^N \sum_{\bar{\ell}=1}^{N \cdot M} \bar{\tau}_k(x) \bar{b}_{k\bar{\ell}}(u) \bar{L}_{k\bar{\ell}}(\alpha)
    \]
    where $\{\bar{\tau}_k\}_{k=1}^N$ is a set of orthonormal neural networks with one hidden layer and a linear output layer with respect to the inner product in $\Lp{2}_\lambda(\Omega_U)$, for $1 \leq k \leq N$, $\{\bar{b}_{k\bar{\ell}}\}_{\bar{\ell}=1}^{N \cdot M}$ is a set of orthonormal neural networks with one hidden layer and a linear output layer with respect to the $\Lp{2}((U,\mu),\bbR)$ inner product and $\{\bar{L}_{k\bar{\ell}}\}_{\bar{\ell}=1}^{N \cdot M}$ is a set of neural networks with one hidden layer and a linear output layer.
\end{lemma}

\begin{proof}
    We first recall that a $\UAPNetwork$ network may be written as \[
\UAPNetwork[\alpha][u](x) = \sum_{k=1}^N \sum_{i=1}^M \tau_k(x) b_{ki}(u) L_{ki}(\alpha)
\]
for $\tau_k(x) = \sigma(\omega_k \cdot x + \zeta_k)$, $b_{ki}(u) = \sigma\l \sum_{l = 1}^{m} \xi_{kil} u(x_l) + \theta_{ki} \r$ and \[L_{ki}(\alpha) = \sum_{j=1}^P c_{kij} \sigma\l \sum_{h=1}^p \varphi_{kijh} \alpha(z_h) + \rho_{kij} \r.\] 

We introduce the following notation: \begin{itemize}
    \item $\tau(x) \in \mathbb{R}^N$ the vector $\{\tau_k(x)\}_{k=1}^N$;
    \item $b_k(u) \in \mathbb{R}^{M}$ the vector $\{b_{ki}(u)\}_{i=1}^{i=M}$ for $1 \leq k \leq N$;
    \item $L_k(\alpha) \in \mathbb{R}^{M}$ the vector $\{L_{ki}(\alpha)\}_{i=1}^{i=M}$ for $1 \leq k \leq N$;
    \item $T(u,\alpha) \in \bbR^{N}$ the vector $\{\langle b_k(u), L_k(\alpha) \rangle_{\ell^2(\bbR^M)} \}_{k=1}^N$.
\end{itemize}
This allows us to re-write
\[
\sum_{k=1}^N \sum_{i=1}^M \tau_k(x) b_{ki}(u) L_{ki}(\alpha)
= \sum_{k=1}^N \tau_k(x) \langle b_k(u), L_k(\alpha) \rangle_{\ell^2(\bbR^M)} = \langle \tau(x) , T(u,\alpha) \rangle_{\ell^2(\bbR^N)}.    
\]

The functions $\{\tau_k(x)\}_{k=1}^N$ are a finite set of $\Lp{2}_\lambda(\Omega_U)$ functions by Assumptions \ref{assumption:Main:assumptions:A2} and \ref{assumption:Main:assumptions:S3}. Hence, by the Gram-Schmidt orthogonalization process, there exists an invertible matrix $\mathcal{Z} \in \mathbb{R}^{N\times N}$ such that $\mathcal{Z} \tau(y)$ is a set of orthogonal function with respect to the $\Lp{2}_\lambda(K)$ inner product (otherwise, remove any terms from the summation that are redundant and reindex the summation). Let $Z = \mathcal{Z}^{-T}$ and we observe that 
\begin{align}
    \langle \tau(x) , T(u,\alpha) \rangle_{\ell^2(\bbR^N)} &= \langle \mathcal{Z} \tau(y) , ZT(u,\alpha) \rangle_{\ell^2(\bbR^N)} \notag \\
    &= \sum_{k=1}^N \ls \mathcal{Z} \tau(x) \rs_k \ls Z T(u,\alpha) \rs_k. \notag
\end{align}
We first note that the vector $\mathcal{Z} \tau(x) \in \bbR^N$ has entries \begin{align*}
    \ls \mathcal{Z} \tau(x) \rs_k = \sum_{r=1}^N \ls \mathcal{Z} \rs_{kr} \sigma(\omega_r \cdot x + \zeta_r)
\end{align*}
which are neural networks with one hidden layer and a linear output layer. Second, we consider the vector $Z T(u,\alpha) \in \bbR^N$ which has entries \begin{align}
    \ls Z T(u,\alpha) \rs_k &= \sum_{r=1}^N \ls Z \rs_{kr} \langle b_r(u), L_r(\alpha) \rangle_{\ell^2(\bbR^M)}  \notag \\
    &= \sum_{r=1}^N \sum_{s=1}^M \ls Z \rs_{kr} b_{rs}(u) L_{rs}(u)  \notag \\
    &= \sum_{r=1}^N \sum_{s=1}^M \sum_{j=1}^P \ls Z \rs_{kr} \sigma\l \sum_{l=1}^m \xi_{rsl} u(x_l) + \theta_{rs} \r \sigma\l \sum_{h=1}^p \varphi_{rsjh} \alpha(z_h) + \rho_{rsj} \r. \label{eq:lem:orthogonal:changeVariables1}
\end{align}
We proceed to the following change of variables: we consider the index variable $1\leq \bar{\ell} \leq N \cdot M$ and replace every occurrence of $r$ by $\lfloor (\bar{\ell}-1)/M \rfloor +1$ and every occurrence of $s$ by $((\bar{\ell}-1) \mod M)+1$ where for $n,m \in \bbN$, $n \mod m$ denotes the remainder of the integer division of $n$ by $m$. This allows us to define \begin{itemize}
    \item $\ls \tilde{Z} \rs_{k,\bar{\ell}} = \ls Z \rs_{k,\lfloor (\bar{\ell}-1)/M \rfloor +1}$
    \item $\tilde{\xi}_{\bar{\ell},l} = \xi_{\lfloor (\bar{\ell}-1)/M \rfloor +1,\, ((\bar{\ell}-1) \mod M)+1, \,l}$
    \item $\tilde{\theta}_{\bar{\ell}} = \theta_{\lfloor (\bar{\ell}-1)/M \rfloor +1,\, ((\bar{\ell}-1) \mod M)+1}$
    \item $\tilde{\varphi}_{\bar{\ell},j,h} = \varphi_{\lfloor (\bar{\ell}-1)/M \rfloor +1,\, ((\bar{\ell}-1) \mod M)+1, \, j,h}$
    \item $\tilde{\rho}_{\bar{\ell},j} = \rho_{\lfloor (\bar{\ell}-1)/M \rfloor +1,\, ((\bar{\ell}-1) \mod M)+1, \, j}$
\end{itemize}   
and we can continue from \eqref{eq:lem:orthogonal:changeVariables1}: \begin{align}
    \ls Z T(u,\alpha) \rs_k &= \sum_{\bar{\ell}=1}^{N \cdot M} \sigma \l \sum_{l=1}^m \tilde{\xi}_{\bar{\ell},l} u(x_l) + \tilde{\theta}_{\bar{\ell}} \r \l  \sum_{j=1}^P \ls \tilde{Z} \rs_{k,\bar{\ell}} \sigma\l \sum_{h=1}^p \tilde{\varphi}_{\bar{\ell},j,h} \alpha(z_h) + \tilde{\rho}_{\bar{\ell},j} \r \r \notag \\
    &=: \langle \tilde{b}_k(u), \tilde{L}_k(\alpha) \rangle_{\ell^2(\bbR^{N \cdot M})} \notag
\end{align}
where, for $1 \leq k \leq N$, $\tilde{b}_k(u) = \left\{ \tilde{b}_{k\bar{\ell}}(u) \right\}_{\bar{\ell}=1}^{N \cdot M} = \left\{ \sigma \l \sum_{l=1}^m \tilde{\xi}_{\bar{\ell},l} u(x_l) + \tilde{\theta}_{\bar{\ell}} \r \right\}_{\bar{\ell}=1}^{N \cdot M}$ and \[
\tilde{L}_k(\alpha) = \left\{ \tilde{L}_{k\bar{\ell}}(\alpha) \right\}_{\bar{\ell}=1}^{N \cdot M} = \left\{ \sum_{j=1}^P \ls \tilde{Z} \rs_{k,\bar{\ell}} \sigma\l \sum_{h=1}^p \tilde{\varphi}_{\bar{\ell},j,h} \alpha(z_h) + \tilde{\rho}_{\bar{\ell},j} \r \right\}_{\bar{\ell}=1}^{N \cdot M}.
\]
We note that, for $1 \leq k \leq N$, $\tilde{b}_k(u)$ is a set of neural network with one hidden layer while $\tilde{L}_k(\alpha)$ is a set of neural networks with one hidden layer and one linear output layer.

By defining the orthonormal set of functions $\bar{\tau}(x) = \mathcal{Z}\tau(x)$, we obtain that\begin{equation} \label{eq:lem:orthogonal:estimate1}
    \UAPNetwork[\alpha][u](x) = \langle \tau(x) , T(u,\alpha) \rangle_{\ell^2(\bbR^N)} = \sum_{k=1}^N \bar{\tau}_k(x) \langle \tilde{b}_k(u), \tilde{L}_k(\alpha) \rangle_{\ell^2(\bbR^{N \cdot M})}.
\end{equation}

For $1\leq k\leq N$ and $1\leq \bar{\ell} \leq N \cdot M$, every functional $\tilde{b}_{k\bar{\ell}}$ can be considered as a random variable mapping from the measure space $(U,\mu)$ into $(\bbR,\lambda)$. In particular, we have: \[
\int_U \tilde{b}_{k\bar{\ell}}(u)^2 \, \dd \mu(u) = \int_U \sigma \l \sum_{l=1}^m \tilde{\xi}_{\bar{\ell},l} u(x_l) + \tilde{\theta}_{\bar{\ell}} \r^2 \, \dd \mu(u) \leq \Vert \sigma \Vert_{\Lp{\infty}(\bbR)}^2 \mu(U)
\] 
and the latter is finite by Assumptions \ref{assumption:Main:assumptions:A2} and \ref{assumption:Main:assumptions:M2}. By \cite[p.6]{berlinet2004rkhs}, 
$\tilde{b}_{k\bar{\ell}}(u)$ is therefore in the Hilbert space $\Lp{2}((U,\mu),\bbR)$ endowed with the inner product \[
\langle f,g \rangle_{\Lp{2}((U,\mu),\bbR)} = \int_U f(u)g(u) \, \dd \mu(u)
\]
for $f,g \in \Lp{2}((U,\mu),\bbR)$. 

For $1 \leq k \leq N$, the functionals $\left\{ \tilde{b}_{k\bar{\ell}}(u) \right\}_{\bar{\ell}=1}^{N \cdot M}$ are a finite set of $\Lp{2}((U,\mu),\bbR)$ functionals. By the Gram-Schmidt orthogonalization process, there therefore exists an invertible matrix $\mathcal{Z}_k \in \mathbb{R}^{(N \cdot M) \times (N \cdot M)}$ such that $\mathcal{Z}_k \tilde{b}_k(u)$ is a set of orthogonal functionals with respect to the $\Lp{2}((U,\mu),\bbR)$ inner product. Continuing from \eqref{eq:lem:orthogonal:estimate1} and defining $Z_k = \mathcal{Z}_k^{-T}$, we have: \begin{align}
    \UAPNetwork[\alpha][u](y) &= \sum_{k=1}^N \bar{\tau}_k(x) \langle \mathcal{Z}_k \tilde{b}_k(u), Z_k\tilde{L}_k(\alpha) \rangle_{\ell^2(\bbR^{N \cdot M})}. \label{eq:lem:orthonormal:MOLE1}
\end{align}
Similarly to the above, for $1 \leq k \leq N$, the vector $\mathcal{Z}_k \tilde{b}_k(u)$ has entries \[
\ls \mathcal{Z}_k \tilde{b}_k(u) \rs_{\bar{\ell}} = \sum_{r=1}^{N \cdot M} \ls \mathcal{Z}_k \rs_{\bar{\ell}r} \tilde{b}_{kr}(u) = \sum_{r=1}^{N \cdot M} \ls \mathcal{Z}_k \rs_{\bar{\ell}r} \sigma \l \sum_{l=1}^m \tilde{\xi}_{r,l} u(x_l) + \tilde{\theta}_{r} \r
\]
which implies that $\mathcal{Z}_k \tilde{b}_k(u)$ is a set of neural networks with one hidden layer and a linear output layer. Furthermore, the vector $Z_k\tilde{L}_k(\alpha)$ has entries \begin{equation} \label{eq:lem:orthogonal:changeVariables2}
    \ls Z_k\tilde{L}_k(\alpha) \rs_{\bar{\ell}} = \sum_{r=1}^{N \cdot M} \ls Z_k \rs_{\bar{\ell},r} \tilde{L}_{k,r}(\alpha) = \sum_{r=1}^{N \cdot M} \sum_{j=1}^P \ls Z_k \rs_{\bar{\ell},r} \ls \tilde{Z} \rs_{k,r} \sigma \l \sum_{h=1}^p \tilde{\varphi}_{rjh} \alpha(z_h) + \tilde{\rho}_{rj} \r.
\end{equation}
We proceed to the following change of variables: we consider the index variable $1\leq \bar{s} \leq N \cdot M \cdot P$ and replace every occurrence of $r$ by $\left\lfloor \frac{\bar{s} - 1}{P} \right\rfloor + 1$ and every occurrence of $j$ by $((\bar{s}-1) \mod P)+1$. This leads us to define the following variables \begin{itemize}
    \item $\ls \bar{Z}_k \rs_{\bar{\ell},\bar{s}} = \ls Z_k \rs_{\bar{\ell},\left\lfloor \frac{\bar{s} - 1}{P} \right\rfloor + 1}$
    \item $\ls \bar{Z} \rs_{k,\bar{s}} = \ls \tilde{Z} \rs_{k,\left\lfloor \frac{\bar{s} - 1}{P} \right\rfloor + 1}$ 
    \item $\bar{\varphi}_{\bar{s},h} = \tilde{\varphi}_{\left\lfloor \frac{\bar{s} - 1}{P} \right\rfloor + 1, \, ((\bar{s}-1) \mod P)+1, \, h}$
    \item $\bar{\rho}_{\bar{s}} = \tilde{\rho}_{\left\lfloor \frac{\bar{s} - 1}{P} \right\rfloor + 1, \, ((\bar{s}-1) \mod P)+1}$
\end{itemize}
and we can continue from \eqref{eq:lem:orthogonal:changeVariables2}:
\begin{align*}
    \ls Z_k\tilde{L}_k(\alpha) \rs_{\bar{\ell}} = \sum_{\bar{s}=1}^{N \cdot M \cdot P} \ls \bar{Z}_k \rs_{\bar{\ell},\bar{s}} \ls \bar{Z} \rs_{k,\bar{s}} \sigma \l \sum_{h=1}^{p}\bar{\varphi}_{\bar{s},h} \alpha(z_h) + \bar{\rho}_{\bar{s}} \r.
\end{align*}
We note that $Z_k\tilde{L}_k(\alpha)$ is therefore a set of neural networks with one hidden layer and a linear output layer. For $1 \leq k \leq N$, by defining $\bar{b}_k(u) = \mathcal{Z}_k \tilde{b}_k(u) =: \{\bar{b}_{k\bar{\ell}}(u)\}_{\bar{\ell}=1}^{N \cdot M}$ and $\bar{L}_{k}(\alpha) = Z_k\tilde{L}_k(\alpha) =: \{\bar{L}_{k\bar{\ell}}(\alpha)\}_{\bar{\ell}=1}^{N \cdot M}$, we obtain from \eqref{eq:lem:orthonormal:MOLE1} that
\begin{align}
    \UAPNetwork[\alpha][u](x) &= \sum_{k=1}^{N} \bar{\tau}(x) \langle \bar{b}_k(u), \bar{L}_k(\alpha) \rangle_{\ell^2(\bbR^{N \cdot M})} \notag \\
    &= \sum_{k=1}^N \sum_{\bar{\ell}=1}^{N \cdot M} \bar{\tau}_k(x) \bar{b}_{k\bar{\ell}}(u) \bar{L}_{k\bar{\ell}}(\alpha) \notag
\end{align}
where $\{\bar{\tau}_k\}_{k=1}^N$ is a set of orthonormal neural networks with one hidden layer and a linear output layer with respect to the inner product in $\Lp{2}_\lambda(\Omega_V)$, for $1 \leq k \leq N$, $\{\bar{b}_{k\bar{\ell}}\}_{\bar{\ell}=1}^{N \cdot M}$ is a set of orthonormal neural networks with one hidden layer and a linear output layer with respect to the $\Lp{2}((U,\mu),\bbR)$ inner product and $\{\bar{L}_{k\bar{\ell}}\}_{\bar{\ell}=1}^{N \cdot M}$ is a set of neural networks with one hidden layer and a linear output layer.

\end{proof}

The proof of Theorem~\ref{thm:mainResult:universalApproximationII} follows a classical idea in approximation theory. The key step is to approximate measurable mappings by continuous ones, which is possible on a arbitrarily large subset by Lusin’s theorem \cite[Theorem 7.1.13]{Bogachev}. On this subset we apply Theorem \ref{thm:mainResult:universalApproximationI}, while on the remaining small complement the error is controlled via a clipping argument. 

\begin{proof}[Proof of Theorem \ref{thm:mainResult:universalApproximationII}]

Using the same argument as in the proof of Theorem \ref{thm:mainResult:universalApproximationI}, without loss of generality, we present the proof using the $\UAPNetwork$ architecture.

In the proof $C>0$ will denote a constant that can be arbitrarily large, independent of all our parameters that may change from line to line.

For $M > 0$, let us define the truncated operator \[
    G_M[\alpha] = \begin{cases}
        G[\alpha] &\text{if $\Vert G[\alpha] \Vert_{\Lp{2}(\mu \times \lambda)(U \times \Omega_V)} \leq M$} \\
        M \frac{G[\alpha]}{\Vert G[\alpha] \Vert_{\Lp{2}(\mu \times \lambda)(U \times \Omega_V)}} &\text{else,}
    \end{cases}
    \]
from which we deduce that $\Vert G_M[\alpha] \Vert_{\Lp{2}(\mu \times \lambda)(U \times \Omega_V)} \leq M$. We also note that for any function $\mathcal{N}[\alpha][u](x)$, we can upper bound the left-hand side of \eqref{eq::mainResult:universalApproximationII} as follows: \begin{align}
&\left\| G[\alpha][u](x) - \mathcal{N}[\alpha][u](x)\right\|_{\Lp{2}_{\nu \times \mu \times \lambda}(W \times U \times \Omega_V)} \notag \\
&\leq \left\| G[\alpha][u](x) -  G_M[\alpha][u](x) \right\|_{\Lp{2}_{\nu \times \mu \times \lambda}(W \times U \times \Omega_V)} + \left\| G_M[\alpha][u](x) - \mathcal{N}[\alpha][u](x)\right\|_{\Lp{2}_{\nu \times \mu \times \lambda}(W \times U \times \Omega_V)} \notag \\
&=: T_1 + T_2. \label{eq:universalApproximationII:reductionBounded}
\end{align}

We first show that $\lim_{M\to \infty} T_1 = 0$. In particular, \[
T_1^2 = \int_W \Vert G[\alpha][u](x) - G_M[\alpha][u](x) \Vert_{\Lp{2}_{\mu \times \lambda}(U \times \Omega_V)}^2 \, \dd \nu(\alpha) 
\]
and, for $\alpha \in W$, we note that: \begin{align}
    \Vert G[\alpha][u](x) - G_M[\alpha][u](x) \Vert_{\Lp{2}_{\mu \times \lambda}(U \times \Omega_V)}^2 &\leq C\l \Vert G[\alpha][u](x) \Vert_{\Lp{2}_{\mu \times \lambda}(U \times \Omega_V)}^2 + \Vert G_M[\alpha][u](x) \Vert_{\Lp{2}_{\mu \times \lambda}(U \times \Omega_V)}^2  \r \notag \\
    &\leq C \Vert G[\alpha][u](x) \Vert_{\Lp{2}_{\mu \times \lambda}(U \times \Omega_V)}^2 + M^2 \label{eq:universalApproximationII:estimate1}
\end{align} 
where we used the fact that $\Vert G_M[\alpha][u](x) \Vert_{\Lp{2}_{\mu \times \lambda}(U \times \Omega_V)} \leq M$ in \eqref{eq:universalApproximationII:estimate1}. Since \[
C \Vert G[\alpha][u](x) \Vert_{\Lp{2}_{\mu \times \lambda}(U \times \Omega_V)}^2 + M^2 \in \Lp{1}(\nu)
\]
by Assumptions \ref{assumption:Main:assumptions:M1} and \ref{assumption:Main:assumptions:O3}, we can apply the dominated convergence to obtain: \[
\lim_{M\to \infty} T_1^2 = \int_W \lim_{M\to \infty} \Vert G[\alpha][u](x) - G_M[\alpha][u](x) \Vert_{\Lp{2}_{\mu \times \lambda}(U \times \Omega_V)}^2 \, \dd \nu(\alpha).
\]
Now, for $\alpha \in W$, \begin{align}
    \Vert G[\alpha][u](x) - G_M[\alpha][u](x) \Vert_{\Lp{2}_{\mu \times \lambda}(U \times \Omega_V)} &\leq \Vert G[\alpha][u](x) \Vert_{\Lp{2}_{\mu \times \lambda}(U \times \Omega_V)} \one_{\{\Vert G[\alpha][u](x) \Vert_{\Lp{2}_{\mu \times \lambda}(U \times \Omega_V)} \geq M\}} \notag
\end{align}
where $\one$ is the indicator function. By Assumption \ref{assumption:Main:assumptions:O3}, we have that $G[\alpha][u](x) \in \Lp{2}_{\nu \times \mu \times \lambda}(W \times U \times \Omega_V)$ which implies that, $\nu$-a.e., $\Vert G[\alpha][u](x) \Vert_{\Lp{2}_{\mu \times \lambda}(U \times \Omega_V)} < \infty$. Hence, $\nu$-a.e., \begin{align*}
    \lim_{M\to \infty} \Vert G[\alpha][u](x) - G_M[\alpha][u](x) \Vert_{\Lp{2}_{\mu \times \lambda}(U \times \Omega_V)} &\leq \lim_{M\to \infty} \Vert G[\alpha][u](x) \Vert_{\Lp{2}_{\mu \times \lambda}(U \times \Omega_V)} \one_{\{\Vert G[\alpha][u](x) \Vert_{\Lp{2}_{\mu \times \lambda}(U \times \Omega_V)} \geq M\}} \\
    &=0
\end{align*}
from which we deduce that $\lim_{M\to \infty} T_1 = 0$: we can therefore pick $M = M(\eps/3)>\frac{2\eps}{9\nu(W)^{1/2}}$ large enough so that \begin{equation} \label{eq:universalApproximation:estimateT1}
    T_1 \leq \frac{\eps}{3}.
\end{equation}

We now tackle the $T_2$ term. We first note that the set $\Ck{0}(U,V)$ is a Polish space by \cite[Theorem 4.19]{kechris1995} since $U$ is a compact subset of the metric space $\Ck{0}(\Omega_U)$ and $V$ is Polish. By Assumptions  \ref{assumption:Main:assumptions:O1} and \ref{assumption:Main:assumptions:O3}, this implies that $G_M: W \mapsto \Ck{0}(U,V)$ is a Borel measurable map from a Polish space into another one. Define $\delta_1 = \frac{4\eps}{ (3M)^2}$: by \cite[Theorem 7.1.13]{Bogachev}, 
we can therefore find a compact set $W_K \subseteq W$ with $\nu(W\setminus W_K) < \delta_1 \eps$ such that $G_M: W_K \mapsto \Ck{0}(U,V)$ is continuous. Define $\delta_2 = 2(81 \nu(W) \mu(U) \lambda(\Omega_V))^{-1/2}$: for the latter map, we can apply Theorem \ref{thm:mainResult:universalApproximationI} to obtain a $\UAPNetwork$ network such that \begin{align}
    &\sup_{\alpha \in W_K} \left\| G_M[\alpha][u](x) - \UAPNetwork[\alpha][u](x)\right\|_{\Lp{2}_{\mu \times \lambda}(U \times \Omega_V)}\notag \\
    &\leq \left\| G_M[\alpha][u](x) - \UAPNetwork[\alpha][u](x)\right\|_{\Lp{\infty}_{\nu \times \mu \times \lambda}(W_K \times U \times \Omega_V)} \l \mu(U) \lambda(\Omega_U) \r^{1/2} \notag \\
    &\leq \delta_2 \eps (\mu(U) \lambda(\Omega_U))^{1/2} = \frac{2\eps}{9\nu(W)^{1/2}}\label{eq:universalApproximation:approximationNetworkK}
\end{align}
By Lemma \ref{lem:orthogonal}, the $\UAPNetwork$ network may be re-written as \[
\UAPNetwork[\alpha][u](x) = \sum_{k=1}^N \sum_{i=1}^M \tau_k(x) b_{ki}(u) L_{ki}(\alpha)
\]
where $\{\tau_k\}_{k=1}^N$ is a set of orthonormal neural networks with one hidden layer and a linear output layer with respect to the inner product in $\Lp{2}_\lambda(\Omega_V)$, for $1 \leq k \leq N$, $\{b_{ki}\}_{i=1}^{M}$ is a set of orthonormal neural networks with one hidden layer and a linear output layer with respect to the $\Lp{2}((U,\mu),\bbR)$ inner product and $\{L_{ki}\}_{i=1}^{M}$ is a set of neural networks with one hidden layer and a linear output layer.

In particular, this implies that for all $\alpha \in W$, 
\begin{align}
    \Vert \UAPNetwork[\alpha][u](x) \Vert_{\Lp{2}_{\mu \times \lambda}( U \times \Omega_V)}^2 &= \sum_{k=1}^N \sum_{i=1}^{M} \sum_{l=1}^{N} \sum_{j=1}^{M} L_{ki}(\alpha) L_{lj}(\alpha) \int_U b_{ki}(u)b_{lj}(u) \, \dd \mu(u) \int_{\Omega_V} \tau_{k}(x)\tau_l(x) \, \dd \lambda(x) \notag \\
    &= \sum_{k=1}^N \sum_{i=1}^{M} \sum_{j=1}^{M} L_{ki}(\alpha) L_{kj}(\alpha) \int_U b_{ki}(u)b_{kj}(u) \, \dd \mu(u) \label{eq:universalApproximation:orthonormalTau} \\
    &= \sum_{k=1}^N \sum_{i=1}^{M} L_{ki}(\alpha)^2 \label{eq:universalApproximation:orthonormalBeta} \\
    &=: \Vert L(\alpha) \Vert_{\ell^2(\bbR^{N\cdot M})}^2 \notag
\end{align}
where we used the orthonomality of $\{\tau_k\}_{k=1}^N$ for \eqref{eq:universalApproximation:orthonormalTau} and the fact that, for $1 \leq k \leq N$, $\{b_{k\bar{\ell}}\}_{\bar{\ell}=1}^{M}$ is a set of orthonormal functionals for \eqref{eq:universalApproximation:orthonormalBeta}.

Define $\delta_3 = \frac{2}{9\nu(W)^{1/2}}$ and the network \[
\overline{\UAPNetwork}[a][u](y) = \sum_{k=1}^N \sum_{i=1}^M \tau_k(x) b_{ki}(u) \gamma_{ki}(L_{ki}(\alpha))
\] 
where $\gamma(x):\mathbb{R}^{N \cdot M} \mapsto \bbR^{N \cdot M}$ is a ReLU neural network with coordinates $\{\gamma_{ki}\}_{k=1,i=1}^{k=N,i=M}$ such that \[
\begin{cases}
    \Vert \gamma(x) - x \Vert_{\ell^2(\bbR^{N\cdot M})} < \eps \delta_3 &\text{if $\Vert x \Vert_{\ell^2(\bbR^{N \cdot M})} \leq M + \frac{2\eps}{9\nu(W)^{1/2}}$} \\
    \Vert \gamma(x) \Vert_{\ell^2(\bbR^{N\cdot M})} \leq 2M &\text{for all $x\in \bbR^{N\cdot M}$}
\end{cases}
\] 
which exists due to \cite[Lemma C.2]{Lanthaler2022} and the fact that $M > \frac{2\eps}{9\nu(W)^{1/2}}$. We now estimate as follows, starting from the $T_2$ term in \eqref{eq:universalApproximationII:reductionBounded} where $\mathcal{N}[\alpha][u](x) = \overline{\UAPNetwork}[a][u](x)$: \begin{align}
    T_2 &= \left\| G_M[\alpha][u](x) - \overline{\UAPNetwork}[a][u](x) \right\|_{\Lp{2}_{\nu \times \mu \times \lambda}(W \times U \times \Omega_V)} \notag \\
    & \leq \left\| G_M[\alpha][u](x) - \overline{\UAPNetwork}[a][u](x) \right\|_{\Lp{2}_{\nu \times \mu \times \lambda}(W_K \times U \times \Omega_V)} \notag \\
    & \qquad+ \left\| G_M[\alpha][u](x) - \overline{\UAPNetwork}[a][u](x) \right\|_{\Lp{2}_{\nu \times \mu \times \lambda}((W\setminus W_K) \times U \times \Omega_V)} \notag \\
    & \leq \left\| G_M[\alpha][u](x) - \UAPNetwork[a][u](x) \right\|_{\Lp{2}_{\nu \times \mu \times \lambda}(W_K \times U \times \Omega_V)} \notag \\
    & \qquad+ \left\| \overline{\UAPNetwork}[a][u](x) - \UAPNetwork[a][u](x) \right\|_{\Lp{2}_{\nu \times \mu \times \lambda}(W_K \times U \times \Omega_V)} \notag \\
    &\qquad+ \left\| G_M[\alpha][u](x) - \overline{\UAPNetwork}[a][u](x) \right\|_{\Lp{2}_{\nu \times \mu \times \lambda}((W\setminus W_K) \times U \times \Omega_V)} \notag \\
    & =: T_3 + T_4 + T_5. \label{eq:universalApproximation:T2}
\end{align} 
By \eqref{eq:universalApproximation:approximationNetworkK}, we have that 
\begin{equation} \label{eq:universalApproximation:T3}
    T_3 \leq \sup_{\alpha \in W_K} \left\| G_M[\alpha][u](x) - \UAPNetwork[\alpha][u](x)\right\|_{\Lp{2}_{\mu \times \lambda}(U \times \Omega_V)} \nu(W)^{1/2} \leq  \frac{2\eps}{9}.
\end{equation}
For $T_4$, we start by computing the following: 
for $\alpha \in W_K$, 
\begin{align}
    \Vert \UAPNetwork[\alpha][u](x) \Vert_{\Lp{2}_{\mu \times \lambda}( U \times \Omega_V)} &= \l \sum_{k=1}^N \sum_{i=1}^{M} L_{ki}(\alpha)^2 \r^{1/2} \notag \\
    &\leq \Vert G_M[\alpha][u](x) - \UAPNetwork[\alpha][u](x) \Vert_{\Lp{2}_{\mu \times \lambda}( U \times \Omega_V)} + \Vert G_M[\alpha][u](x) \Vert_{\Lp{2}_{\mu \times \lambda}( U \times \Omega_V)} \notag \\
    &\leq \frac{2\eps}{9\nu(W)^{1/2}} + M \label{eq:universalApproximation:boundNetworkK1}
\end{align}
where we used \eqref{eq:universalApproximation:approximationNetworkK} for \eqref{eq:universalApproximation:boundNetworkK1}.
Then, we estimate: \begin{align}
    T_4 &\leq \sup_{\alpha \in W_k}  \left\| \overline{\UAPNetwork}[a][u](x) - \UAPNetwork[a][u](x) \right\|_{\Lp{2}_{ \mu \times \lambda}(U \times \Omega_V)} \nu(W_K)^{1/2} \notag \\
    &\leq \nu(W)^{1/2} \sup_{\alpha \in W_K} \Vert L(\alpha) - \gamma(L(\alpha)) \Vert_{\ell^2(\bbR^{N\cdot M})} \label{eq:universalApproximation:l2norm} \\
    &\leq \nu(W)^{1/2} \sup_{\Vert x \Vert_{\ell^2(\bbR^{N\cdot M})} \leq M + \frac{2\eps}{9\nu(W)^{1/2}} } \Vert x - \gamma(x) \Vert_{\ell^2(\bbR^{N\cdot M})} \label{eq:universalApproximation:gamma1} \\
    &\leq \nu(W)^{1/2} \eps \delta_3 = \frac{2\eps}{9} \label{eq:universalApproximation:T4}
\end{align}
where we used the same computation to obtain \eqref{eq:universalApproximation:orthonormalBeta} for \eqref{eq:universalApproximation:l2norm}, \eqref{eq:universalApproximation:boundNetworkK1} for \eqref{eq:universalApproximation:gamma1} and the definition of $\gamma$ for \eqref{eq:universalApproximation:T4}.

For $T_5$, we estimate as follows: \begin{align}
    T_5 &\leq \nu(W\setminus W_K)^{1/2} \l \sup_{\alpha \in W \setminus W_K } \Vert \overline{\UAPNetwork}[\alpha][u](x) \Vert_{\Lp{2}_{\mu \times \lambda}( U \times \Omega_V)} + \sup_{\alpha\in W\setminus W_K } \Vert G[\alpha][u](x) \Vert_{\Lp{2}_{\mu \times \lambda}( U \times \Omega_V)} \r \notag \\
    &\leq \nu(W\setminus W_K)^{1/2} \l  \Vert \gamma(L(\alpha)) \Vert_{\ell^2(\bbR^{N \cdot M})} + M \r \label{eq:universalApproximation:l2norm2} \\
    &\leq 3M\nu(W\setminus W_K)^{1/2} \label{eq:universalApproximation:gamma2} \\
    &\leq 3M\l \eps\delta_1 \r^{1/2} = \frac{2\eps}{9} \label{eq:universalApproximation:T5}
\end{align}
where we proceeded analogously to \eqref{eq:universalApproximation:orthonormalBeta} for \eqref{eq:universalApproximation:l2norm2} and used the definition of $\gamma$ for \eqref{eq:universalApproximation:gamma2}. 

Combining our estimates \eqref{eq:universalApproximation:T3}, \eqref{eq:universalApproximation:T4} and \eqref{eq:universalApproximation:T5}, by \eqref{eq:universalApproximation:T2}, we deduce that $T_2 \leq \frac{2}{3}\eps$. Using \eqref{eq:universalApproximation:estimateT1} and \eqref{eq:universalApproximationII:reductionBounded} allows us to conclude that \[
\left\| G[\alpha][u](x) - \overline{\UAPNetwork}[\alpha][u](x)\right\|_{\Lp{2}_{\nu \times \mu \times \lambda}(W \times U \times \Omega_V)} \leq \eps. \qedhere
\] 

\end{proof}

\subsection{Scaling laws proofs}
\subsubsection{Proof of Theorem \ref{thm:back:functionalApproximationLinfty}}
\label{sec_backfunctionalApproximationLinfty}
We first prove the functional approximation rate in $\Lp{\infty}$. The intuition behind the proof parallels our discussion in Section \ref{sec:back:scalingLaws}.

\begin{proof}[Proof of Theorem \ref{thm:back:functionalApproximationLinfty}]

Let $\delta > 0$ and $ \mathcal{C}_U = \{\mathcal{B}_{\delta}(c_m) \}_{ m  = 1}^{n_{c_U}}$ be a finite cover of $\Omega_U$ by $c_U$ Euclidean balls where $c_U$ can be further estimated by \cite[Corollary 2]{liu2024neuralscalinglawsdeep}.
By \cite[Lemma 1]{liu2024neuralscalinglawsdeep}, there exists a $\Ck{\infty}(\Omega_U) \subseteq \Ck{\infty}(\mathcal{C}_U)$ partition of unity 
$\{\omega_m(x):\Omega_U \mapsto \bbR\}_{m=1}^{n_{c_U}}$ subordinate to the cover $\mathcal{C}_U$. This allows us to consider a discrete-to-continuum lifting from $[-\beta_U, \beta_U]^{n_{c_U}}$ to $\Ck{\infty}(\Omega_U)$: we define the mapping $I_{\mathcal{C}_U}:[-\beta_U, \beta_U]^{n_{c_U}} \mapsto \Ck{\infty}(\Omega_U)$ by
\begin{align*}
        I_{\mathcal{C}_U}[z](x) = \sum_{m=1}^{n_{c_U}} [z]_m\omega_m(x)
        %\label{eqn_z_rho}
\end{align*}
for all $z \in [-\beta_U, \beta_U]^{n_{c_U}}$ and $ x \in \Omega_U$. 
Conversely, we can define a continuum-to-discrete projection $P_{\mathcal{C}_U}:\Ck{0}(\Omega_U)\mapsto [-\beta_U, \beta_U]^{n_{c_U}}$ by $P_{\mathcal{C}_U}(z) = (z(c_1),\dots,z(c_{n_{C_U}}))^\top$ for $z \in \Ck{0}(\Omega_U)$.

 We note the following point-wise error approximation for any $u\in U$ and $x \in \Omega_U$:
    \begin{align*}
        |u(x)-I_{\mathcal{C}_U}[P_{\mathcal{C}_U}(u)](x)|\leq &\sum_{m=1}^{n_{c_U}} |u(x) - u(c_m)| |\omega_m(x)|\\
        =&\sum_{m: \|x - c_m\|_2\leq \delta} |u(x)-u(c_m)||\omega_m(x)| \leq  L_U\delta
    \end{align*}
implying that $\Vert u -  I_{\mathcal{C}_U}[P_{\mathcal{C}_U}(u)] \Vert_{\Lp{\infty}} \leq \delta L_U$.
Setting $\delta=\frac{\varepsilon}{2L_fL_U}$ and using the Lipschitz property of $f$, continuing from the above, we obtain 
    \begin{align}
       |f(u)-f(I_{\mathcal{C}_U}[P_{\mathcal{C}_U}(u)])|\leq L_f\|u-I_{\mathcal{C}_U}[P_{\mathcal{C}_U}(u)]\|_{\Lp{\infty}(\Omega_U)}\leq L_f L_U\delta =\frac{\varepsilon}{2}\label{eq:functionalApproximation:LipschitzBound}
    \end{align}
   
Next, 
we define $\hat{f}: [-\beta_U, \beta_U]^{n_{c_U}}\rightarrow \mathbb{R}$ such that $\hat{f}(z) = f(I_{\mathcal{C}_U}[z]) = f\l \sum_{m=1}^{n_{c_U}} [z]_m \omega_m(x)\r$. 
We claim that $\hat{f}$ is a Lipschitz function on $[-\beta_U, \beta_U]^{n_{c_U}}$. Indeed, let $z_1,z_2 \in [-\beta_U, \beta_U]^{n_{c_U}}$ and  estimate as follows:
    \begin{align*}
       |\hat{f}(z_1) - \hat{f}(z_2)| = & |f(I_{\mathcal{C}_U}[z_1])-f(I_{\mathcal{C}_U}[z_2])|\\
       \leq &L_f\|I_{\mathcal{C}_U}[z_1] - I_{\mathcal{C}_U}[z_2]\|_{\Lp{\infty}(\Omega_U)}\\
       %=& L_f \sqrt{\int_{\Omega_U} (u_{\omega} - \bar{u}_{\omega})^2 d\xb}\\
    \leq & L_f \sup_{x \in \Omega_U}  \sum_{m=1}^{n_{c_U}} \left| \left([z_1]_m-[z_2]_m\right)\omega_m(x)\right|\\
&\leq L_f \sup_{x \in \Omega_U}  \sqrt{\sum_{m=1}^{n_{c_U}} \left([z_1]_m-[z_2]_m\right)^2 } \sqrt{ \sum_{m=1}^{n_{c_U}} \left(\omega_m(x)\right)^2}\\
&\leq L_f \| z_1-z_2 \|_{\ell^2(\bbR^{n_{c_U}})} \sup_{x \in \Omega_U} \sqrt{ \sum_{m=1}^{n_{c_U}} \omega_m(x) \dd x}\\
&= L_f\| z_1-z_2 \|_{\ell^2(\bbR^{n_{c_U}})}.
    \end{align*}
Since $\hat{f}$ is Lipschitz continuous on the compact set $[-\beta_U, \beta_U]^{n_{c_U}}$, it is bounded by some constant $C_{\hat{f}}$ and 
we can deduce that $\hat{f} \in V(n_{c_U},\beta_U,L_f,C_{\hat{f}})$ for some set of functions $V$ (see Assumption \ref{assumption:Main:assumptions:S4}).

Consequently, we apply the function approximation Theorem \ref{thm:back:functionApproximation}. Specifically, for any  $\varepsilon_0>0$, there exists a constant $C$ depending on $\beta_U$ and $L_f$ such that the following holds. There exists 
\begin{itemize}
    \item a network class $\cF_{\rm NN}(n_{c_U},1,L,p,K,\kappa,R)$ whose parameters scale as \begin{align*}
    &L=\mathcal{O}\left(n_{c_U}^2\log(n_{c_U})+n_{c_U}^2\log(\varepsilon_0^{-1})\right),\quad  p = \mathcal{O}(1),\quad K = \mathcal{O}\left(n_{c_U}^2\log n_{c_U}+n_{c_U}^2\log(\varepsilon_0^{-1})\right), \\ &\kappa=\mathcal{O}(n_{c_U}^{n_{c_U}/2+1}\varepsilon_0^{-n_{c_U}-1}),\qquad \, R=1
    \end{align*}
    where the constants hidden in $\mathcal{O}$ depend on $\beta_U$ and $L_f$,
    \item networks $\{b_k\}_{k=1}^{H^{n_{c_U}}} \subset \cF_{\rm NN}(n_{c_U},1,L,p,K,\kappa,R)$ with $H := C \sqrt{n_{c_U}} \eps_0^{-1}$ and
    \item points $\{s_k\}_{k=1}^{H^{n_{c_U}}} \subset [-\beta_U, \beta_U]^{n_{c_U}}$
\end{itemize}
such that 
    \begin{align} \notag
        \sup_{z\in [-\beta_U, \beta_U]^{n_{c_U}}}\left|\hat{f}(z)-\sum_{k=1}^{H^{n_{c_U}}} \hat{f}(s_k) b_k(z)\right|\leq \varepsilon_0. %\label{eq:functionalApproximation:approximation} 
    \end{align}
We note that $P_{\mathcal{C}_U}(U) \subset [-\beta_U, \beta_U]^{n_{c_U}}$ by the fact that $U$ satisfies $U(d_U,\gamma_U,L_U,\beta_U)$ and hence, setting $\varepsilon_0 = \eps/2$, this does not change the network class scalings,
    \begin{align}
        \sup_{u\in U}\left|\hat{f}(P_{\mathcal{C}_U}[u])-\sum_{k=1}^{H^{n_{c_U}}} \hat{f}(s_k) b_k(P_{\mathcal{C}_U}[u])\right|\leq \frac{\varepsilon}{2}. \label{eq:functionalApproximation:approximation} 
    \end{align}
We conclude as follows: using \eqref{eq:functionalApproximation:LipschitzBound} and \eqref{eq:functionalApproximation:approximation}, for any $u  \in U$, we have
    \begin{align*}
        \sup_{u \in U} &\left|f(u)-\sum_{k=1}^{H^{n_{c_U}}} \hat{f}(s_k) b_k(P_{\mathcal{C}_U}(u)) \right| \leq \sup_{u\in U} \left|f(u)-\hat{f}(P_{\mathcal{C}_U}(u)) \right|\\
        &\qquad \qquad + \sup_{u \in U}\left| \hat{f}(P_{\mathcal{C}_U}(u))-\sum_{k=1}^{H^{n_{c_U}}} \hat{f}(s_k) b_k(P_{\mathcal{C}_U}(u))\right|\\
        &\qquad \qquad  = \sup_{u\in U} \left|f(u)-f(I_{\mathcal{C}_U}[P_{\mathcal{C}_U}(u)]) \right| + \sup_{u \in U}\left| \hat{f}(P_{\mathcal{C}_U}[u])-\sum_{k=1}^{H^{n_{c_U}}} \hat{f}(s_k) b_k(P_{\mathcal{C}_U}[u])\right|\\
         & \qquad \qquad \leq \frac{\varepsilon}{2}+ \frac{\varepsilon}{2}=\varepsilon.
    \end{align*}
Recalling that $ \hat{f}(s_k) = f(I_{\mathcal{C}_U}[s_k])$, we set $u_k = I_{\mathcal{C}_U}[s_k]$ and obtain the claim of the theorem.
\end{proof}

\subsubsection{Proof of Theorem \ref{thm:main:multipleOperatorApproximation}}
\label{sec_main_multipleOperatorApproximation}
In this section, we derive the convergence rate in the multiple operator setting. In particular, the proof is an application of Theorems \ref{thm:back:functionalApproximationLinfty} and \ref{thm:back:operatorApproximation}. 

\begin{proof}[Proof of Theorem \ref{thm:main:multipleOperatorApproximation}]

For $u \in U$ and $x \in \Omega_V$, define the functional $f_{u,x}: \{ \alpha:\Omega_W \mapsto \bbR \spaceBar \Vert \alpha \Vert_{\Lp{\infty}} \leq \beta_W \} \mapsto \bbR$ as \[
    f_{u,x}(\alpha) = G[\alpha][u](x). 
    \]
    In particular, we have that \begin{align}
        \vert f_{u,x}(\alpha_1) - f_{u,x}(\alpha_2) \vert &= \vert G[\alpha_1][u](x) - G[\alpha_2][u](x) \vert \notag \\
        &\leq L_G \Vert \alpha_1 - \alpha_2 \Vert_{\Lp{r_G}(\Omega_W)} \label{eq:multipleOperatorApproximation:Lipschitz1} \\
        &\leq L_G \vert \Omega_W \vert^{1/(r_G)} \Vert \alpha_1 - \alpha_2 \Vert_{\Lp{\infty}(\Omega_U)}. \notag 
    \end{align}
where we use \eqref{eq:multipleOperatorApproximation:Lipschitz} for \eqref{eq:multipleOperatorApproximation:Lipschitz1}.
Therefore, we can apply Theorem \ref{thm:back:functionalApproximationLinfty}. Specifically, for any  $\varepsilon_0>0$, there exists constants $C''$ and $C_{\zeta}$ depending on $\beta_W, L_G \vert \Omega_W \vert^{1/r_G}$ and $L_G \vert \Omega_W \vert^{1/r_G},L_W$ respectively such that the following holds. 
There exists 
\begin{itemize}
    \item a constant $\zeta:=C_{\zeta}\varepsilon$ and points $\{y_m\}_{m=1}^{n_{c_W}}\subset \Omega_W$ so that $\{\mathcal{B}_{\zeta}(y_m) \}_{ m = 1}^{n_{c_W}}$ is a cover of $\Omega_W$ for some $n_{c_W}$,
    \item a network class $\cF_3 = \cF_{\rm NN}(n_{c_W},1,L_3,p_3,K_3,\kappa_3,R_3)$ whose parameters scale as \begin{align*}
    &L=\mathcal{O}\left(n_{c_W}^2+n_{c_W}\log(\varepsilon_0^{-1})\right),\quad  p = \mathcal{O}(1),\quad K = \mathcal{O}\left(n_{c_W}^2\log n_{c_W}+n_{c_W}^2\log(\varepsilon_0^{-1})\right), \\ &\kappa=\mathcal{O}(n_{c_W}^{n_{c_W}/2+1}\varepsilon_0^{-n_{c_W}-1}),\qquad \, R=1
    \end{align*}
    where the constants hidden in $\mathcal{O}$ depend on $\beta_W$ and $L_G \vert \Omega_W \vert^{1/r_G}$,
    \item networks $\{l_p\}_{p=1}^{P^{n_{c_W}}} \subset \cF_3$ with $P := C'' \sqrt{n_{c_W}} \eps_0^{-1}$ and
    \item functions $\{\alpha_p\}_{p=1}^{P^{n_{c_W}}} \subset \{ \alpha:\Omega_W \mapsto \bbR \spaceBar \Vert \alpha \Vert_{\Lp{\infty}} \leq \beta_W \}$
\end{itemize}
such that 
    \begin{align}
\sup_{\alpha \in W} \left| f_{u,x}(\alpha) - \sum_{p=1}^{P^{n_{c_W}}} f_{u,x}(\alpha_k)l_p(P_{\mathcal{C}_W}(\alpha)) \right| = \sup_{\alpha \in U} \left| G[\alpha][u](x) - \sum_{p=1}^{P^{n_{c_W}}} G[\alpha_p][u](x) l_p(P_{\mathcal{C}_W}(\alpha)) \right| &\leq \eps_0 \notag %\label{eq:multipleOperatorApproximation:functionalApproximation}
\end{align}
where $P_{\mathcal{C}_W}(\alpha)$ is defined in the proof of Theorem \ref{thm:back:functionalApproximationLinfty}.

By assumption, $G[\alpha_p] \in \mathcal{G}$ for all $1 \leq p \leq P^{n_{c_W}}$. This corresponds to the situation in \eqref{eq:rem:multipleOperatorSameDomainSameRange} where $G^{(p)} = G[\alpha_p]$, i.e. to the problem of approximating $P^{n_{c_W}}$ single operators. 
For $1 \leq p \leq P^{n_{c_W}}$, we therefore apply Theorem \ref{thm:back:operatorApproximation} 
%in combination with Remark \ref{rem:alternative}
and the rest of the argument to obtain \eqref{eq:main:multipleOperatorApproximation} is analogous to the proof of Theorem \ref{thm:back:operatorApproximation} in Remarks \ref{rem:uniformOperatorApproximation} and \ref{rem:balancingComplexity} (with $\eps_0 = \eps/2$ and $\eps_1 = \eps/(2P^{n_{c_W}})$).
\end{proof}

\section{Numerical Experiments}\label{sec:experiments}

\begin{figure*}[t]
    \centering
    \includegraphics[width=0.8\linewidth]{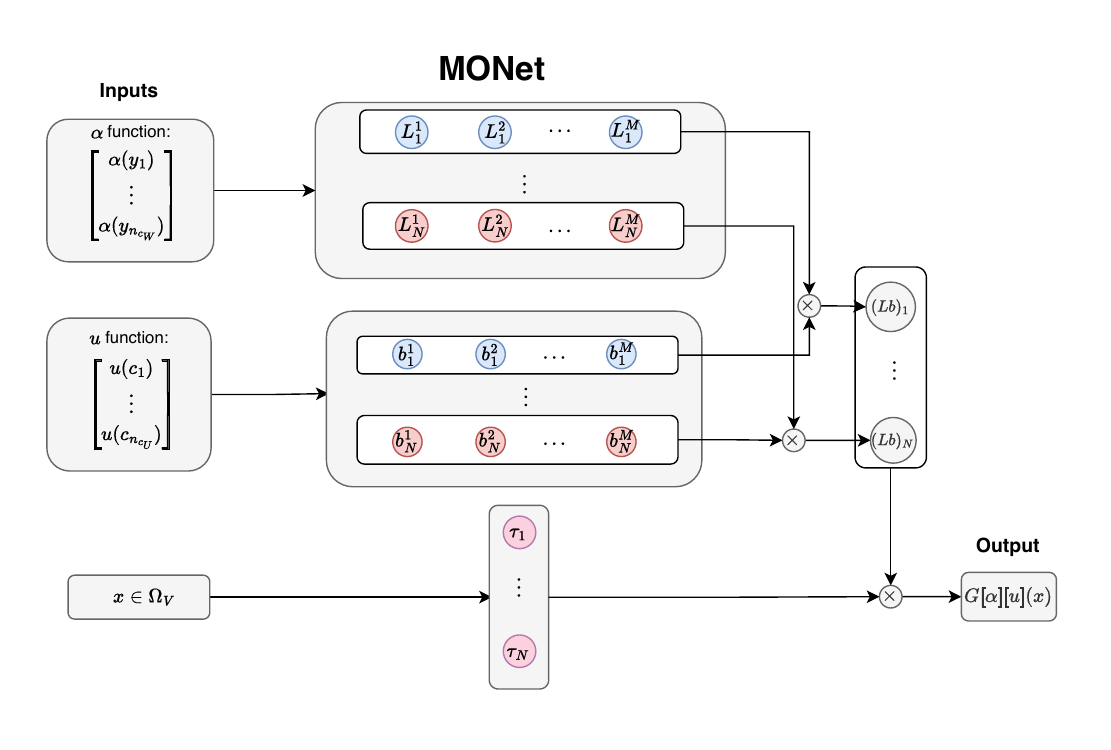}
\caption{\textbf{MONet architecture}: The $\alpha$ function  is the input for the parameter-approximation network. The $u$ function is the input for the function-approximation network. The spatial values $x\in\Omega_V$ are the input for the space-approximation network. 
} 
    \label{fig:monet}
\end{figure*}

\begin{figure*}[t]
    \centering
    \includegraphics[width=0.8\linewidth]{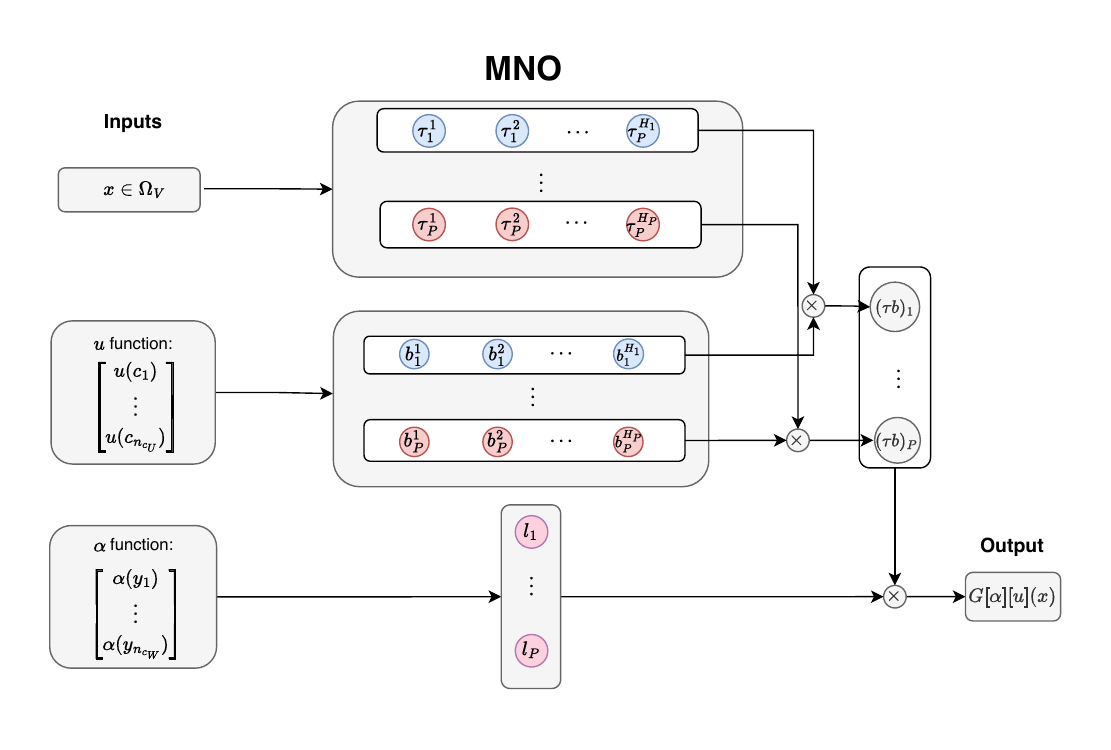}
\caption{\textbf{MNO architecture}: The $\alpha$ function  is the input for the parameter-approximation network. The $u$ function is the input for the function-approximation network. The spatial values $x\in\Omega_V$ are the input for the space-approximation network. }
    \label{fig:mno}
\end{figure*}

In this section, we refer to $\tau$ as the space-approximation network, $b$ as the function-approximation network, and $L$ or $l$ as the parameter-approximation network. 
To evaluate the versatility and effectiveness of $\ScalingNetwork$ and $\UAPNetwork$, we test both architectures (see Figures \ref{fig:mno} and \ref{fig:monet}) on five representative parametric PDEs, spanning settings in which the parameter~$\alpha$ is modeled either as a function or as a finite-dimensional vector in~$\mathbb{R}^{p}$. In all experiments, the objective is to predict the PDE solution at points $x = (t,x_{\text{spatial}}) \in (0,2] \times [0,2]$, given the parametric function~$\alpha$ and the initial condition~$u_0$.
 
We construct  50 initial conditions for each PDE following the sinusoidal formulation proposed in~\cite{takamoto2022pdebench}:
\begin{equation} \label{sineIC}
    u_0(x) = \sum_{i=1}^{4} A_i \sin(k_i x + \phi_i),
\end{equation}
where $k_i = \pi n_i$ and $n_i$ are uniformly sampled integers in $[1, 4]$.
The amplitudes $A_i$ are sampled uniformly from $[0,1]$, 
and $\phi_i$ are random phases drawn from $(0, 2\pi)$. Following the setup in~\cite{takamoto2022pdebench,sun2025foundation}, 
after computing~\eqref{sineIC}, each initial condition undergoes random post-processing: 
with 10\% probability, the absolute value of $u_0$ is taken, with 50\% probability, 
its sign is flipped (i.e., multiplied by $-1$), and with 10\% probability, it is multiplied by the indicator function of a randomly chosen smooth subdomain of~$[0,2]$.

The resulting initial conditions are then sampled at points 
$\{c_i\}_{i=1}^{n_{c_U}}$ over the domain~$[0,2]$ and provided as inputs 
to the function-approximation networks. In particular, we use:
 $n_{c_U} = 64$ and cell-center points (i.e. midpoints of uniform grid cells) $\{c_i\}_{i=1}^{64}$ .
The space-approximation networks take as input the spatiotemporal coordinates \((t,x) \in \mathbb{R}^2\).

All models are trained with mean squared loss (MSE). We evaluate them on a $32 \times 64$ spacetime grid over $[0,2] \times [0,2]$, 
and report the average relative $\Lp{2}$ error across all test cases:
\[
\frac{1}{N_{\text{test}}}\sum_{i=1}^{N_{\text{test}}}
\frac{\|u_{\text{pred}}^{(i)} - u_{\text{target}}^{(i)}\|_2}
{\|u_{\text{target}}^{(i)}\|_2 + \varepsilon},
\]
where $(u_{\text{pred}}^{(i)}, u_{\text{target}}^{(i)})$ corresponds to the $i$-th pair of predicted and reference solution functions, each evaluated at all points of the discretized spacetime grid, $\varepsilon = 10^{-5}$ and $N_{\text{test}} = 80 \times 50$, corresponding to 80 distinct parameter samples~$\alpha$ (i.e., 80 distinct parametrized PDEs), each evaluated with 50 different initial conditions~$u_0$.

We compare our models $\ScalingNetwork$ and $\UAPNetwork$, with DeepONet and MIONet with different configurations, as detailed in Table~\ref{tab:config}. We employ two network configurations for MNO: MNO-S (small) which uses 1.19M parameters and MNO-L (large) which uses 16.7M. This is computational feasible since MNO's tensor structure is more amenable to larger model complexity.  In ``DeepONet-C", we simply concatenate the $\alpha$ and $u$ inputs together and put them into a single function-approximation layer. This is also a theoretically valid approach to training multiple operators; however, as shown in the experiments, does not preform as well as $\ScalingNetwork$ and $\UAPNetwork$. Note that in the experiments, the training time for MNO-S, MONet, and DeepONet are comparable, while the training times for MNO-L and MIONet are larger as expected. Additional details on the experimental setup, including network hyperparameters and training times, are provided in Appendix~\ref{MOLEdetail}.

\begin{table}[H]
\centering
\small
\renewcommand{\arraystretch}{1.4}
\setlength{\tabcolsep}{2pt}
\begin{tabularx}{\linewidth}{>{\bfseries}p{2.2cm} *{6}{Y}}
\toprule
& DeepONet & DeepONet-C & MIONet & $\UAPNetwork$ & $\ScalingNetwork$-S & $\ScalingNetwork$-L \\
\midrule
\textbf{Number of parameters} 
& 1.47M 
& 1.47M 
& 1.50M 
& 1.15M 
& 1.19M 
& 16.7M \\
\midrule
\textbf{\# of $L$ or $l$ (P)} 
& \multicolumn{3}{c|}{\textbf{N/A}} 
& $1$
& 10 
& 40 \\
\textbf{Depth of $L$ or $l$ } 
& \multicolumn{3}{c|}{\textbf{N/A}} 
& 4 
& 4 
& 4 \\
\midrule
\textbf{\# of $b$ (M/H)} 
& 100
& 100
& 75 
& 100
& 20 
& 100 \\
\textbf{Depth of $b$} 
& 4 
& 4 
& 4 
& 4 
& 4 
& 4 \\
\midrule
\textbf{\# of $\tau$ (N)} 
& 20 
& 20 
& 75
& 20 
& 20
& 20 \\
\textbf{Depth of $\tau$} 
& 6 
& 6 
& 6 
& 6 
& 6 
& 6 \\
\bottomrule
\end{tabularx}
\caption{Model configurations and architectural details for all tested variants. 
Abbreviations: \textbf{S} refers to the small version; \textbf{L} to the large version; and the symbol \# denotes the number of corresponding elements. For reference, the number of parameter-approximation networks ($L$ or $l$) corresponds to P in Equations~\eqref{eq:mole1} and~\eqref{eq:main:multipleOperatorApproximation};
the number of function-approximation networks ($b$) corresponds to M in Equation~\eqref{eq:mole1} and H in Equation~\eqref{eq:main:multipleOperatorApproximation};
and the number of space-approximation networks ($\tau$) corresponds to N in both Equation~\eqref{eq:mole1} and~\eqref{eq:main:multipleOperatorApproximation}. For simplicity, all powers of $P$, $M$, $H$, and $N$ are omitted.}

\label{tab:config}
\end{table}

\subsection{Conservation Laws} %\label{subsec:conservation}

We consider the following one-dimensional conservation law with periodic boundary conditions:
\begin{align*}
u_t + (\alpha_1 u + \alpha_2 u^2 + \alpha_3 u^3)_x &= \alpha_4 u_{xx}, \quad (t,x) \in [0,2] \times [0,2], \\
u(0,x) &= u_0(x), \\
u(t,0) &= u(t,2),
\end{align*}
where the parameter vector $\boldsymbol{\alpha} = [\alpha_1, \alpha_2, \alpha_3, \alpha_4]^\top$ is encoded within the parameter-approximation layers of both $\ScalingNetwork$ and $\UAPNetwork$.
The components of $\boldsymbol{\alpha}$ are sampled from the ranges $\alpha_i \in [0.9\alpha_i^c, 1.1 \alpha_i^c]$, with the reference values given by $\alpha^c = [1, 1, 1, 0.1]^\top$.

\begin{table}[htbp]
    \centering
    \caption{Performance comparison on conservation laws. In in-distribution (IN) experiments, we set $\alpha_i\in [0.9 \alpha_i^c, 1.1 \alpha_i^c]$, and in out-of-distribution (OOD) experiments, we set  $\alpha_i\in [0.8 \alpha_i^c, 1.2 \alpha_i^c]$
    }%%output7
    \label{tab:conresult}
    \begin{tabular}{c|c|c}
    \multirow{2}{*}{\textbf{Model}} & \multicolumn{2}{c}{\textbf{Relative $\Lp{2}$ Error}} \\ \cline{2-3}
    & IN & OOD \\\hline
    DeepONet & 6.59\% & 9.00\% \\
    DeepONet-C & 5.36\% & 6.59\% \\
    MIONet & 5.65\% & 8.48\% \\
    $\UAPNetwork$ & 5.67\% & 7.20\% \\
    $\ScalingNetwork$-S & 4.49\% & 6.64\% \\
    $\ScalingNetwork$-L & \textbf{3.84\%} & \textbf{5.92\%} \\
\end{tabular}
\end{table}

In this experiment, the family of operators emit solutions with similar (viscous) shock or rarefraction profiles, mainly differing in speeds. The space of potential solutions likely lie on a lower dimensional structure which shares commonalities between each randomly sampled PDE (i.e. each randomized flux). Thus we expect that the empirical rates and scalings are faster than the general rates proven in the previous sections.  We observe that $\ScalingNetwork$ and $\UAPNetwork$ outperform DeepONet-type architectures with comparable parameter counts, demonstrating the efficacy of its ${\alpha}$-encoding strategy over simple concatenation in the function-approximation networks. Our smaller networks produce in-distribution and out-of-distribution errors which are lower than the standard and concatenated DeepONet (see Table~\ref{tab:conresult}). Figure~\ref{fig:consoutput} shows that the (local) errors for DeepONet and MIONet are more concentrated in the shock formation and dynamics, while $\ScalingNetwork$ and $\UAPNetwork$ demonstrate a more even error distribution with relatively less error around regions of large gradients.

\begin{figure*}[t]
    \centering
    \includegraphics[width=0.8\linewidth]{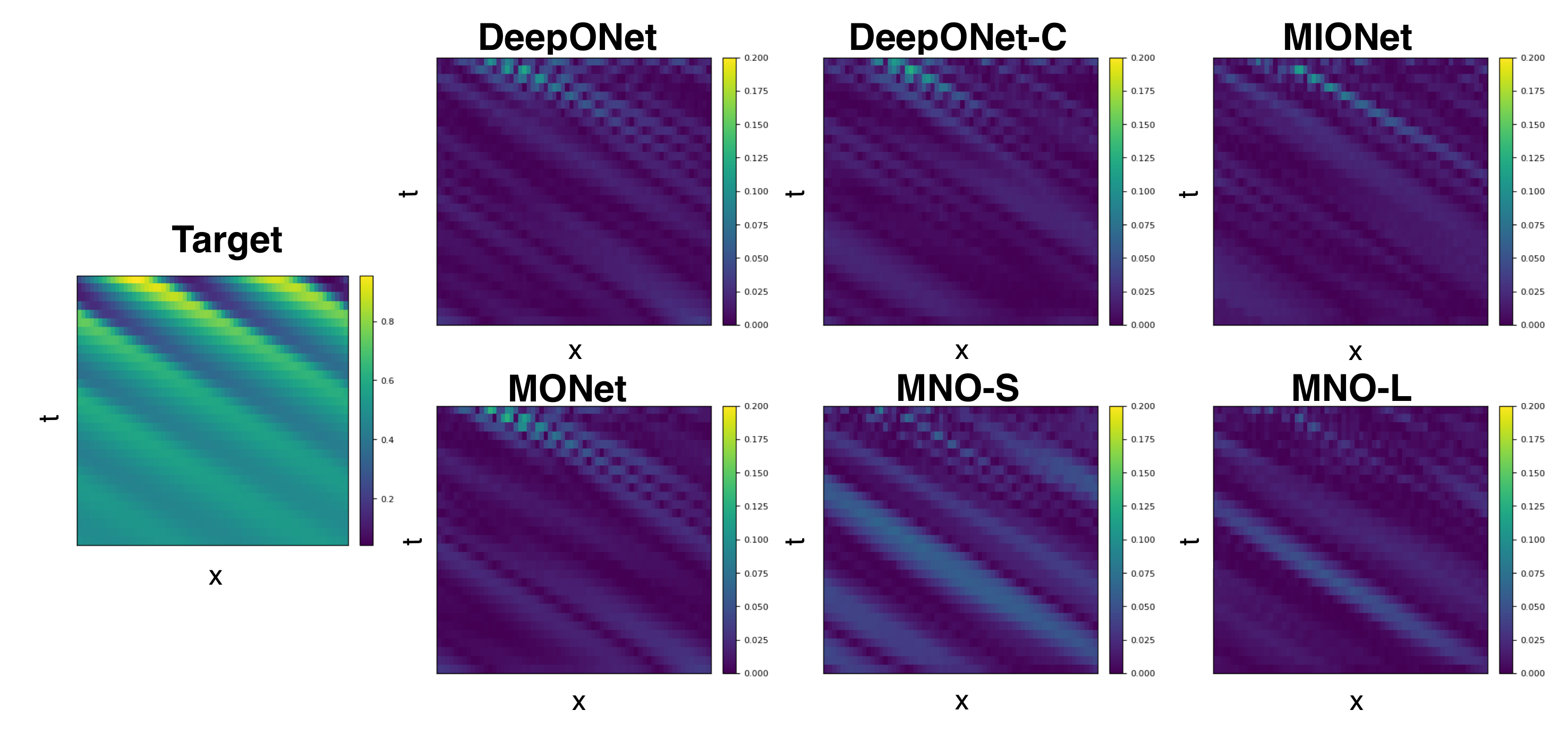}
    \caption{\textbf{Representative solution for conservation laws}:  The target solution (left) and error maps for  DeepONet,  DeepONet-C,  MIONet, $\UAPNetwork$, $\ScalingNetwork$-S and $\ScalingNetwork$-L. The instance-specific relative errors are 5.49\%, 4.82\%,  2.92\%, 5.11\%, 2.52\% and 1.81\%, respectively, aligning with the trends observed in  Table~\ref{tab:conresult}.}
    \label{fig:consoutput}
\end{figure*}

\subsection{Diffusion-Reaction-Advection Equation} %\label{subsec:diffusion}

We consider the following one-dimensional diffusion–reaction–advection equation:
\begin{align*}
u_t &= \alpha_1 u_{xx} + \alpha_2 u_x + \alpha_3 u^{\alpha_4}(1 - u^{\alpha_5}),
\quad (t,x) \in [0,2] \times [0,2], \\
u(0,x) &= u_0(x), \\
u(t,0) &= u(t,2),
\end{align*}
where the parameter vector $\boldsymbol{\alpha} = [\alpha_1, \alpha_2, \alpha_3, \alpha_4, \alpha_5]^\top$ is encoded within the parameter-approximation layers of both $\ScalingNetwork$ and $\UAPNetwork$.
The first three components are sampled from the ranges $\alpha_i \in [0.9\alpha_i^c, 1.1\alpha_i^c]$ for $1 \leq i \leq 3$, with reference values $\alpha^c = [0.01, 1, 1]^\top$, while $\alpha_4$ and $\alpha_5$ are drawn uniformly from $[1,3]$.

\begin{table}[b]
    \centering
    \caption{Performance comparison on diffusion-reaction-advection equation. For $1 \leq i \leq 3$, in in-distribution (IN) experiments, $\alpha_i\in [0.9\alpha_i^c, 1.1\alpha_i^c]$ whereas in out-of-distribution (OOD) experiments,  we set $\alpha_i\in [0.8\alpha_i^c, 1.2\alpha_i^c]$.}
    \label{tab:draresult}
    \begin{tabular}{c|c|c}
    \multirow{2}{*}{\textbf{Model}} & \multicolumn{2}{c}{\textbf{Relative $\Lp{2}$ Error}} \\ \cline{2-3}
    & IN & OOD \\\hline
    DeepONet & 13.63\% & 15.10\% \\
    DeepONet-C & 4.91\% & 7.07\% \\
    MIONet & 3.95\% & 7.06\% \\
    $\UAPNetwork$ & 3.80\% & 6.21\% \\
    $\ScalingNetwork$-S & 3.39\% & 5.47\% \\
    $\ScalingNetwork$-L & \textbf{2.51}\% & \textbf{4.27\%} \\
\end{tabular}
\end{table}

In this experiment, the parameters change the governing equation in a nonlinear fashion. This may be the cause for the larger error observed in the DeepONet models. From Table~\ref{tab:draresult}, we observe that $\ScalingNetwork$ and $\UAPNetwork$ variants produce more accurate solutions even with comparable parameter counts. The $\ScalingNetwork$ and $\UAPNetwork$ produce better in- and out-of-distribution predictions. The structured encoding in $\ScalingNetwork$ ensures more effective parameter sharing, which could be contributing to the lower error rates.  Figure~\ref{fig:draoutput} illustrates this performance difference: notably, $\ScalingNetwork$ is able to substantially reduce errors in regions that consistently show elevated error across all other methods. This shows an intrinsic difference between the underlying features learned by the family of models. 

\begin{figure*}[htbp]
    \centering
    \includegraphics[width=0.8\linewidth]{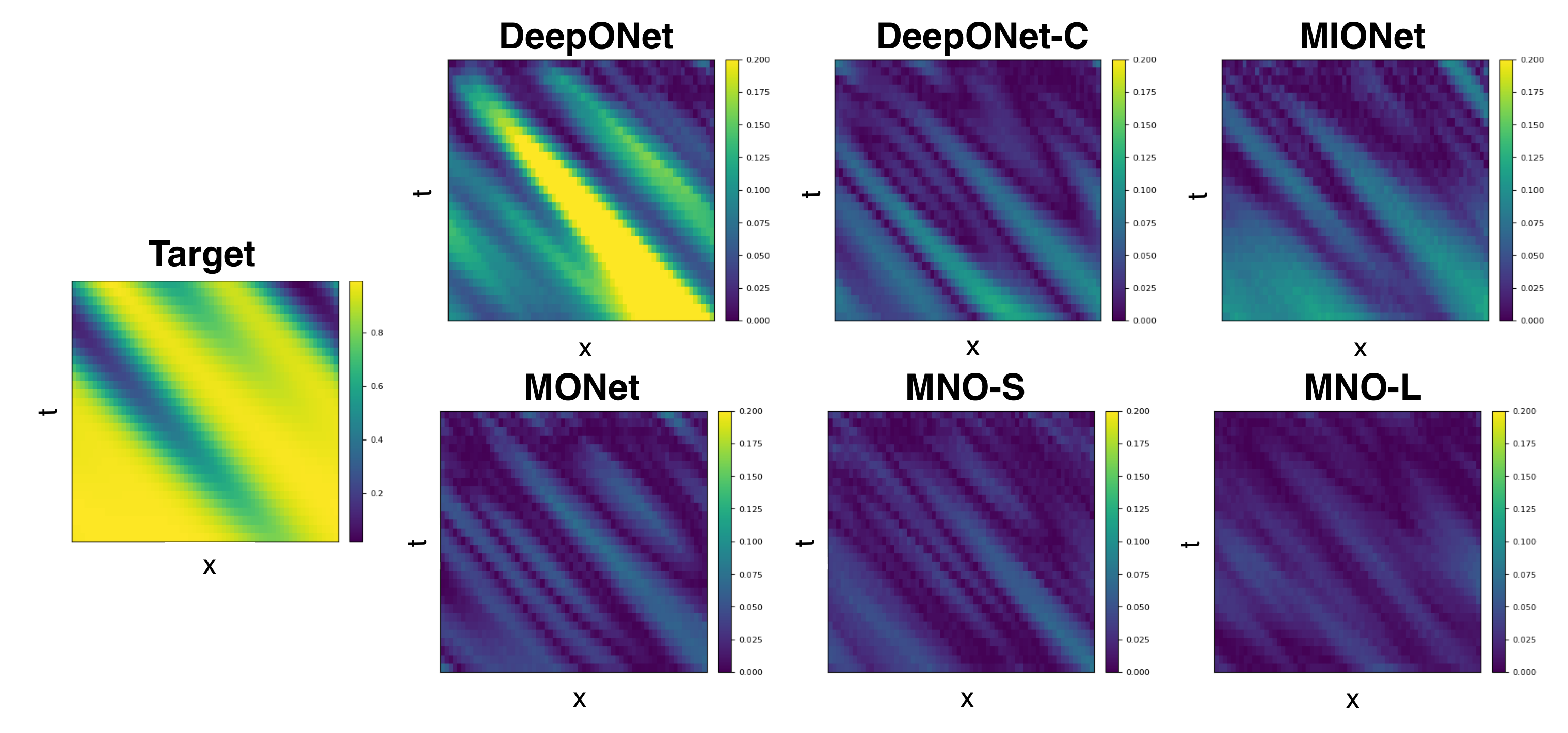}%%output4
    \caption{\textbf{Representative solution for diffusion-reaction-advection equation}: The target solution (left) and error maps for DeepONet,  DeepONet-C,  MIONet, $\UAPNetwork$, $\ScalingNetwork$, and $\ScalingNetwork$-L.  The instance-specific relative errors are  15.02\%, 5.26\% , 6.27\% , 5.26\%, 3.08\% and 2.38\%, respectively, aligning with the trends observed in Table~\ref{tab:draresult}.}
    \label{fig:draoutput}
\end{figure*}

\subsection{Nonlinear Klein-Gordon Equation} %\label{subsec:kleinGordon}

We consider the following nonlinear Klein–Gordon equation:
\begin{align*}
u_{tt} &= \alpha_1^2 u_{xx} - \alpha_2^2 \alpha_1^4 u - \alpha_3 u^3,
\quad (t,x) \in [0,2] \times [0,2], \\
u(0,x) &= u_0(x), \\
u_t(0,x) &= 0, \\
u(t,0) &= u(t,2).
\end{align*}
The parameter vector $\boldsymbol{\alpha} = [\alpha_1, \alpha_2, \alpha_3]^\top$ is encoded within the parameter-approximation layers of both $\ScalingNetwork$ and $\UAPNetwork$.
The components of $\boldsymbol{\alpha}$ are sampled from the ranges $\alpha_i \in [0.9,\alpha_i^c,,1.1,\alpha_i^c]$ with reference values $\alpha^c = [1,1,1]^\top$.

\begin{table}[htbp]
    \centering
    \caption{Performance comparison on the nonlinear Klein-Gordon equation. In in-distribution (IN) experiments, $\alpha_i\in [0.9\alpha_i^c, 1.1\alpha_i^c]$ whereas in out-of-distribution (OOD) experiments,  we set $\alpha_i\in [0.85\alpha_i^c, 1.15\alpha_i^c]$ for $i\in[1,2,3]$}
    \label{tab:nkgresult}%%output5
     \begin{tabular}{c|c|c}
    \multirow{2}{*}{\textbf{Model}} & \multicolumn{2}{c}{\textbf{Relative $\Lp{2}$ Error}} \\ \cline{2-3}
    & IN & OOD \\\hline
    DeepONet & 24.03\% & 33.82\% \\
    DeepONet-C & 5.67\% & 7.90\% \\
    MIONet & 7.73\% & 13.78\% \\
    $\UAPNetwork$ & 4.53\% & 7.87\% \\
    $\ScalingNetwork$-S & 3.56\% & 7.30\% \\
    $\ScalingNetwork$-L & \textbf{2.50}\% & \textbf{5.90\%} \\
\end{tabular}
\end{table}

In this experiment, the governing equation is a second-order hyperbolic PDE and thus produces wave-like solutions. Notably, Table~\ref{tab:nkgresult} shows that DeepONet-C achieves lower relative errors than MIONet on this task, while $\ScalingNetwork$ further improves performance, yielding substantially smaller errors overall. Figure~\ref{fig:nkgoutput} shows that most models' errors have coarse and low-frequency patterns appear while $\ScalingNetwork$ does not. Additionally, as the parameter counts increase, the error associated with $\ScalingNetwork$ decreases locally as well.

\begin{figure*}[htbp]
    \centering
    \includegraphics[width=0.8\linewidth]{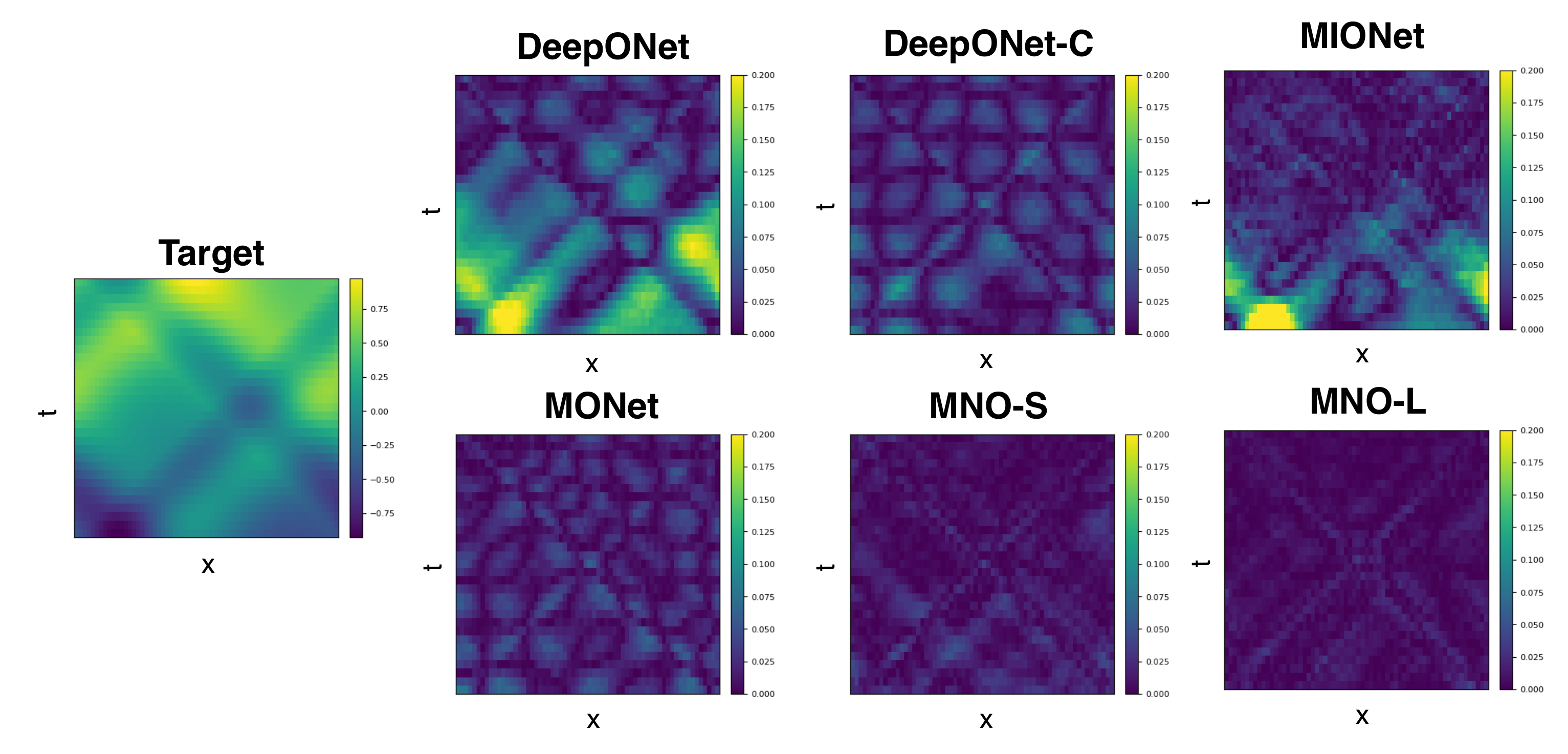}
    \caption{\textbf{Representative solution for the nonlinear Klein-Gordon equation}: The target solution (left) and error maps for DeepONet,  DeepONet-C,  MIONet, $\UAPNetwork$, $\ScalingNetwork$-S, and $\ScalingNetwork$-L.  The instance-specific relative errors are 19.84\%, 7.99\%,  14.26\%, 6.06\%, 3.76\% and 2.34\%,  respectively, aligning with the trends observed in Table~\ref{tab:nkgresult}.}
    \label{fig:nkgoutput}
\end{figure*}

\subsection{Parametric Diffusion-Reaction Equation} %\label{subsec:reaction}

We consider the following parametric diffusion–reaction equation:
\begin{align*}
u_t &= (\alpha(x)u_{x})_x + u(1 - u), \quad (t,x) \in [0,2] \times [0,2], \\
u(0,x) &= u_0(x), \\
u(t,0) &= u(t,2),
\end{align*}
where the spatially varying diffusivity  $\alpha(x) $ %is given by $\alpha(x) = 0.01f(x)$, and $f(x)$ is 
is sampled from a Gaussian random process with variance $0.01^2$. 
The parametric function $\alpha(x)$ is evaluated at 129 sensor locations corresponding to the boundaries of uniformly spaced cells, $\{x_i^b\}_{i=1}^{129}$, and the resulting values ${\alpha(x_i^b)}$ are encoded within the parameter-approximation networks of $\ScalingNetwork$ and $\UAPNetwork$.

\begin{table}[b]
    \centering
    
   \caption{Performance comparison on the parametric diffusion-reaction equation (in-distribution).}
    \label{tab:pdrresult}
   \begin{tabular}{l|c}
    \textbf{Model} & \textbf{Relative $\Lp{2}$ Error} \\ \hline
    DeepONet & 9.68\% \\
    DeepONet-C & 6.59\% \\
    MIONet & 5.65\% \\
    $\UAPNetwork$ & 5.77\% \\
    $\ScalingNetwork$-S & 4.62\% \\
    $\ScalingNetwork$-L & \textbf{3.34\%} \\
\end{tabular}
\end{table}

This problem is more challenging since the parametric inputs are spatial dependent and are differentiated within the diffusion term. As shown in Table~\ref{tab:pdrresult}, $\ScalingNetwork$-L achieves the highest accuracy among all tested models.
This improvement stems from the structured parameter encoding introduced by the parameter-approximation layers, yielding substantially better performance than simply concatenating $\alpha$ with the function-approximation inputs.
Figure~\ref{fig:pdroutput} further illustrates that, although all models exhibit localized error regions near the bottom right of the domain, $\ScalingNetwork$ markedly reduces this region and yields significantly lower local errors.

 \begin{figure*}[htbp]
    \centering
    \includegraphics[width=0.8\linewidth]{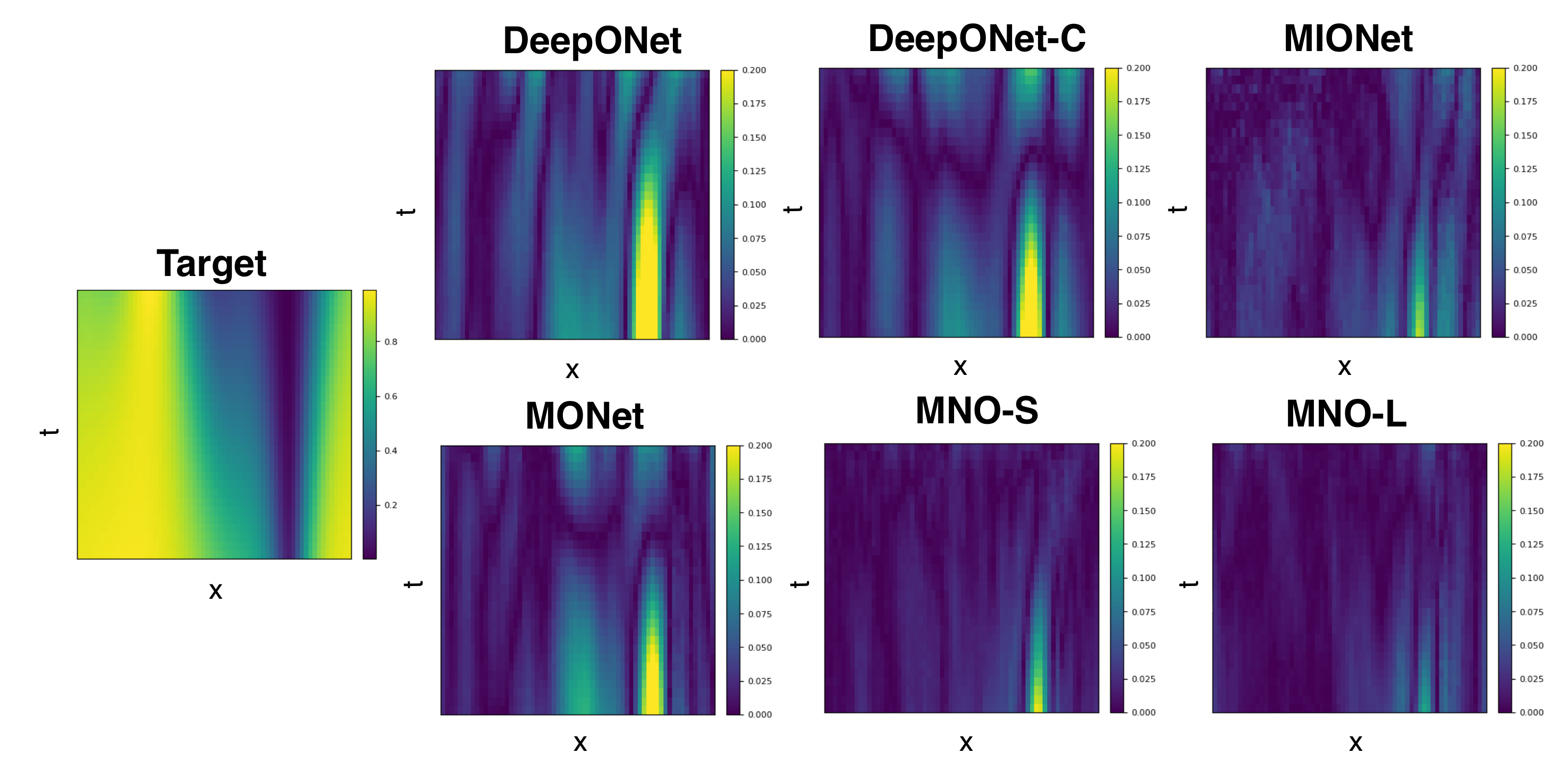}%%output8
 \caption{\textbf{Representative solution for the parametric diffusion-reaction equation}: The target solution (left) and error maps for  DeepONet,  DeepONet-C,  MIONet, $\UAPNetwork$,$\ScalingNetwork$-S and $\ScalingNetwork$-L. The instance-specific relative errors are 9.66\%, 7.52\%,  4.60\%, 6.84\%, 3.48\% and 2.83\%,  respectively, aligning with the trends observed in Table~\ref{tab:pdrresult}.}
   
    \label{fig:pdroutput}
\end{figure*}

\subsection{Parametric wave Equation} %\label{subsec:wave}

We consider the following parametric wave equation:
\begin{align*}
u_{tt} &= \alpha^2(t) u_{xx}, \quad (t,x) \in [0,2] \times [0,2], \\
u(0,x) &= u_0(x), \\
u_t(0,x) &= 0, \\
u(t,0) &= u(t,2),
\end{align*}
where the time-dependent parameter function $\alpha(t)$ is drawn from a Gaussian random process with variance $1$.
The parametric function $\alpha(t)$ is evaluated at 64 sensor locations corresponding to the boundaries of uniformly spaced cells, $\{t_i^b\}_{i=1}^{64}$, and the resulting values ${\alpha(t_i^b)}$ are encoded within the parameter-approximation networks of $\ScalingNetwork$ and $\UAPNetwork$.

\begin{table}[htbp]
    \centering
    
   \caption{Performance comparison on the parametric diffusion-reaction equation (in-distribution).}
    \label{tab:pwaveresult}
   \begin{tabular}{l|c}
    \textbf{Model} & \textbf{Relative $\Lp{2}$ Error} \\ \hline
    DeepONet & 56.37\% \\
    DeepONet-C & 9.31\% \\
    MIONet & 13.66\% \\
    $\UAPNetwork$ & 6.95\% \\
    $\ScalingNetwork$-S & 5.72\% \\
    $\ScalingNetwork$-L & \textbf{4.41}\% \\
\end{tabular}
\end{table}

Since the parametric function is the time-dependent wave speed, an error in capturing the dependence can lead to incorrect dynamics for all time. From Figure~\ref{fig:pwaveoutput} we see that $\ScalingNetwork$ and $\UAPNetwork$ demonstrate a more balanced and overall lower local error distribution (see also Table \ref{tab:pwaveresult}), whereas the remaining models show pronounced error concentrations and patterned error. The patterns likely indicate that larger features are missing in the model. In particular, in regions with higher contrast, the comparable models emit coarse scale errors that degrade their predictive capabilities. 

 \begin{figure*}[htbp]
    \centering
    \includegraphics[width=0.8\linewidth]{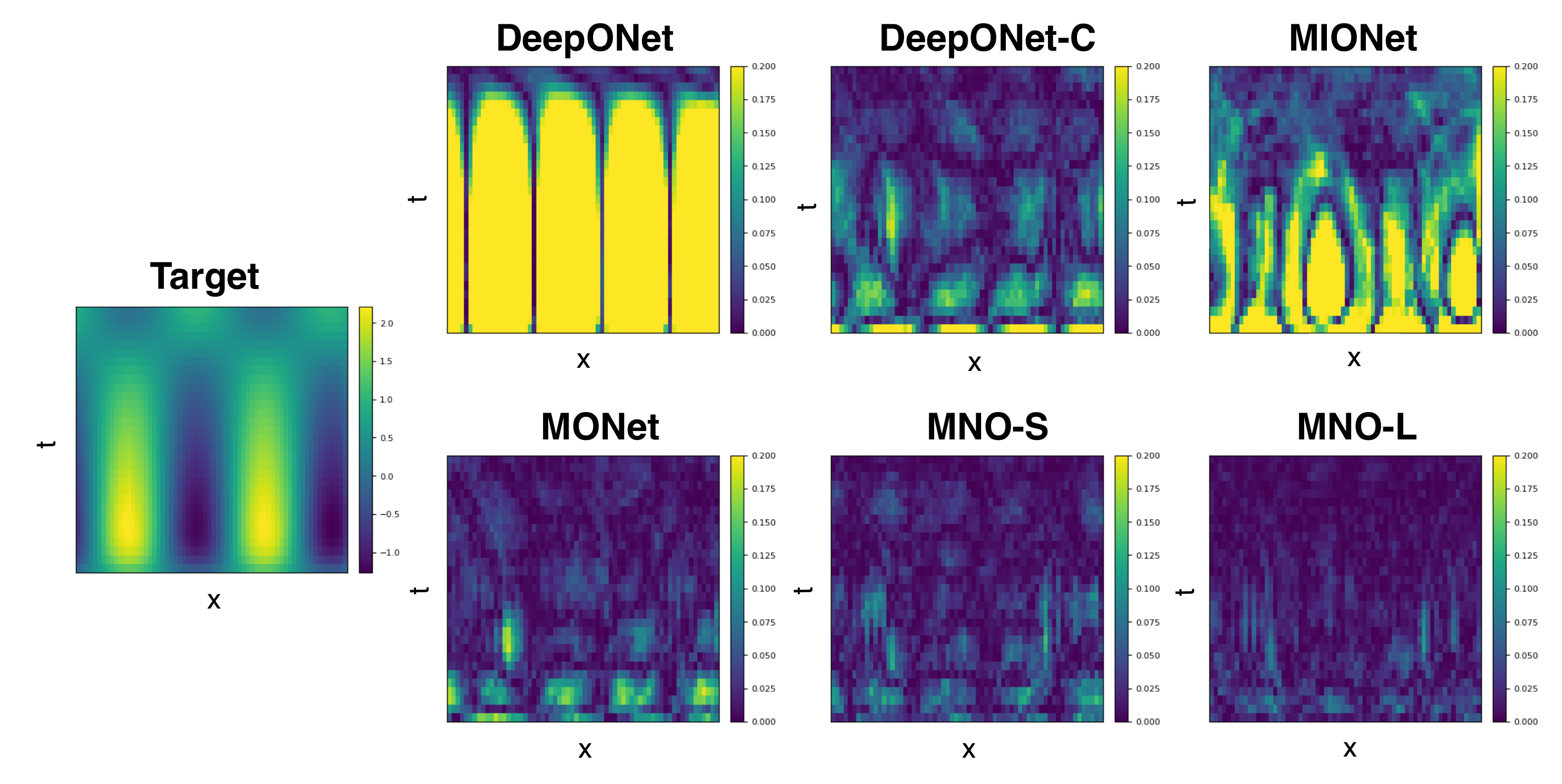}%%output9

 \caption{\textbf{Representative solution for the parametric wave equation}: The target solution (left) and error maps for  DeepONet,  DeepONet-C,  MIONet, $\UAPNetwork$, $\ScalingNetwork$-S and $\ScalingNetwork$-L, The instance-specific relative errors are 79.45\%, 6.99\%,   17.33\% , 4.83\%, 3.70\% and 2.47\%, respectively, aligning with the trends observed in Table~\ref{tab:pdrresult}.}
   
    \label{fig:pwaveoutput}
\end{figure*}

\section{Conclusion} \label{sec:discussion}

In this work, we provided theoretical insights into the problem of learning a collection of operators using neural networks. For the multiple operator learning setting, we introduced two new architectures, $\ScalingNetwork$ and $\UAPNetwork$, and established their universal approximation properties across different classes of operators. 
Our analysis covered continuous, integrable, and Lipschitz operators. 
In the latter case, we derived explicit scaling laws for $\ScalingNetwork$, quantifying how the network size must grow to achieve a prescribed approximation accuracy. 
We further empirically validated the effectiveness of both architectures on a range of parametric PDE problems, confirming their strong performance in practice. For the case of learning several single operators, we showed that the theoretical approximation order yields new insights into how computational complexity can be balanced among subnetworks and how overall scaling efficiency can be improved. 
This provides a principled framework for architectural design. 

Future research directions in the multiple operator learning context include establishing lower bounds on approximation and sample complexity similarly to \cite{lanthalerStuart}, developing a rigorous theory of generalization error as in \cite{liu2024neuralscalinglawsdeep}, extending the current analysis of approximation order, and exploring possible extensions to kernel-based operator learning frameworks \cite{BATLLE2024112549,jalalian2025dataefficientkernelmethodslearning}.

\section*{Acknowledgment}
AW and HS were supported in part by NSF DMS 2427558.  JS was supported in part by AFOSR FA9550-23-1-0445.
ZZ was supported by the U.S. Department of Energy (DOE) Office of Science Advanced Scientific Computing Research program DE-SC0025440.

\bibliographystyle{plain}
\bibliography{references}{}

\begin{thebibliography}{10}

\bibitem{anandkumar2019neural}
Anima Anandkumar, Kamyar Azizzadenesheli, Kaushik Bhattacharya, Nikola Kovachki, Zongyi Li, Burigede Liu, and Andrew Stuart.
\newblock Neural operator: Graph kernel network for partial differential equations.
\newblock In {\em ICLR 2020 Workshop on Integration of Deep Neural Models and Differential Equations}, 2019.

\bibitem{BATLLE2024112549}
Pau Batlle, Matthieu Darcy, Bamdad Hosseini, and Houman Owhadi.
\newblock Kernel methods are competitive for operator learning.
\newblock {\em Journal of Computational Physics}, 496:112549, 2024.

\bibitem{berlinet2004rkhs}
Alain Berlinet and Christine Thomas-Agnan.
\newblock {\em Reproducing Kernel Hilbert Spaces in Probability and Statistics}.
\newblock Springer, New York, NY, 2004.

\bibitem{Bhattacharya}
Kaushik Bhattacharya, Bamdad Hosseini, Nikola~B. Kovachki, and Andrew~M. Stuart.
\newblock Model {Reduction} {And} {Neural} {Networks} {For} {Parametric} {PDEs}.
\newblock {\em The SMAI Journal of computational mathematics}, 7:121--157, 2021.

\bibitem{Bogachev}
Vladimir~I. Bogachev.
\newblock {\em Measure theory}.
\newblock Springer Berlin, Heidelberg, Berlin, 2007.

\bibitem{BoulleGreen}
Nicolas Boullé, Seick Kim, Tianyi Shi, and Alex Townsend.
\newblock Learning green's functions associated with time-dependent partial differential equations.
\newblock {\em Journal of Machine Learning Research}, 23(218):1--34, 2022.

\bibitem{cao2024vicon}
Yadi Cao, Yuxuan Liu, Liu Yang, Rose Yu, Hayden Schaeffer, and Stanley Osher.
\newblock Vicon: Vision in-context operator networks for multi-physics fluid dynamics prediction.
\newblock {\em arXiv preprint arXiv:2411.16063}, 2024.

\bibitem{CASTRO2023127413}
Javier Castro.
\newblock The kolmogorov infinite dimensional equation in a hilbert space via deep learning methods.
\newblock {\em Journal of Mathematical Analysis and Applications}, 527(2):127413, 2023.

\bibitem{castro2022}
Javier Castro, Claudio Mu{\~n}oz, and Nicol{\'a}s Valenzuela.
\newblock The calder{\'o}n's problem via deeponets.
\newblock {\em Vietnam Journal of Mathematics}, 52(3):775--806, 2024.

\bibitem{chen2023neuraloperator}
Chuanqi Chen and Jinlong Wu.
\newblock Neural operator for modeling dynamic systems.
\newblock {\em arXiv preprint arXiv:2306.XXXX}, 2023.

\bibitem{ChenChen1993}
T.~Chen and H.~Chen.
\newblock Approximations of continuous functionals by neural networks with application to dynamic systems.
\newblock {\em IEEE Transactions on Neural Networks}, 4(6):910--918, 1993.

\bibitem{ChenChen1995}
Tianping Chen and Hong Chen.
\newblock Universal approximation to nonlinear operators by neural networks with arbitrary activation functions and its application to dynamical systems.
\newblock {\em IEEE Transactions on Neural Networks}, 6(4):911--917, 1995.

\bibitem{dehoop2022costaccuracytradeoffoperatorlearning}
Maarten~V. de~Hoop, Daniel~Zhengyu Huang, Elizabeth Qian, and Andrew~M. Stuart.
\newblock The cost-accuracy trade-off in operator learning with neural networks, 2022.

\bibitem{dugundji1951extension}
J.~Dugundji.
\newblock {An extension of Tietze's theorem.}
\newblock {\em Pacific Journal of Mathematics}, 1(3):353 -- 367, 1951.

\bibitem{furuya2023globally}
Takashi Furuya, Michael~Anthony Puthawala, Matti Lassas, and Maarten~V. de~Hoop.
\newblock Globally injective and bijective neural operators.
\newblock In {\em Thirty-seventh Conference on Neural Information Processing Systems}, 2023.

\bibitem{deepGreen}
Craig~R. Gin, Daniel~E. Shea, Steven~L. Brunton, and J.~Nathan Kutz.
\newblock Deepgreen: deep learning of green's functions for nonlinear boundary value problems.
\newblock {\em Scientific Reports}, 11(1):21614, 2021.

\bibitem{Goswami2023}
Somdatta Goswami, Aniruddha Bora, Yue Yu, and George~Em Karniadakis.
\newblock {\em Physics-Informed Deep Neural Operator Networks}, pages 219--254.
\newblock Springer International Publishing, Cham, 2023.

\bibitem{Graves2013SpeechRW}
Alex Graves, Abdel rahman Mohamed, and Geoffrey~E. Hinton.
\newblock Speech recognition with deep recurrent neural networks.
\newblock {\em 2013 IEEE International Conference on Acoustics, Speech and Signal Processing}, pages 6645--6649, 2013.

\bibitem{Jentzen}
Jiequn Han, Arnulf Jentzen, and Weinan E.
\newblock Solving high-dimensional partial differential equations using deep learning.
\newblock {\em Proceedings of the National Academy of Sciences}, 115(34):8505--8510, 2018.

\bibitem{deepLearningImages}
Kaiming He, Xiangyu Zhang, Shaoqing Ren, and Jian Sun.
\newblock Deep residual learning for image recognition.
\newblock In {\em 2016 IEEE Conference on Computer Vision and Pattern Recognition (CVPR)}, pages 770--778, 2016.

\bibitem{herrman}
Lukas Herrmann, Christoph Schwab, and Jakob Zech.
\newblock Neural and spectral operator surrogates: unified construction and expression rate bounds.
\newblock {\em Advances in Computational Mathematics}, 50(4):72, 2024.

\bibitem{Huang2025}
Daniel~Zhengyu Huang, Nicholas~H. Nelsen, and Margaret Trautner.
\newblock An operator learning perspective on parameter-to-observable maps.
\newblock {\em Foundations of Data Science}, 7(1):163--225, 2025.

\bibitem{jalalian2025dataefficientkernelmethodslearning}
Yasamin Jalalian, Juan Felipe~Osorio Ramirez, Alexander Hsu, Bamdad Hosseini, and Houman Owhadi.
\newblock Data-efficient kernel methods for learning differential equations and their solution operators: Algorithms and error analysis, 2025.

\bibitem{jiang2023fouriermionet}
Zhongyi Jiang, Min Zhu, Dongzhuo Li, Qiuzi Li, Yanhua~O. Yuan, and Lu~Lu.
\newblock Fourier-mionet: Fourier-enhanced multiple-input neural operators for multiphase modeling of geological carbon sequestration.
\newblock {\em arXiv preprint arXiv:2303.04778}, 2023.

\bibitem{mionet}
Pengzhan Jin, Shuai Meng, and Lu~Lu.
\newblock Mionet: Learning multiple-input operators via tensor product.
\newblock {\em SIAM Journal on Scientific Computing}, 44(6):A3490--A3514, 2022.

\bibitem{jollie2025time}
Derek Jollie, Jingmin Sun, Zecheng Zhang, and Hayden Schaeffer.
\newblock Time-series forecasting and refinement within a multimodal pde foundation model.
\newblock {\em Journal of Machine Learning for Modeling and Computing}, 6(2):77--89, 2025.

\bibitem{kaplan2020scalinglawsneurallanguage}
Jared Kaplan, Sam McCandlish, Tom Henighan, Tom~B. Brown, Benjamin Chess, Rewon Child, Scott Gray, Alec Radford, Jeffrey Wu, and Dario Amodei.
\newblock Scaling laws for neural language models, 2020.

\bibitem{kechris1995}
Alexander~S. Kechris.
\newblock {\em Classical Descriptive Set Theory}, volume 156 of {\em Graduate Texts in Mathematics}.
\newblock Springer-Verlag, New York, 1995.

\bibitem{KHOO_LU_YING_2021}
Yuehaw Khoo, Jianfeng Lu, and Lexing Ying.
\newblock Solving parametric pde problems with artificial neural networks.
\newblock {\em European Journal of Applied Mathematics}, 32(3):421–435, 2021.

\bibitem{Kovachki2021}
Nikola Kovachki, Samuel Lanthaler, and Siddhartha Mishra.
\newblock On universal approximation and error bounds for fourier neural operators.
\newblock {\em J. Mach. Learn. Res.}, 22(1), January 2021.

\bibitem{Kovachki2023}
Nikola Kovachki, Zongyi Li, Burigede Liu, Kamyar Azizzadenesheli, Kaushik Bhattacharya, Andrew Stuart, and Anima Anandkumar.
\newblock Neural operator: learning maps between function spaces with applications to pdes.
\newblock {\em J. Mach. Learn. Res.}, 24(1), January 2023.

\bibitem{KOVACHKI2024419}
Nikola~B. Kovachki, Samuel Lanthaler, and Andrew~M. Stuart.
\newblock Chapter 9 - operator learning: Algorithms and analysis.
\newblock In Siddhartha Mishra and Alex Townsend, editors, {\em Numerical Analysis Meets Machine Learning}, volume~25 of {\em Handbook of Numerical Analysis}, pages 419--467. Elsevier, 2024.

\bibitem{lanthalerPCAnet}
Samuel Lanthaler.
\newblock Operator learning with pca-net: upper and lower complexity bounds.
\newblock {\em J. Mach. Learn. Res.}, 24(1), January 2023.

\bibitem{Lanthaler2022}
Samuel Lanthaler, Siddhartha Mishra, and George~E Karniadakis.
\newblock Error estimates for deeponets: a deep learning framework in infinite dimensions.
\newblock {\em Transactions of Mathematics and Its Applications}, 6(1):tnac001, 03 2022.

\bibitem{lanthalerStuart}
Samuel Lanthaler and Andrew~M Stuart.
\newblock The parametric complexity of operator learning.
\newblock {\em IMA Journal of Numerical Analysis}, page draf028, 08 2025.

\bibitem{benitez}
Jose~Antonio Lara~Benitez, Takashi Furuya, Florian Faucher, Anastasis Kratsios, Xavier Tricoche, and Maarten~V. de~Hoop.
\newblock Out-of-distributional risk bounds for neural operators with applications to the helmholtz equation.
\newblock {\em J. Comput. Phys.}, 513(C), September 2024.

\bibitem{li2023fnoseismic}
Bian Li, Hanchen Wang, Shihang Feng, Xiu Yang, and Youzuo Lin.
\newblock Solving seismic wave equations on variable velocity models with fourier neural operator.
\newblock {\em IEEE Transactions on Geoscience and Remote Sensing}, 61:1--18, 2023.

\bibitem{multipole}
Zongyi Li, Nikola Kovachki, Kamyar Azizzadenesheli, Burigede Liu, Kaushik Bhattacharya, Andrew Stuart, and Anima Anandkumar.
\newblock Multipole graph neural operator for parametric partial differential equations.
\newblock In {\em Proceedings of the 34th International Conference on Neural Information Processing Systems}, NIPS '20, Red Hook, NY, USA, 2020. Curran Associates Inc.

\bibitem{li2021fourier}
Zongyi Li, Nikola Kovachki, Kamyar Azizzadenesheli, Burigede Liu, Kaushik Bhattacharya, Andrew Stuart, and Anima Anandkumar.
\newblock Fourier neural operator for parametric partial differential equations.
\newblock In {\em International Conference on Learning Representations (ICLR)}, 2021.
\newblock preprint arXiv:2010.08895.

\bibitem{liu2024}
Hao Liu, Haizhao Yang, Minshuo Chen, Tuo Zhao, and Wenjing Liao.
\newblock Deep nonparametric estimation of operators between infinite dimensional spaces.
\newblock {\em J. Mach. Learn. Res.}, 25(1), January 2024.

\bibitem{liu2024neuralscalinglawsdeep}
Hao Liu, Zecheng Zhang, Wenjing Liao, and Hayden Schaeffer.
\newblock Neural scaling laws of deep relu and deep operator network: A theoretical study, 2024.

\bibitem{liu2024prosefd}
Yuxuan Liu, Jingmin Sun, Xinjie He, Griffin Pinney, Zecheng Zhang, and Hayden Schaeffer.
\newblock Prose-fd: A multimodal pde foundation model for learning multiple operators for forecasting fluid dynamics.
\newblock {\em arXiv preprint arXiv:2409.09811}, 2024.

\bibitem{liu2025bcat}
Yuxuan Liu, Jingmin Sun, and Hayden Schaeffer.
\newblock Bcat: A block causal transformer for pde foundation models for fluid dynamics.
\newblock {\em arXiv preprint arXiv:2501.18972}, 2025.

\bibitem{liu2024prose}
Yuxuan Liu, Zecheng Zhang, and Hayden Schaeffer.
\newblock Prose: Predicting multiple operators and symbolic expressions using multimodal transformers.
\newblock {\em Neural Networks}, 180:106707, 2024.

\bibitem{deepOnet}
Lu~Lu, Pengzhan Jin, Guofei Pang, Zhongqiang Zhang, and George~Em Karniadakis.
\newblock Learning nonlinear operators via deeponet based on the universal approximation theorem of operators.
\newblock {\em Nature Machine Intelligence}, 3(3):218--229, 2021.

\bibitem{lu2022comprehensive}
Lu~Lu, Xuhui Meng, Shengze Cai, Zhiping Mao, Somdatta Goswami, Zhongqiang Zhang, and George~Em Karniadakis.
\newblock A comprehensive and fair comparison of two neural operators (with practical extensions) based on fair data.
\newblock {\em Computer Methods in Applied Mechanics and Engineering}, 393:114778, 2022.

\bibitem{marcati2023}
Carlo Marcati and Christoph Schwab.
\newblock Exponential convergence of deep operator networks for elliptic partial differential equations.
\newblock {\em SIAM Journal on Numerical Analysis}, 61(3):1513--1545, 2023.

\bibitem{markovsky2012lowrank}
Ivan Markovsky.
\newblock {\em Low Rank Approximation: Algorithms, Implementation, Applications}.
\newblock Communications and Control Engineering. Springer London, 1st edition, 2012.

\bibitem{mccabe2023multiple}
Michael McCabe, Bruno R{\'e}galdo-Saint~Blancard, Liam~Holden Parker, Ruben Ohana, Miles Cranmer, Alberto Bietti, Michael Eickenberg, Siavash Golkar, G{\'e}raud Krawezik, Francois Lanusse, et~al.
\newblock Multiple physics pretraining for physical surrogate models.
\newblock {\em arXiv preprint arXiv:2310.02994}, 2023.

\bibitem{Mezo2022LambertW}
István Mező.
\newblock {\em The Lambert W Function: Its Generalizations and Applications}.
\newblock Routledge / Chapman and Hall, London / New York, 2022.

\bibitem{moya2023operatorgrid}
Christian Moya, Guang Lin, Tianqiao Zhao, and Meng Yue.
\newblock On approximating the dynamic response of synchronous generators via operator learning: A step towards building deep operator-based power grid simulators.
\newblock {\em arXiv preprint arXiv:2301.12538}, 2023.

\bibitem{negrini2025multimodal}
Elisa Negrini, Yuxuan Liu, Liu Yang, Stanley~J Osher, and Hayden Schaeffer.
\newblock A multimodal pde foundation model for prediction and scientific text descriptions.
\newblock {\em arXiv preprint arXiv:2502.06026}, 2025.

\bibitem{pathak2022fourcastnet}
Jaideep Pathak, Shashank Subramanian, Peter Harrington, Sanjeev Raja, Ashesh Chattopadhyay, Morteza Mardani, Thorsten Kurth, David Hall, Zongyi Li, Kamyar Azizzadenesheli, and Anima Anandkumar.
\newblock Fourcastnet: A global data-driven high-resolution weather model using adaptive fourier neural operators.
\newblock {\em arXiv preprint arXiv:2202.11214}, 2022.

\bibitem{sun2025foundation}
Jingmin Sun, Yuxuan Liu, Zecheng Zhang, and Hayden Schaeffer.
\newblock Towards a foundation model for partial differential equations: Multioperator learning and extrapolation.
\newblock {\em Physical Review E}, 111(3):035304, 2025.

\bibitem{sun2025lemonlearninglearnmultioperator}
Jingmin Sun, Zecheng Zhang, and Hayden Schaeffer.
\newblock Lemon: Learning to learn multi-operator networks, 2025.

\bibitem{takamoto2022pdebench}
Makoto Takamoto, Timothy Praditia, Raphael Leiteritz, Dan MacKinlay, Francesco Alesiani, Dirk Pfl\"{u}ger, and Mathias Niepert.
\newblock Pdebench: an extensive benchmark for scientific machine learning.
\newblock In {\em Proceedings of the 36th International Conference on Neural Information Processing Systems}, NIPS '22, Red Hook, NY, USA, 2022. Curran Associates Inc.
\newblock doi: 10.5555/3600270.3600387.

\bibitem{yang2023incontext}
Liu Yang, Siting Liu, Tingwei Meng, and Stanley~J Osher.
\newblock In-context operator learning with data prompts for differential equation problems.
\newblock {\em Proceedings of the National Academy of Sciences}, 120(39):e2310142120, 2023.

\bibitem{yang2023prompting}
Liu Yang, Tingwei Meng, Siting Liu, and Stanley~J Osher.
\newblock Prompting in-context operator learning with sensor data, equations, and natural language.
\newblock {\em arXiv preprint arXiv:2308.05061}, 2023.

\bibitem{ye2025pdeformer}
Zhanhong Ye, Zining Liu, Bingyang Wu, Hongjie Jiang, Leheng Chen, Minyan Zhang, Xiang Huang, Qinghe~Meng Zou, Hongsheng Liu, and Bin Dong.
\newblock Pdeformer-2: A versatile foundation model for two-dimensional partial differential equations.
\newblock {\em arXiv preprint arXiv:2507.15409}, 2025.

\bibitem{zhang2025probabilistic}
Benjamin~J Zhang, Siting Liu, Stanley~J Osher, and Markos~A Katsoulakis.
\newblock Probabilistic operator learning: generative modeling and uncertainty quantification for foundation models of differential equations.
\newblock {\em arXiv preprint arXiv:2509.05186}, 2025.

\bibitem{zhang2024modno}
Zecheng Zhang.
\newblock Modno: Multi-operator learning with distributed neural operators.
\newblock {\em Computer Methods in Applied Mechanics and Engineering}, 431:117229, 2024.

\bibitem{zhang2025discretization}
Zecheng Zhang, Wing~Tat Leung, and Hayden Schaeffer.
\newblock A discretization-invariant extension and analysis of some deep operator networks.
\newblock {\em Journal of Computational and Applied Mathematics}, 456:116226, 2025.

\bibitem{zhang2024d2no}
Zecheng Zhang, Christian Moya, Lu~Lu, Guang Lin, and Hayden Schaeffer.
\newblock D2no: Efficient handling of heterogeneous input function spaces with distributed deep neural operators.
\newblock {\em Computer Methods in Applied Mechanics and Engineering}, 428:117084, 2024.

\bibitem{zhangBelnet}
Zecheng Zhang, Leung Wing~Tat, and Hayden Schaeffer.
\newblock Belnet: basis enhanced learning, a mesh-free neural operator.
\newblock {\em Proceedings of the Royal Society A: Mathematical, Physical and Engineering Sciences}, 479(2276):20230043, 2023.

\bibitem{zhu2023fourierdeeponet}
Min Zhu, Shihang Feng, Youzuo Lin, and Lu~Lu.
\newblock Fourier-deeponet: Fourier-enhanced deep operator networks for full waveform inversion with improved accuracy, generalizability, and robustness.
\newblock {\em arXiv preprint arXiv:2305.17289}, 2023.

\end{thebibliography}

\appendix

\section{Experiment Setup}
\label{MOLEdetail}
\subsection{Training}

The models are trained using the AdamW optimizer for 50 epochs where each epoch is 2K steps.  On 2 NVIDIA GeForce RTX 4090 GPUs with 24 GB memory, Table~\ref{tab:time} indicates the training time for different models and configurations.

 \begin{table}[htbp]
     \centering
          \caption{Training time for different models and configurations.}
     \label{tab:time}
     \begin{tabular}{c|c}
     \hline
       $\UAPNetwork$&  30 min  \\
       $\ScalingNetwork$-S & 30 min \\
       $\ScalingNetwork$-L   & 1 h 9 min\\
       DeepONet & 29 min \\
       DeepONet-C &  29 min\\
       MIONet & 47 min\\
       \hline
     \end{tabular}
 \end{table}
 
\subsection{Hyperparameters}

The optimizer hyperparameters are summarized in Table~\ref{tab:optim_hyperM}.

\begin{table}[htbp]
    \centering
    \caption{Optimizer hyperparameters.}
        \small
    \setlength{\tabcolsep}{2pt} 
\renewcommand{\arraystretch}{2}
    \begin{tabular}{l l | l l }
    \hline
    Learning rate & $10^{-4}$ &Gradient norm clip & 1.0 \\
    Scheduler & Cosine &  Weight decay & $10^{-4}$\\
    Batch data size& 150 & Warmup steps & 10\% of total steps\\ 
     Batch task size& 5 & \\
    \hline
    \end{tabular}
    \label{tab:optim_hyperM}
\end{table}

\end{document}